\documentclass[11pt]{article}
\usepackage{amsmath}
\usepackage{amsthm}
\usepackage{amssymb}
\usepackage{algorithm}
\usepackage{subfig}
\usepackage{color}
\usepackage[english]{babel}
\usepackage{graphicx}
\usepackage{wrapfig,epsfig}
\usepackage{epstopdf}
\usepackage{url}
\usepackage{graphicx}
\usepackage{color}
\usepackage{epstopdf}
\usepackage{algpseudocode}
\usepackage{scrextend}
\usepackage[T1]{fontenc}
\usepackage{bbm}
\usepackage{comment}

\usepackage{url}            % simple URL typesetting

\usepackage[margin=1in]{geometry}
\usepackage{appendix}

\usepackage{enumitem}
\usepackage[hidelinks,pdfencoding=auto,psdextra]{hyperref}
\hypersetup{colorlinks=true,citecolor=blue,linkcolor=blue}

\DeclareMathOperator{\poly}{poly}

\DeclareMathOperator*{\argmin}{\mathrm{argmin}}
\newcommand{\sign}{\mathrm{sign}}
\newcommand{\diag}{\mathrm{diag}}
\newcommand{\rank}{\mathrm{rank}}
\newcommand{\wh}{\widehat}
\newcommand{\wt}{\widetilde}
\renewcommand{\bar}{\overline}
\newcommand{\ov}{\overline}

\renewcommand{\tilde}{\widetilde}
\newcommand{\ReLU}{{$\mathsf{ReLU}$}}
\DeclareMathOperator*{\R}{{\mathbb{R}}}
\newcommand{\bone}{{\mathbf{1}}}
\DeclareMathOperator*{\E}{{\mathbb{E}}}
\DeclareMathOperator*{\D}{{\mathcal{D}}}
\DeclareMathOperator*{\N}{{\mathcal{N}}}

\makeatletter
\newcommand*{\RN}[1]{\expandafter\@slowromancap\romannumeral #1@}
\makeatother

\usepackage{etoolbox}

\makeatletter

\newcommand{\define}[4][ignore]{%
  \ifstrequal{#1}{ignore}{}{
  \@namedef{thmtitle@#2}{#1}}%
  \@namedef{thm@#2}{#4}%
  \@namedef{thmtypen@#2}{lemma}%
  \newtheorem{thmtype@#2}[theorem]{#3}%
  \newtheorem*{thmtypealt@#2}{#3~\ref{#2}}%
}

\newcommand{\state}[1]{%
  \@namedef{curthm}{#1}
  \@ifundefined{thmtitle@#1}{
  \begin{thmtype@#1}
    }{
  \begin{thmtype@#1}[\@nameuse{thmtitle@#1}]
  }
    \label{#1}
    \@nameuse{thm@#1}
  \end{thmtype@#1}
  \@ifundefined{thmdone@#1}{
  \@namedef{thmdone@#1}{stated}%
  }{}
}

\newcommand{\restate}[1]{%
  \@namedef{curthm}{#1}
  \@ifundefined{thmtitle@#1}{
    \begin{thmtypealt@#1}
    }{
  \begin{thmtypealt@#1}[\@nameuse{thmtitle@#1}]
  }
    \@nameuse{thm@#1}
  \end{thmtypealt@#1}
  \@ifundefined{thmdone@#1}{
  \@namedef{thmdone@#1}{stated}%
  }{}
}

\newcommand{\thmlabel}[1]{
  \@ifundefined{thmdone@\@nameuse{curthm}}{\label{#1}
    }{\tag*{\eqref{#1}}}
}
\makeatother
\author{
  Kai Zhong\thanks{Supported in part by NSF grants CCF-1320746, IIS-1546452 and CCF-1564000, and part of the work was done while interning in Microsoft research, India.} \\
  \texttt{zhongkai@ices.utexas.edu}\\
  UT-Austin
  \and
  Zhao Song\thanks{Supported in part by UTCS TAship (CS361 Spring 17 Introduction to Computer Security).} \\
  \texttt{zhaos@utexas.edu}\\
  UT-Austin
  \and
  Prateek Jain \\
  \texttt{prajain@microsoft.com}\\
  Microsoft Research, India
  \and
  Peter L. Bartlett\thanks{Supported in part by Australian Research Council through an Australian Laureate Fellowship (FL110100281) and through the Australian Research Council Centre of Excellence for Mathematical and Statistical Frontiers (ACEMS), and NSF grants IIS-1619362.} \\
  \texttt{bartlett@cs.berkeley.edu}\\
  UC, Berkeley
  \and
  Inderjit S. Dhillon\thanks{Supported in part by NSF grants CCF-1320746, IIS-1546452 and CCF-1564000.} \\
  \texttt{inderjit@cs.utexas.edu}\\
  UT-Austin
}

\date{}
\title{Recovery Guarantees for One-hidden-layer Neural Networks\thanks{A preliminary version of this paper appears in Proceedings of the Thirty-fourth International Conference on Machine Learning (ICML 2017).}}

\newtheorem{theorem}{Theorem}[section]
\newtheorem{lemma}[theorem]{Lemma}
\newtheorem{definition}[theorem]{Definition}

\newtheorem{corollary}[theorem]{Corollary}

\newtheorem{assumption}[theorem]{Assumption}

\newtheorem{fact}[theorem]{Fact}
\newtheorem{remark}[theorem]{Remark}
\newtheorem{claim}[theorem]{Claim}

\newtheorem{property}[theorem]{Property}

\begin{document}

\begin{titlepage}
  \maketitle
  \begin{abstract}
In this paper, we consider regression problems with one-hidden-layer
neural networks (1NNs). We distill some properties of activation
functions that lead to \emph{local strong convexity} in the
neighborhood of the ground-truth parameters for the 1NN squared-loss
objective. Most popular nonlinear activation functions satisfy the
distilled properties, including rectified linear units (\ReLU s), leaky \ReLU s,
squared \ReLU s and sigmoids. For activation functions that are also
smooth, we show \emph{local linear convergence} guarantees of gradient
descent under a resampling rule. 
For homogeneous activations, we show tensor methods are able to initialize the parameters to fall into the local strong convexity region. As a result, tensor initialization followed by gradient descent is guaranteed to recover the ground truth with sample complexity $ d \cdot \log(1/\epsilon) \cdot \poly(k,\lambda )$ and computational complexity $n\cdot d \cdot \poly(k,\lambda) $ for smooth homogeneous activations with high probability, 
where $d$ is the dimension of the input, $k$ ($k\leq d$) is the number of hidden nodes, $\lambda$ is a conditioning property of the ground-truth parameter matrix between the input layer and the hidden layer, $\epsilon$ is the targeted precision and $n$ is the number of samples. 
To the best of our knowledge, this is the first work that provides recovery guarantees for 1NNs with both sample complexity and computational complexity \emph{linear} in the input dimension and \emph{logarithmic} in the precision.

  \end{abstract}
  \thispagestyle{empty}
\end{titlepage}
{\hypersetup{linkcolor=black}
\tableofcontents
}
\newpage
%%%%%%%%%%%%%% SECTION %%%%%%%%%%%%%% 
\section{Introduction}
Neural Networks (NNs) have achieved great practical success recently.
Many theoretical contributions have been made very recently to
understand the extraordinary performance of NNs. The remarkable
results of NNs on complex tasks in computer vision and natural
language processing inspired works on the expressive power of NNs
\cite{%h91,
css16,cs16,rpk16,dfs16,plr16,mpcb14,t16}. Indeed, several works found NNs are very powerful and the deeper the
more powerful. However, due to the high non-convexity of NNs, knowing
the expressivity of NNs doesn't guarantee that the targeted functions
will be learned. Therefore, several other works focused on the
achievability of global optima. Many of them considered the
over-parameterized setting, where the global optima or local minima
close to the global optima will be achieved when the number of
parameters is large enough, including \cite{fb16,hv15,lss14,dpg14,ss16,hm17}. This, however, leads to overfitting easily and can't provide any generalization guarantees, which are actually the essential goal in most tasks. 

A few works have considered generalization performance. For example,
\cite{xls17} provide generalization bound under the
Rademacher generalization analysis framework. Recently
\cite{zbh17} describe some experiments showing that
NNs are complex enough that they actually memorize the training data
but still generalize well. As they claim, this cannot be explained by
applying generalization analysis techniques, like VC dimension and
Rademacher complexity, to classification loss (although it does not
rule out a margins analysis---see, for example,~\cite{b98};
their experiments involve the unbounded cross-entropy loss).

In this paper, we don't develop a new generalization analysis. Instead
we focus on parameter recovery setting, where we assume there are underlying ground-truth parameters and we provide recovery guarantees for the ground-truth parameters up to equivalent permutations. Since the parameters are exactly recovered, the generalization performance will also be guaranteed. 

Several other techniques are also provided to recover the parameters or to guarantee generalization performance, such as tensor methods \cite{jsa15} and kernel methods \cite{agmr16}. These methods require sample complexity $O(d^3)$ or computational complexity $\widetilde O(n^2)$, which can be intractable in practice. 
We propose an algorithm that has recovery guarantees for 1NN with sample complexity $\widetilde O(d)$ and computational time $\widetilde O(d n)$ under some mild assumptions. 

Recently \cite{s16} show that neither specific assumptions on the niceness of the input distribution or niceness of the target function alone is sufficient to guarantee learnability using gradient-based methods.
In this paper, we assume data points are sampled from Gaussian distribution and the parameters of hidden neurons are linearly independent. 

Our main contributions are as follows,
\begin{enumerate}
\item We distill some properties for activation functions, which are satisfied by a wide range of activations, including ReLU, squared ReLU, sigmoid and tanh. With these properties we show positive definiteness (PD) of the Hessian in the neighborhood of the ground-truth parameters given enough samples  (Theorem~\ref{thm:lsc_nn_informal}). Further, for activations that are also smooth, we show local linear convergence is guaranteed using gradient descent. 
\item We propose a tensor method to initialize the parameters such that the initialized parameters fall into the local positive definiteness area. Our contribution is that we reduce the sample/computational complexity from cubic dependency on dimension to linear dependency (Theorem~\ref{thm:tensor_final}). 
\item Combining the above two results, we provide a globally converging algorithm (Algorithm~\ref{alg:overall_alg}) for smooth homogeneous activations satisfying the distilled properties. The whole procedure requires sample/computational complexity linear in dimension and logarithmic in precision (Theorem~\ref{thm:overall}).  
\end{enumerate}

%However, the optimization behavior on NNs is not well understood. Recent theoretical studies on the optimization of NNs / deep learning include non-existence of spurious local minimum \cite{k16, sc16,xls17} and closeness between local minima and the global optima \cite{fb16, chm15}. These methods require the number of parameters in the network to be larger than the number of data points, which however may lead to overfitting easily. 

\section{Related Work}
The recent empirical success of NNs has boosted their theoretical analyses \cite{fzk16,b16,bmb16,sbl16,apvz14,agmr16,gkkt16}. In this paper, we classify them into three main directions.

\subsection{Expressive Power}
Expressive power is studied to understand the remarkable performance
of neural networks on complex tasks. Although one-hidden-layer neural
networks with sufficiently many hidden nodes can approximate any
continuous function \cite{h91}, shallow networks
can't achieve the same performance in practice as deep networks.
Theoretically, several recent works show the depth of NNs plays an
essential role in the expressive power of neural networks
\cite{dfs16}. As shown in
\cite{css16,cs16,t16},
functions that can be implemented by a deep network of polynomial
size require exponential size in order to be implemented by a shallow
network.
\cite{rpk16,plr16,mpcb14,agmr16}
design some measures of expressivity that display an exponential
dependence on the depth of the network. However, the increasing of the
expressivity of NNs or its depth also increases the difficulty of the
learning process to achieve a good enough model. In this paper, we
focus on 1NNs and provide recovery guarantees using a finite number of
samples.

\subsection{Achievability of Global Optima}
The global convergence is in general not guaranteed for NNs due to
their non-convexity. It is widely believed that training deep models
using gradient-based methods works so well because the error surface
either has no local minima, or if they exist they need to be close in
value to the global minima. \cite{scp16} present examples showing that for this to be true additional assumptions on the data, initialization schemes and/or the model classes have to be made. Indeed the achievability of global optima has been shown under many different types of assumptions.

In particular, 
\cite{chm15} analyze the loss surface of a special
random neural network through spin-glass theory and show that it has
exponentially many local optima, whose loss is small and close to that
of a global optimum.  Later on, \cite{k16} eliminate
some assumptions made by \cite{chm15} but still
require the independence of activations as \cite{chm15}, which is unrealistic. 
\cite{ss16} study the geometric structure of the neural network objective function. They have shown that with high probability random initialization will fall into a basin with a small objective value when the network is over-parameterized.  
\cite{lss14} consider polynomial networks where the activations are square functions, which are typically not used in practice. 
\cite{hv15} show that when a local minimum has zero parameters related to a hidden node, a global optimum is achieved. 
\cite{fb16} study the landscape of 1NN in terms of
topology and geometry, and show that the level set becomes connected as the network is increasingly over-parameterized. 
\cite{hm17} show that products of matrices don't have
spurious local minima and that deep residual networks can represent
any function on a sample, as long as the number of parameters is
larger than the sample size. 
\cite{sc16} consider over-specified NNs, where the number of samples is smaller than the number of weights.
\cite{dpg14} propose a new approach to second-order optimization that identifies and attacks the saddle point problem in high-dimensional non-convex optimization. They apply the approach to recurrent neural networks and show practical performance. 
\cite{agmr16} use results from tropical geometry to
show global optimality of an algorithm, but it requires $(2n)^k\poly(n)$ computational complexity. %%% poly is able to hide any constant.

Almost all of these results require the number of parameters is larger than the number of points, which probably overfits the model and no generalization performance will be guaranteed. 
In this paper, we propose an efficient and provable algorithm for 1NNs
that can achieve the underlying ground-truth parameters. 

\subsection{Generalization Bound / Recovery Guarantees}
The achievability of global optima of the objective from the training data doesn't guarantee the learned model to be able to generalize well on unseen testing data. In the literature, we find three main approaches to generalization guarantees. 

1) {\it Use generalization analysis frameworks}, including VC
dimension/Rademacher complexity, to bound the generalization error. A
few works have studied the generalization performance for NNs.
\cite{xls17} follow \cite{sc16} but additionally
provide generalization bounds using Rademacher complexity. They assume
the obtained parameters are in a regularization set so that the
generalization performance is guaranteed, but this assumption can't be
justified theoretically.  \cite{hrs16} apply stability analysis to the
generalization analysis of SGD for convex and non-convex problems,
arguing early stopping is important for generalization performance. 

2) {\it Assume an underlying model and try to recover this model.}
This direction is popular for many non-convex problems including
matrix completion/sensing \cite{jns13,h14,sl15,blwz17}, mixed linear regression
\cite{zjd16}, subspace recovery \cite{ev09}
and other latent models \cite{agh14}. 

Without making any assumptions, those non-convex problems are intractable \cite{agkm12, gv15,swz17a,gg11,rsw16,sr11,hm13,agm12,yi2014alternating}. Recovery guarantees for NNs also need assumptions. Several different approaches under different assumptions are provided to have recovery guarantees on different NN settings.

Tensor methods
\cite{agh14,wsta15,wa16,swz16} are a general
tool for recovering models with latent factors by assuming the data
distribution is known. Some existing recovery guarantees for NNs are
 provided by tensor methods \cite{sa15,jsa15}.
However, \cite{sa15} only provide guarantees to recover the subspace spanned by the weight matrix and no sample complexity is given, while \cite{jsa15} require $O(d^3 /\epsilon^2)$ sample complexity.
In this paper, we use tensor methods as an initialization step so that we don't need very accurate estimation of the moments, which enables us to reduce the total sample complexity from $1/\epsilon^2$ to $\log(1/\epsilon)$. 

\cite{abgm14} provide polynomial sample complexity and
computational complexity bounds for learning deep representations in
unsupervised setting, and they need to assume the weights are sparse
and randomly distributed in $[-1,1]$.
 
 \cite{t17} analyze 1NN by assuming Gaussian inputs in a
supervised setting, in particular, regression and classification
with a teacher. This paper also considers this setting. However, there are some key differences. a)
\cite{t17} require the second-layer parameters are all
ones, while we can learn these parameters. b) In \cite{t17}, the ground-truth first-layer weight vectors are required to be orthogonal, while we only require linear independence. c) \cite{t17} require a good initialization but doesn't provide initialization methods, while we show the parameters can be efficiently initialized by tensor methods. d)
In \cite{t17}, only the population case (infinite sample size) is
considered, so there is no sample complexity analysis, while we show finite sample complexity. 

Recovery guarantees for convolution neural network with Gaussian inputs are provided in \cite{bg17}, where they show a globally converging guarantee of gradient descent on a one-hidden-layer no-overlap convolution neural network. However, they consider population case, so no sample complexity is provided. Also their analysis depends on \ReLU~activations and the no-overlap case is very unlikely to be used in practice. In this paper, we consider a large range of activation functions, but for one-hidden-layer fully-connected NNs. 

3) {\it Improper Learning.} In the improper learning setting for NNs, the learning algorithm is not restricted to output a NN, but only should output a prediction function whose error is not much larger than the error of the best NN among all the NNs considered. 
\cite{zlj16,zlw16} propose kernel methods
to learn the prediction function which is
guaranteed to have generalization performance close to that of the
NN. However, the sample complexity and computational
complexity are exponential. \cite{azs14} transform NNs to
convex semi-definite programming. The works by
\cite{b14} and \cite{brv05}
are also in this direction.
However, these methods are actually not learning the original NNs.
Another work by \cite{zlwj17} uses random
initializations to achieve arbitrary small excess risk. However, their
algorithm has exponential running time in $1/\epsilon$.

%general non-convex approaches, \cite{dpg14,hrs16}

\paragraph{Roadmap.} The paper is organized as follows. In
Section~\ref{sec:preliminary}, we present our problem setting and show
three key properties of activations required for our guarantees. In Section~\ref{sec:lsc_nns}, we introduce the formal theorem of local strong convexity and show local linear convergence for smooth activations. Section~\ref{sec:tensor} presents a tensor method to initialize the parameters so that they fall into the basin of the local strong convexity region.

%%%%%%%%%%%%%% SECTION %%%%%%%%%%%%%% 
\section{Problem Formulation}\label{sec:preliminary}
We consider the following regression problem. Given a set of $n$ samples 
\begin{align*}
S=\{ ( x_1, y_1), ( x_2,y_2), \cdots ( x_n,y_n) \} \subset \mathbb{R}^d \times \mathbb{R},
\end{align*} 
let ${\cal D}$ denote a underlying distribution over
$\mathbb{R}^d\times\mathbb{R}$ with parameters 
\begin{align*}
\{  w^*_{1},  w^*_2,
\cdots  w^*_k \} \subset  \mathbb{R}^d, \text{~and~} \{v^*_1, v^*_2, \cdots,
v^*_k \} \subset \mathbb{R}
\end{align*}
 such that each sample $( x,y) \in S$ is
sampled i.i.d.~from this distribution, with 
\begin{align}\label{eq:model}
{\cal D}:\qquad
 x \sim {\cal N} (0,I), ~ ~
y= \sum_{i=1}^k v^*_i \cdot \phi( w_{i}^{*\top }  x),
\end{align}
where $\phi(z)$ is the activation function, $k$ is the number of nodes in the hidden layer. 
The main question we want to answer is:  How many samples are sufficient to recover the underlying parameters?

It is well-known that, training one hidden layer neural network is NP-complete \cite{br88}. Thus, without making any assumptions, learning deep neural network is intractable. Throughout the paper, we assume $x$ follows a standard normal distribution; the data is noiseless; the dimension of input data is at least the number of hidden nodes; and activation function $\phi(z)$ satisfies some reasonable properties.

 Actually our results can be easily extended to
multivariate Gaussian distribution with positive definite covariance
and zero mean since we can estimate the covariance first and then
transform the input to a standard normal distribution but with some
loss of accuracy. Although this paper focuses on the regression problem, we can transform classification problems to regression problems if a good teacher is provided as described in \cite{t17}. Our analysis requires $k$ to be no greater than $d$, since the first-layer parameters will be linearly dependent otherwise.

For activation function $\phi(z)$, we assume it is continuous and if it is non-smooth let its first derivative be left derivative. Furthermore, we assume it satisfies Property \ref{pro:gradient}, \ref{pro:expect}, and \ref{pro:hessian}. These properties are critical for the later analyses. We also observe that most activation functions actually satisfy these three properties.

\begin{property}\label{pro:gradient}
The first derivative $\phi'(z)$ is nonnegative and homogeneously bounded, i.e., $0\leq \phi'(z) \leq L_1 |z|^p$ for some constants $L_1>0$ and $p\geq 0$.
\end{property}

\begin{property} \label{pro:expect}
Let $\alpha_q(\sigma) = {\E}_{z\sim {\cal N}(0,1)}[\phi'(\sigma \cdot z) z^q], \forall q\in \{0,1,2\}$, and $\beta_q(\sigma) = {\E}_{z\sim {\cal N}(0,1)} [\phi'^2(\sigma \cdot z) z^q ] , \forall q\in \{0,2\}.$
Let $\rho(\sigma)$ denote
$
 \min\{\beta_0(\sigma) -\alpha_0^2(\sigma) - \alpha_1^2(\sigma), \;
 \beta_2(\sigma) - \alpha_1^2(\sigma)- \alpha_2^2(\sigma),\;
  \alpha_0(\sigma)\cdot \alpha_2(\sigma) - \alpha_1^2(\sigma) \}
$
The first derivative $\phi'(z)$ satisfies that, for all $\sigma>0$, we have $\rho(\sigma)>0$.
\end{property}
\begin{property}\label{pro:hessian}
The second derivative $\phi''(z)$ is either {\bf (a)} globally bounded $|\phi''(z)| \leq L_2$ for some constant $L_2$, i.e., $\phi(z)$ is $L_2$-smooth, or {\bf (b)} $\phi''(z)=0$ except for $e$ ($e$ is a finite constant) points.
\end{property}

\begin{remark} The first two properties are related to the first derivative $\phi'(z)$ and the last one is about the second derivative $\phi''(z)$. At high level, Property~\ref{pro:gradient} requires $\phi$ to be non-decreasing with homogeneously bounded derivative; Property~\ref{pro:expect} requires $\phi$ to be highly non-linear; Property~\ref{pro:hessian} requires $\phi$ to be either smooth or piece-wise linear. 
 \end{remark}
\define{proposition:act_example}{Theorem}{
\ReLU~$\phi(z) =\max\{z,0\}$, leaky \ReLU~$\phi(z)=\max\{z,0.01z\}$, squared \ReLU~$\phi(z) =\max\{z,0\}^2$ and any non-linear non-decreasing smooth functions with bounded symmetric $\phi'(z)$, like the sigmoid function $\phi(z) = 1/(1+e^{-z})$, the $\mathrm{tanh}$ function and the $\mathrm{erf}$ function $\phi(z) = \int_0^z e^{-t^2} dt $, satisfy Property~\ref{pro:gradient},\ref{pro:expect},\ref{pro:hessian}. The linear function, $\phi(z)=z$, doesn't satisfy Property~\ref{pro:expect} and the quadratic function, $\phi(z) = z^2$,  doesn't satisfy Property~\ref{pro:gradient} and \ref{pro:expect}. 
}
\state{proposition:act_example}
\ifdefined\icmlcamera
The proof can be found in the full version~\cite{full}.
\else
The proof can be found in Appendix~\ref{app:proof_prop1}.
\fi

%%%%%%%%%%%%%% SECTION %%%%%%%%%%%%%% 
\section{Positive Definiteness of Hessian}\label{sec:lsc_nns}
In this section, we study the Hessian of empirical risk near the ground truth. We consider the case when $ v^*$ is already known. Note that for homogeneous activations, we can assume $v^*_i\in\{-1,1\}$ since $v \phi(z) = \frac{v}{|v|} \phi(|v|^{1/p} z)$, where $p$ is the degree of homogeneity. As $v_i^*$ only takes discrete values for homogeneous activations, in the next section, we show we can exactly recover $ v^*$ using tensor methods with finite samples. %For non-homogeneous activations, the tensor methods can approximate $v^*$ to a small error.

%For any $W\in \mathbb{R}^{d\times k}$, we consider the following squared loss objective %to learn the parameters $W^* := [ w_1^*, w_2^*,\cdots, w_K^*]$ and $ v^* := [v^*_1;\cdots;v^*_K]$,
For a set of samples $S$, we define the {\it Empirical Risk},
\begin{equation}\label{eq:nns_emp_obj}
\widehat{f}_S(W) =  \frac{1}{2|S|}\sum_{ ( x,y) \in S } \left( \sum_{i=1}^k v^*_i \phi( w_{i}^\top  x) - y \right) ^2.
\end{equation}
For a distribution ${\cal D}$, we define the {\it Expected Risk},
\begin{align}\label{eq:nns_pop_obj}
f_{\cal D}(W) =  \frac{1}{2}\underset{( x,y)\sim{\cal D}} {\mathbb{E}} \left[\left( \sum_{i=1}^k v^*_i \phi( w_{i}^\top  x) - y \right) ^2\right].
\end{align}
Let's calculate the gradient and the Hessian of $\widehat{f}_S(W)$ and $f_{\cal D}(W)$. For each $j\in [k]$, the partial gradient of $f_{\cal D}(W) $ with respect to $ w_j$ can be represented as
\begin{align*}
 \frac{\partial f_{\cal D}(W)}{\partial  w_j} = \underset{( x,y)\sim{\cal D}}{\mathbb{E}} \left[ \left( \sum_{i=1}^k v^*_i \phi( w_{i}^\top  x)  - y \right) v^*_j \phi'( w_j^\top  x)  x \right].
\end{align*}
For each $j,l\in[k]$ and $j\neq l$, the second partial derivative of $f_{\cal D}(W)$ for the $(j,l)$-th off-diagonal block is,
\begin{align*}
\frac{\partial^2 f_{\cal D}(W)}{\partial  w_j \partial  w_l} = & ~ \underset{( x,y)\sim{\cal D}}{\mathbb{E}} \left[  v^*_j v^*_l \phi'( w_j^\top  x)  \phi'( w_l^\top  x)   x  x^\top \right] ,
\end{align*}
and for each $j\in [k]$, the second partial derivative of $f_{\cal D}(W)$ for the $j$-th diagonal block is
\ifdefined\icmlcamera
\begin{align*}
\frac{\partial^2 f_{\cal D}(W)}{\partial  w_j^2} = & \underset{( x,y)\sim{\cal D}}{\mathbb{E}} [ ( \overset{k}{ \underset{i=1}{\sum} } v_i^* \phi( w_{i}^\top  x)  - y ) v_j^* \phi''( w_j^\top  x)  x  x^\top  \\
& + (  v^*_j \phi'( w_j^\top  x) )^2  x  x^\top ].
\end{align*}
\else
\begin{align*}
\frac{\partial^2 f_{\cal D}(W)}{\partial  w_j^2} = & \underset{( x,y)\sim{\cal D}}{\mathbb{E}} \left[ \left( \overset{k}{ \underset{i=1}{\sum} } v_i^* \phi( w_{i}^\top  x)  - y \right) v_j^* \phi''( w_j^\top  x)  x  x^\top + (  v^*_j \phi'( w_j^\top  x) )^2  x  x^\top \right].
\end{align*}
\fi
If $\phi(z)$ is non-smooth, we use the Dirac function and its derivatives to represent $\phi''(z)$.
Replacing the expectation $\E_{(x,y)\sim \D}$ by the average over the samples $|S|^{-1} \sum_{( x,y)\in S}$, we obtain the Hessian of the empirical risk.

Considering the case when $W=W^* \in \mathbb{R}^{d\times k}$, for all $j,l\in[k]$, we have,
\begin{align*}
\frac{\partial^2 f_{\cal D}(W^*)}{\partial  w_j \partial  w_l} = & ~ \underset{( x,y)\sim{\cal D}}{\mathbb{E}} \left[  v^*_j v^*_l \phi'( w_j^{*\top}  x)  \phi'( w_l^{*\top}  x)   x  x^\top \right].% \in \mathbb{R}^{d\times d}.
\end{align*}
If Property~\ref{pro:hessian}(b) is satisfied, $\phi''(z) = 0$ almost surely. So in this case the diagonal blocks of the empirical Hessian can be written as, %w.p. $1$,
 \begin{align*}
\frac{\partial^2 \widehat{f}_S(W)}{\partial  w_j^2} = & ~\frac{1}{|S|}\sum_{( x,y)\in S} (v^*_j \phi'( w_j^\top  x))^2  x  x^\top.
\end{align*}
Now we show the Hessian of the objective near the global optimum is positive definite.
\begin{definition}\label{def:W_tau}
Given the ground truth matrix $W^*\in \mathbb{R}^{d\times k}$, let
$\sigma_i(W^*)$ denote the $i$-th singular value of $W^*$, often
abbreviated as $\sigma_i$. Let $\kappa = \sigma_1/ \sigma_k$,
$\lambda= (\prod_{i=1}^k \sigma_i) / \sigma_{k}^{k}$. Let $v_{\max}$
denote $\max_{i\in[k]} |v_i^*|$ and $v_{\min}$
denote $\min_{i\in[k]} |v_i^*|$ . Let $\nu=v_{\max}/v_{\min}$.  Let $\rho$ denote $\rho(\sigma_k)$. Let $\tau=  (3\sigma_1/2)^{4p} /
\min_{\sigma\in [\sigma_k/2,3\sigma_1/2 ]}\{\rho^2(\sigma)\} $. 
\end{definition}

\ifdefined\icmlcamera
\begin{theorem}
\else
\begin{theorem}[Informal version of Theorem~\ref{thm:lsc_nn}]
\fi
\label{thm:lsc_nn_informal}
For any $W\in \mathbb{R}^{d\times k}$ with $\|W - W^*\| \leq
\poly(1/k,$ $1/\lambda, 1/\nu, \rho/\sigma_1^{2p} ) \cdot \| W^*\|$, let $S$ denote a set of i.i.d. samples from distribution ${\cal D}$ (defined in~(\ref{eq:model})) and let the activation function satisfy Property~\ref{pro:gradient},\ref{pro:expect},\ref{pro:hessian}. Then for any $t\geq 1$, if $|S| \geq d\cdot \poly(\log d, t,$ $k, \nu, \tau,\lambda,\sigma_1^{2p}/\rho)$, we have with probability at least $1-d^{-\Omega(t)}$,
\begin{align*}
 \Omega( v_{\min}^2 \rho(\sigma_k) / (\kappa^2 \lambda ) ) I\preceq \nabla^2 \widehat{f}_S(W) \preceq O(kv_{\max}^2\sigma_1^{2p}) I.
\end{align*}
\end{theorem}

\begin{remark} As we can see from Theorem~\ref{thm:lsc_nn_informal}, 
$\rho(\sigma_k)$ from Property~\ref{pro:expect} plays an important role
for positive definite (PD) property. Interestingly, many popular activations, like ReLU,
sigmoid and tanh, have $\rho(\sigma_k)>0$, while some simple functions
like linear ($\phi(z)=z$) and square ($\phi(z)=z^2$) functions have $\rho(\sigma_k) = 0$ and their Hessians are rank-deficient. Another
important numbers are $\kappa$ and $\lambda$, two different condition numbers of the weight
matrix, which directly influences the positive definiteness. If $W^*$ is rank deficient, $\lambda \to \infty$, $\kappa \to \infty$ and we don't have PD property. In the best case when $W^*$ is orthogonal, $\lambda=\kappa=1$. In the worse case, $\lambda$ can be exponential in $k$. Also $W$ should be close enough to $W^*$. In the next section, we provide tensor methods to initialize $w_i^*$ and $ v_i^*$ such that they satisfy the conditions in Theorem~\ref{thm:lsc_nn_informal}. 
\end{remark}

For the PD property to hold, we need the samples to be independent of the current parameters. Therefore, we need to do resampling at each iteration to guarantee the convergence in iterative algorithms like gradient descent. The following theorem provides the linear convergence guarantee of gradient descent for smooth activations. 
\ifdefined\icmlcamera
\begin{theorem}[Linear convergence of gradient descent]
\else
\begin{theorem}[Linear convergence of gradient descent, informal version of Theorem~\ref{thm:lc_gd}]
\fi
\label{thm:lc_gd_informal}
Let $W$ be the current iterate satisfying $\|W - W^*\| \leq \poly(1/\nu, 1/k,1/\lambda, \rho/\sigma_1^{2p}) \| W^* \|$.
%v_{\min}^4 \rho^2(\sigma_k) \sigma_k/ ( k^2\kappa^4 {\lambda}^2 v_{\max}^4 \sigma_1^{4p} )$.
 Let $S$ denote a set of i.i.d. samples from distribution ${\cal D}$ (defined in~(\ref{eq:model})) with $|S| \geq d \cdot \poly(\log d, t,k,\nu,\tau,\lambda,\sigma_1^{2p}/\rho)$
 %\widetilde{\Omega}(d\cdot t^{p} k^2v_{\max}^4 \tau \kappa^8 \lambda^2 \sigma_1^{4p}/(v_{\min}^4 \rho^2(\sigma_k)))$ 
  and let the activation function satisfy Property~\ref{pro:gradient},\ref{pro:expect} and \ref{pro:hessian}(a). 
Define $m_0 := \Theta( v_{\min}^2 \rho(\sigma_k)/ (\kappa^2 \lambda) )$ and $M_0:= \Theta ( kv_{\max}^2 \sigma_1^{2p} )$. If we perform gradient descent with step size $1/M_0$ on $\widehat{f}_S(W)$ and obtain the next iterate,
$$\wt{W} = W - \frac{1}{M_0} \nabla \widehat{f}_S(W),$$ 
then with probability at least $1-d^{-\Omega(t)}$,
$$\|\wt{W} - W^*\|_F^2  \leq  (1- \frac{m_0}{M_0} ) \| W-W^*\|_F^2.$$
\end{theorem}

\ifdefined\icmlcamera
Due to the space limitation, we provide the proofs in the full version.
\else
We provide the proofs in the Appendix~\ref{app:convex_main_result}
\fi

%%%%%%%%%%%%%% SECTION %%%%%%%%%%%%%% 

\section{Tensor Methods for Initialization}\label{sec:tensor}
In this section, we show that Tensor methods can recover the parameters $W^*$ to some precision and exactly recover $ v^*$ for homogeneous activations.% with finite number of samples.

It is known that most tensor problems are NP-hard \cite{h90,hl13} or even hard to approximate \cite{swz17b}. However, by making some assumptions, tensor decomposition method becomes efficient \cite{agh14,wsta15,wa16,swz16}. Here we utilize the noiseless assumption and Gaussian inputs assumption to show a provable and efficient tensor methods.

\subsection{Preliminary}
%We assume they are not zero and $\gamma_0 \neq \gamma_2$. This is satisfied by ReLU and squared ReLU.

%Let's define a special outer product for simplification of the notation. For two matrices $M,N$, define $M \tilde \otimes N$ as the sum of the outer products between $M,N$ of all the possible permutations. Specifically, $(M \tilde \otimes N)_{i,j,k,l} := \sum_{\text{all permutations of}\; \{i,j,k,l\}} M_{i,j}N_{k,l}$. If both $M,N$ are non-symmetric, there are $12$ permutations totally. If they are symmetric matrices, there are 6 permutations. If they are equal to each other and symmetric as well, there are $3$ permutations. This definition can also be extended to the product between other order tensors, or even different orders of tensors. We give the following examples. 
Let's define a special outer product $\wt{\otimes}$ for simplification of the notation. If $ v \in \R^d$ is a vector and $I$ is the identity matrix, then
$ v \wt{\otimes} I = \sum_{j=1}^d [ v\otimes e_j\otimes e_j+ e_j\otimes v\otimes e_j+ e_j \otimes e_j \otimes v] .$
If $M$ is a symmetric rank-$r$ matrix factorized as $M = \sum_{i=1}^r s_i v_i v_i^\top$ and $I$ is the identity matrix, then
\begin{align*}
 M \wt{\otimes} I = \sum_{i=1}^r s_i \sum_{j=1}^d \sum_{l=1}^6 A_{l,i,j},
\end{align*}
where $A_{1,i,j} = v_i \otimes  v_i \otimes e_j \otimes e_j$, $A_{2,i,j} = v_i \otimes e_j \otimes v_i \otimes e_j$, $A_{3,i,j}= e_j \otimes v_i \otimes v_i \otimes e_j $, $A_{4,i,j}= v_i \otimes e_j \otimes e_j  \otimes  v_i$, $A_{5,i,j}= e_j \otimes v_i \otimes e_j  \otimes v_i$ and $A_{6,i,j} = e_j \otimes e_j  \otimes v_i \otimes  v_i $.

%Another example is for two identical symmetric matrices,
%\begin{align*}
%I \tilde \otimes I =  \sum_{i=1}^d\sum_{j=1}^d \sum_{l=1}^3 B_{l,i,j},
%\end{align*}
%where $B_{1,i,j} = e_j \otimes e_j \otimes e_i \otimes e_i$, $B_{2,i,j} = e_j \otimes e_i \otimes e_j \otimes e_i$, $B_{3,i,j} = e_i \otimes e_j \otimes e_j \otimes e_i $. 
Denote $\ov{ w} = w/\| w\|$.
Now let's calculate some moments. 
\begin{definition}\label{def:M_m}
We define $M_1,M_2,M_3,M_4$ and $m_{1,i}, m_{2,i}, m_{3,i}, m_{4,i}$ as follows : \\
 $M_1 = \E_{ (x,y) \sim \D }[y \cdot x]$.\\
 $M_2=  \E_{ (x,y) \sim \D }[y \cdot (x\otimes x -I)]$. \\
 $M_3 = \E_{ (x,y) \sim \D }[y \cdot ( x^{\otimes 3} - x \wt{\otimes} I)]$.\\
 $M_4 = \E_{ (x,y) \sim \D }[y \cdot ( x^{\otimes 4} -  ( x\otimes x) \wt{\otimes} I +   I \tilde \otimes I)]$.\\
 $\gamma_j(\sigma) = \E_{z\sim \N(0,1)}[ \phi(\sigma\cdot z)z^j], \; \forall j=0,1,2,3,4.$ \\
 $m_{1,i} = \gamma_1(\| w_i^*\|) $.\\
 $m_{2,i} = \gamma_2(\| w_i^*\|) -\gamma_0(\| w_i^*\|)$.\\
 $m_{3,i} = \gamma_3 ( \| w_i^*\|)- 3\gamma_1( \|w_i^*\|)$.\\
 $m_{4,i} =  \gamma_4(\| w_i^*\|)+3\gamma_0(\| w_i^*\|) - 6\gamma_2(\| w_i^*\|)$.
\end{definition}

\begin{comment}
{%\small
\begin{align}
%M_0 & := \mathbb{E}[y] =  \sum_{i=1}^k v^*_i \gamma_0(\| w_i^*\|). \\
 M_1 &= \E_{ (x,y) \sim \D }[y \cdot x] =  \sum_{i=1}^k v^*_i  \gamma_1(\|w_i^*\|) \ov{w}_i^*. \label{eq:1st_moment}
 \end{align}
 \begin{align}
 M_2 &=  \E_{ (x,y) \sim \D }[y \cdot (x\otimes x -I)] \notag \\
  &  = \sum_{i=1}^k v^*_i (\gamma_2(\|w_i^*\|) -\gamma_0(\| w_i^*\|)) \ov{ w}_i^{*\otimes 2}. \label{eq:2nd_moment} 
 \end{align}
 \begin{align}
 M_3 &= \E_{ (x,y) \sim \D }[y \cdot ( x^{\otimes 3} - x \wt{\otimes} I)] \notag \\
  & =  \sum_{i=1}^k  v_i^* (\gamma_3 ( \| w_i^*\|)- 3\gamma_1( \| w_i^*\|)) \ov{ w}_i^{* \otimes 3}. \label{eq:3rd_moment}
  \end{align}
\begin{align*}
M_4 & = \E_{ (x,y) \sim \D }[y \cdot ( x^{\otimes 4} -  ( x\otimes x) \wt{\otimes} I +   I \tilde \otimes I)] \notag \\
 & =  \sum_{i=1}^k v_i^* (\gamma_4(\| w_i^*\|)+3\gamma_0(\| w_i^*\|) - 6\gamma_2(\| w_i^*\|))    \ov{ w}_i^{*\otimes 4} .%\label{eq:4th_moment}
\end{align*}
}
Define %$m_{0,i} = \gamma_0(\| w_i^*\|)$, 
$m_{1,i} = \gamma_1(\| w_i^*\|) $, 
$ m_{2,i} = \gamma_2(\| w_i^*\|) -\gamma_0(\| w_i^*\|)$, $m_{3,i} = \gamma_3 ( \| w_i^*\|)- 3\gamma_1( \|w_i^*\|)$ and $m_{4,i} =  \gamma_4(\| w_i^*\|)+3\gamma_0(\| w_i^*\|) - 6\gamma_2(\| w_i^*\|)$. Therefore, $M_j = \sum_{i=1}^k v_i^* m_{j,i} \ov{w}_i^{*\otimes j}$, $\forall j=1,2,3,4$.
\end{comment}
According to Definition~\ref{def:M_m}, we have the following results,
\begin{claim}\label{cla:M_m}
For each $j\in [4]$, $ M_j = \sum_{i=1}^k v_i^* m_{j,i} \ov{w}_i^{*\otimes j}$.
\end{claim}

Note that some $m_{j,i}$'s will be zero for specific activations. For example, for activations with symmetric first derivatives, i.e., $\phi'(z) = \phi'(-z)$, like sigmoid and erf, we have $\phi(z)+\phi(-z)$ being a constant and $M_2 = 0$ since $\gamma_0(\sigma) = \gamma_2(\sigma)$. Another example is \ReLU. \ReLU~functions have vanishing $M_3$, i.e., $M_3=0$, as $\gamma_3(\sigma) =3 \gamma_1(\sigma)$. To make tensor methods work, we make the following assumption.%So we assume that at least one of $M_3,M_4$ is non-zero and has rank-$k$. 
\begin{assumption}\label{assumption:non_zero_moment}
Assume the activation function $\phi(z)$ satisfies the following conditions:\\%, given $\{ w_1^*, w_2^*,\cdots, w_k^*\}$,
1. If $M_j\neq 0 $, then $m_{j,i} \neq 0$ for all $i\in [k]$. \\
2. At least one of $M_3$ and $M_4$ is non-zero.\\
3. If $M_1= M_3=0$, then $\phi(z)$ is an even function, i.e., $\phi(z) = \phi(-z)$.\\%\label{assume:13}
4. If $M_2= M_4=0$, then $\phi(z)$ is an odd function, i.e., $\phi(z) = - \phi(-z)$.\\%\label{assume:24}
%\end{enumerate}
\end{assumption}
%This assumes that 1) at least one of $M_1,M_2$ is non-zero and has rank-$k$; 2)  at least one of $M_3,M_4$ is non-zero and has rank-$k$. It ensures us there is enough information to recover $ v^*$ and $W^*$. 
If $\phi(z)$ is an odd function then $\phi(z) = - \phi(-z)$ and $v\phi( w^\top  x) = -v\phi(- w^\top  x)$. Hence we can always assume $v>0$. If $\phi(z)$ is an even function, then $v\phi( w^\top  x) = v\phi(- w^\top  x)$. So if $ w$ recovers $ w^*$ then $- w$ also recovers $ w^*$.  Note that \ReLU, leaky \ReLU~and squared \ReLU~satisfy Assumption~\ref{assumption:non_zero_moment}. We further define the following  non-zero moments. 
\begin{definition}\label{def:P2_P3}
Let $\alpha \in \R^{d}$ denote a randomly picked vector. We define $P_2$ and $P_3$ as follows: $P_{2} = M_{j_2}(I,I,  \alpha,\cdots, \alpha) $ , %\label{eq:def_P2}
where $j_2 = \min \{j\geq 2| M_{j}\neq 0\}$ and
$
P_{3} =  M_{j_3}(I,I,I,\alpha,\cdots, \alpha) $, %\label{eq:def_P3}
where $j_3 =  \min \{j\geq 3| M_{j}\neq 0\} $.
\end{definition}
According to Definition~\ref{def:M_m} and \ref{def:P2_P3}, we have,
\begin{claim}\label{cla:P2_P3}
$P_2 = \sum_{i=1}^k v_i^* m_{j_2,i} ( \alpha^\top \ov{ w}_i^*)^{{j_2-2}} \ov{ w}_i^{*\otimes 2}$ and
$P_3=\sum_{i=1}^k v_i^* m_{j_3,i}( \alpha^\top \ov{ w}_i^*)^{{j_3-3}} \ov{ w}_i^{*\otimes 3}$.
\end{claim}
% and $V\in \R^{d\times k}$ is an estimation of the orthogonal basis of the subspace spanned by $\{ w_1^*, w_2^*,\cdots, w_k^*\}$. 
In other words for the above definition, $P_{2}$ is equal to the first non-zero matrix in the ordered sequence $\{M_2, M_3(I,I, \alpha), M_4(I,I, \alpha, \alpha)\}$. $P_{3}$ is equal to the first non-zero tensor in the ordered sequence $\{M_3, M_4(I,I,I, \alpha)\}$. Since $ \alpha$ is randomly picked up, $ w_i^{*\top}  \alpha\neq 0$ and we view this number as a constant throughout this paper. So by construction and Assumption~\ref{assumption:non_zero_moment}, both $P_{2}$ and $P_{3}$ are rank-$k$.
%Let's denote $P_2 = \sum_{i\in[k]} r_i^* \ov{ w}_i^{*\otimes 2}$, $P_3 = \sum_{i\in[k]} o^{(3)}_i \ov{ w}_i^{*\otimes 3}$. So by construction and Assumption~\ref{assumption:tensor_act}, $|o^{(j)}_i| >0$ for all $i\in[k],j=2,3$.
Also, let $\wh{P}_2 \in \mathbb{R}^{d\times d}$ and $\wh{P}_3 \in \mathbb{R}^{d\times d\times d}$ denote the corresponding empirical moments of $P_2 \in \mathbb{R}^{d\times d}$ and $P_3 \in \mathbb{R}^{d\times d\times d}$ respectively.
%Define $R_3 := P_3(V,V,V)$. Also, let $\wh{P}_2,\wh{P}_3,\wh{R}_3$ denote the corresponding empirical moments of $P_2,P_3,R_3$ respectively.
%In the next section, we will show how to estimate the subspace $V$ spanned by $\{ w_i^*\}_{i=1,2,\cdots,k}$  from $\wh{P}_2$ and how to estimate $\{V^\top \ov{ w}_i^*\}_{i=1,2,\cdots,k}$ from $\wh{R}_3$. 

%For \ReLU, we have the following, %$m_{1,i} =\frac{1}{\sqrt{2\pi}}v_i^* \| w_i^*\|/2$, 
%$m_{2,i} = \frac{1}{\sqrt{2\pi}} v^*_i \| w_i^*\|$, $m_{3,i} =0$ and 
%$m_{4,i} = - \frac{1}{\sqrt{2\pi}} v^*_i \| w_i^*\|$. For squared \ReLU, we have $m_{3,i} = \frac{2}{\sqrt{2\pi}} v^*_i \| w_i^*\|$. Sigmoid and tanh are (biased) odd functions, so they have $|m_{3,i}| >0$

\subsection{Algorithm}
Now we briefly introduce how to use a set of samples with size linear
in dimension to recover the ground truth parameters to some precision.
As shown in the previous section, we have a rank-$k$ 3rd-order
moment $P_3$ that has tensor decomposition formed by $\{\ov{ w}_1^*,\ov w_2^*,\cdots,\ov w_k^*\}$. Therefore, we can use the non-orthogonal decomposition method \cite{kuleshov2015tensor} to decompose the corresponding estimated tensor $\widehat P_3$ and obtain
an approximation of the parameters. %If the 3rd-order tensor is zero, we can form a 3rd-order tensor from the 4th-order tensor, i.e., $R_3 = M_4(I,I,I, \alpha)$, with a fixed vector $ \alpha$. 
The precision of the obtained parameters depends on the estimation error of $P_3$, which requires $\Omega(d^3/\epsilon^2)$ samples to achieve $\epsilon$ error. Also, the time complexity for tensor decomposition on a $d\times d \times d$ tensor is $\Omega(d^3)$. 
%To achieve a certain precision for the recovered parameters, we need as many samples as possible to estimate the population tensors. 

In this paper, we reduce the cubic dependency of sample/computational complexity in dimension
\cite{jsa15} to linear dependency. 
%The cubic dependency mainly comes from the approximation of a third-order tensor which contains $d\times d\times d$ entries. 
%To the best of our knowledge, to approximate a tensor, we first flat the tensor into a matrix and then use matrix Chernoff/Bernstein inequality to bound the errors, which leads to cubic dependency. 
Our idea follows the techniques used in \cite{zjd16}, where they first used a 2nd-order moment $P_2$ to approximate the subspace spanned by $\{\ov{ w}_1^*,\ov w_2^*,\cdots,\ov w_k^*\}$, denoted as $V$, 
%(when the 2nd-order moment is zero, we can use a higher-order tensor to generate a 2nd order moment, e.g., the eigenspace of$M_3(I,I, \alpha)$ for some fixed random vector $ \alpha$ also covers the subspace spanned by $\{ w_i^*\}_{i\in[k]}$). 
then use $V$ to reduce a higher-dimensional third-order tensor $P_3 \in\R^{d\times d\times d}$ to a lower-dimensional tensor $R_3 := P_3(V,V,V) \in \R^{k\times k\times k}$. Since the tensor decomposition and the tensor estimation are conducted on a lower-dimensional $\R^{k\times k\times k}$ space, the sample complexity and computational complexity are reduced. 

The detailed algorithm is shown in Algorithm~\ref{alg:tensor}. First, we randomly partition the dataset into three subsets each with size $\tilde O(d)$. Then apply the power method on $\wh{P}_2$, which is the estimation of $P_2$ from $S_2$, to estimate $V$. After that, the non-orthogonal tensor decomposition (KCL)\cite{kuleshov2015tensor} on $\wh{R}_3$ outputs $ \wh{u}_i$ which estimates $s_iV^\top \ov{ w}_i^*$ for $i\in [k]$ with unknown sign $s_i\in\{-1,1\}$. Hence $\ov{ w}_i^*$ can be estimated by $s_iV \wh{u}_i$. Finally we estimate the magnitude of $ w_i^*$ and the signs $s_i, v_i^*$ in the \textsc{RecMagSign} function for homogeneous activations. 
\ifdefined\icmlcamera
We discuss the details of each procedure and provide \textsc{PowerMethod} and \textsc{RecMagSign} algorithms in the full version. 
\else
We discuss the details of each procedure and provide \textsc{PowerMethod} and \textsc{RecMagSign} algorithms in Appendix~\ref{app:tensor}. 
\fi

\begin{algorithm}[t]
\caption{Initialization via Tensor Method}
\label{alg:tensor}
\begin{algorithmic}[1]
\Procedure{\textsc{Initialization}}{$S$} \Comment{Theorem~\ref{thm:tensor_final}}
\State $S_2,S_3,S_4\leftarrow \textsc{Partition}(S,3)$
%\State $V\leftarrow [\ov{w}_1^*, \ov{w}_2^*,\cdots, \ov{w}_k^*]$
%Partition $S$ into $S = S_2 \cup S_3 \cup S_4$ each with size $\wt{\Omega}(d)$.
\State $\wh{P}_2 \leftarrow \E_{S_2}[P_2]$
\State $V\leftarrow \textsc{PowerMethod}(\wh{P}_2,k)$
%\State Compute the subspace estimation of $\text{span}(\{ w_i^*\}_{i\in[k]})$, $V$, using $\wh{P}_2$, an estimation of $P_2$ from $S_2$. %\Comment{ Algorithm~\ref{alg:power_method}}
\State $\wh{R}_3 \leftarrow \E_{S_3} [ P_3(V,V,V) ]$
%Compute the empirical $R_3 = P_3(V,V,V)$, $\wh{R}_3$, from $S_3$.
\State $ \{ \wh{u}_i\}_{i\in[k]}  \leftarrow \textsc{KCL}(\wh{R}_3)$ 
%Compute the eigenvectors $\{ \wh{u}_i\}_{i=1}^k$ of the tensor $\wh{R}_3 \in \R^{k\times k\times k}$ by using non-orthogonal tensor decomposition methods \cite{kuleshov2015tensor}.\label{rtpm_alg}
\begin{comment}
\State Partition the dataset $S$ into $S = S_2 \cup S_3 \cup S_4$ each with size $\wt{\Omega}(d)$.
\State Compute the subspace estimation of $\text{span}(\{ w_i^*\}_{i=1,2,\cdots,k})$, $V$, from $\wh{P}_2$, an estimation of $P_2$ using $S_2$, via power method. %\Comment{ Algorithm~\ref{alg:power_method}}
\State  Compute the empirical $R_3 = P_3(V,V,V)$, $\wh{R}_3$, from $S_3$.
\State Compute the eigenvectors $\{ \wh{u}_i\}_{i=1}^k$ of the tensor $\wh{R}_3 \in \R^{k\times k\times k}$ by using non-orthogonal tensor decomposition methods \cite{kuleshov2015tensor}.\label{rtpm_alg}
\end{comment}
%\State Obtain the estimation of the normalized $ w_i^*$ with an unknown sign, i.e., $\ov{ w}_i^*\approx s_i V  \wh{u}_i $ with some $s_i \in \{-1,1\}$.
\State {\small $\{ w_i^{(0)},v_i^{(0)}\}_{i\in[k]} \leftarrow $ \textsc{RecMagSign}$(V,\{ \wh{u}_i\}_{i\in[k]}, S_4)$} %\Comment{Algorithm~\ref{alg:recover} } 
\State {\bf Return} $\{ w_i^{(0)},v_i^{(0)}\}_{i \in[k]}$
\EndProcedure 
\end{algorithmic}
\end{algorithm}

\begin{algorithm}[t]
\caption{Globally Converging Algorithm}
\label{alg:overall_alg}
\begin{algorithmic}[1]
\Procedure{\textsc{Learning1NN}}{$S$, $d$, $k$, $\epsilon$} \Comment{Theorem~\ref{thm:overall}}
\State $T \leftarrow \log(1/\epsilon)\cdot \poly(k,\nu,\lambda,\sigma_1^{2p}/\rho)$.
\State $\eta \leftarrow 1/ (kv^2_{\max}\sigma_1^{2p})$.
\State $S_0,S_1,\cdots,S_q\leftarrow \textsc{Partition}(S,q+1)$.
%\State Partition $S$ into $S = \cup_{q=0}^T S_q$ each with size $\wt{\Omega}(d)$.
\State $W^{(0)}, v^{(0)} \leftarrow \textsc{Initialization}(S_0)$. % \Comment{ Algorithm~\ref{alg:tensor}}
\State Set $v_i^* \leftarrow v_i^{(0)}$ in Eq.~\eqref{eq:nns_emp_obj} for all $\widehat f_{S_q}(W)$, $q\in [T]$ 
\For{$q=0,1,2,\cdots, T-1$}
\State $W^{(q+1)} = W^{(q)} - \eta \nabla \widehat{f}_{S_{q+1} }(W^{(q)})$
\EndFor
\State {\bf Return }  $\{ w_i^{(T)},v_i^{(0)}\}_{i\in[k]} $
\EndProcedure 
\end{algorithmic}
\end{algorithm}

\subsection{Theoretical Analysis}
\ifdefined\icmlcamera
We formally present our theorem for Algorithm~\ref{alg:tensor}, and provide the proof in the full version. 
\else
We formally present our theorem for Algorithm~\ref{alg:tensor}, and provide the proof in the Appendix~\ref{app:tensor_main_result}.
\fi

\define{thm:tensor_final}{Theorem}{
Let the activation function be homogeneous satisfying Assumption~\ref{assumption:non_zero_moment}. For any $0<\epsilon < 1$ and $t\geq 1$, if $|S|\geq \epsilon^{-2} \cdot 
d\cdot \poly(t, k, \kappa,\log d) $, then there exists an algorithm (Algorithm~\ref{alg:tensor}) that takes $|S| k \cdot \wt{O}(d)$ time and  outputs a matrix $W^{(0)}\in \mathbb{R}^{d\times k}$ and a vector $ v^{(0)}\in \mathbb{R}^{k}$ such that, with probability at least $1-d^{-\Omega(t)}$,
\begin{align*}
\|W^{(0)} - W^*\|_F \leq \epsilon \cdot \poly(k, \kappa) \| W^* \|_F , \mathrm{~and~} v_i^{(0)} = v_i^*.
\end{align*}
}
\state{thm:tensor_final}

\begin{figure*}[!t]
\begin{tabular}{cccc}
\includegraphics[width=.30\textwidth]{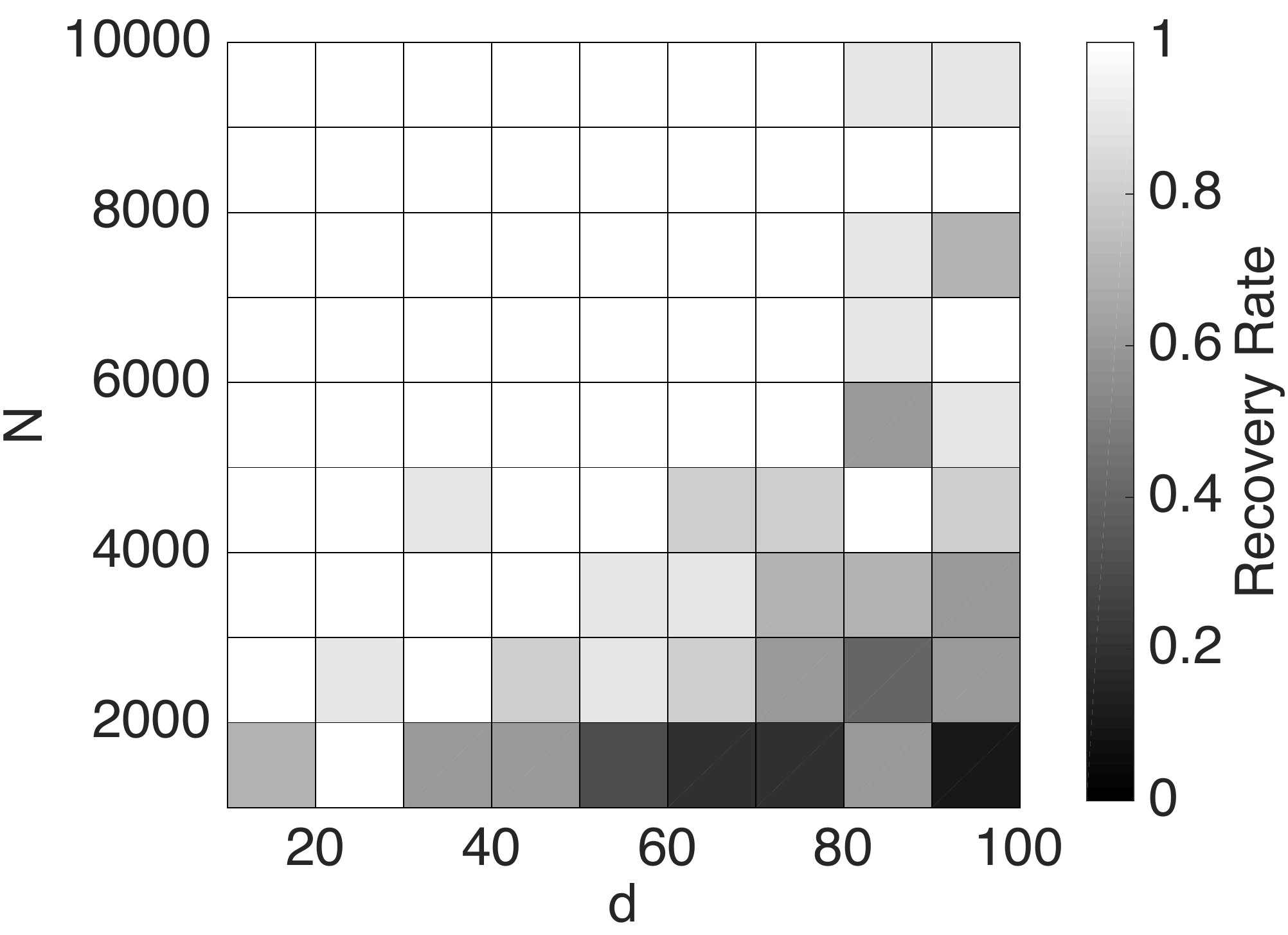} &
\includegraphics[width=.30\textwidth]{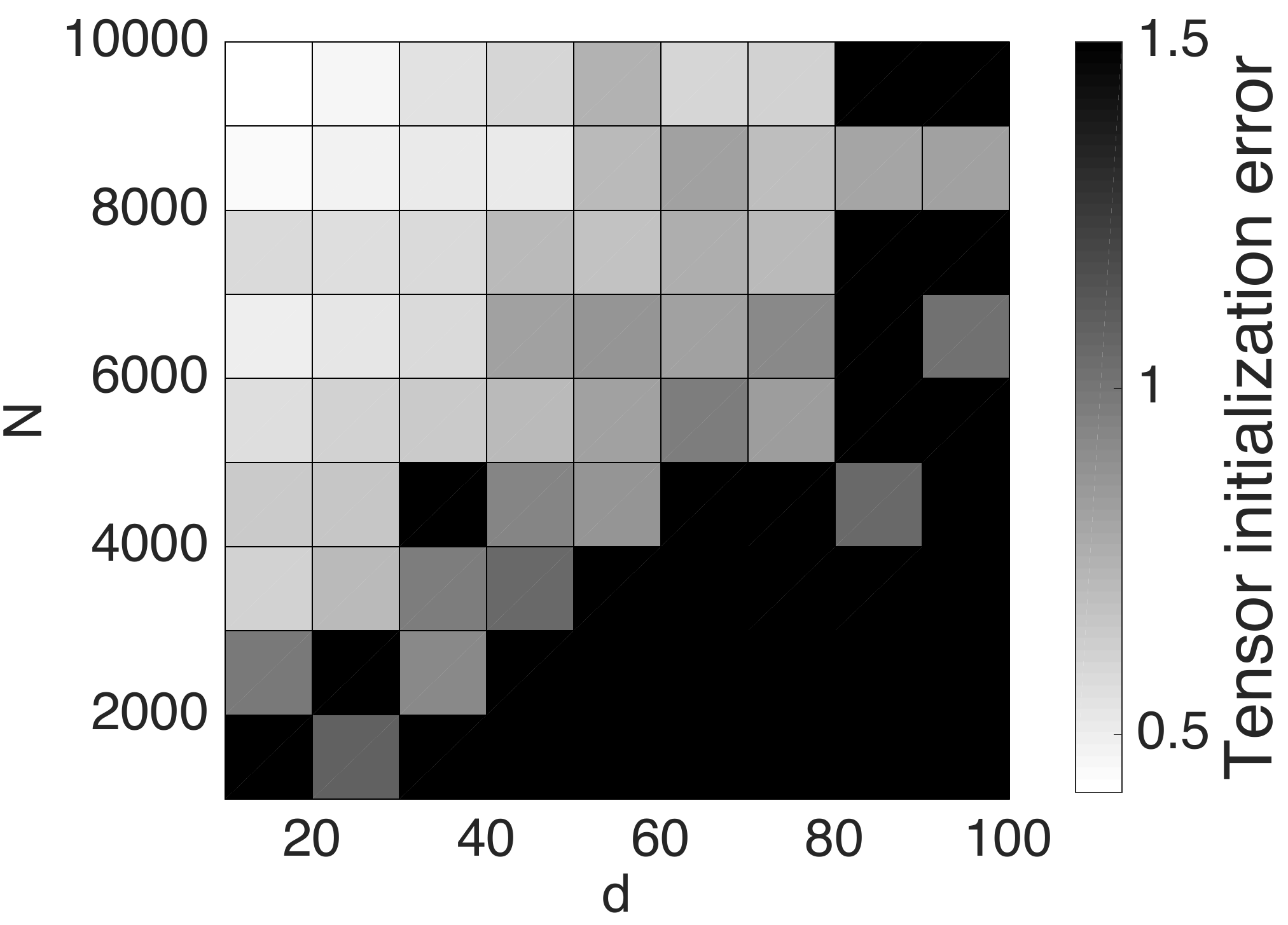} &
\includegraphics[width=.28\textwidth]{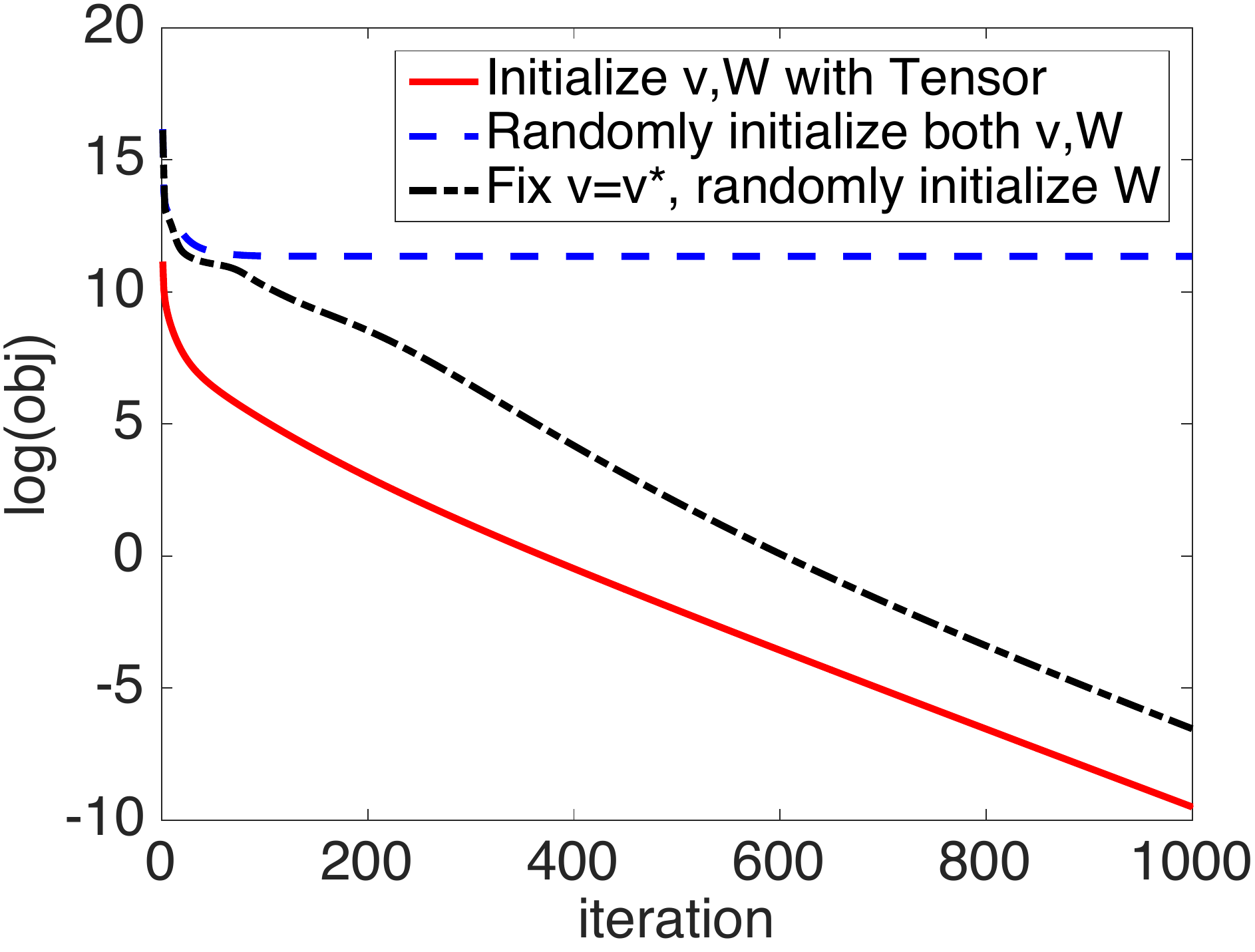}  \\
(a) Sample complexity for recovery & (b) Tensor initialization error  &(c) Objective v.s. iterations
\end{tabular}
\caption{Numerical Experiments}
\label{fig:performance}
\end{figure*}

\section{Global Convergence}
Combining the positive definiteness of the Hessian near the global optimal in Section~\ref{sec:lsc_nns} and the tensor initialization methods in Section~\ref{sec:tensor}, we come up with the overall globally converging algorithm Algorithm~\ref{alg:overall_alg} and its guarantee Theorem~\ref{thm:overall}.

\begin{theorem}[Global convergence guarantees]\label{thm:overall}
Let $S$ denote a set of i.i.d. samples from distribution ${\cal D}$ (defined in~(\ref{eq:model})) and let the activation function be homogeneous satisfying Property~\ref{pro:gradient}, \ref{pro:expect}, \ref{pro:hessian}(a) and Assumption~\ref{assumption:non_zero_moment}. Then for any $t\geq 1$ and any $\epsilon>0$, if $|S| \geq d \log(1/\epsilon) \cdot \poly(\log d, t, k,\lambda )$, $T \geq \log(1/\epsilon)\cdot \poly(k,\nu,\lambda,\sigma_1^{2p}/\rho) $ and $0< \eta \leq 1/(kv_{\max}^2\sigma_1^{2p})$, then there is an Algorithm (procedure \textsc{Learning1NN} in Algorithm~\ref{alg:overall_alg}) taking $|S | \cdot d \cdot \poly(\log d, k,\lambda)  $ time and outputting a matrix $W^{(T)}\in \mathbb{R}^{d\times k}$ and a vector $v^{(0)}\in \mathbb{R}^{k}$ satisfying
\begin{align*}
\|W^{(T)} - W^*\|_F \leq \epsilon \| W^* \|_F, \mathrm{~and~} v_i^{(0)} = v_i^*.
\end{align*}
with probability at least $1-d^{-\Omega(t)}$.
\end{theorem}
This follows by combining Theorem~\ref{thm:lc_gd_informal} and Theorem~\ref{thm:tensor_final}. 

%%%%%%%%%%%%%% SECTION %%%%%%%%%%%%%% 

\section{Numerical Experiments}
In this section we use synthetic data to verify our theoretical results. We generate data points $\{x_i,y_i\}_{i=1,2,\cdots,n}$ from Distribution ${\cal D}$(defined in Eq.~\eqref{eq:model}). We set $W^* = U\Sigma V^\top$, where $U\in \R^{d\times k}$ and $V\in \R^{k\times k}$ are orthogonal matrices generated from QR decomposition of Gaussian matrices, $\Sigma$ is a diagonal matrix whose diagonal elements are $1,1+\frac{\kappa-1}{k-1},1+\frac{2(\kappa-1)}{k-1},\cdots ,\kappa$. In this experiment, we set $\kappa=2$ and $k=5$. We set $v^*_i$ to be randomly picked from $\{-1,1\}$ with equal chance. We use squared \ReLU~$\phi(z) = \max\{z,0\}^2$, which is a smooth homogeneous function. For non-orthogonal tensor methods, we directly use the code provided by \cite{kuleshov2015tensor} with the number of random projections fixed as $L=100$. We pick the stepsize $\eta = 0.02$ for gradient descent. In the experiments, we don't do the resampling since the algorithm still works well without resampling. 

First we show the number of samples required to recover the parameters for different dimensions. We fix $k=5$, change $d$ for $d = 10,20,\cdots,100$ and $n$ for $n=1000,2000,\cdots,10000$. For each pair of $d$ and $n$, we run $10$ trials. We say a trial successfully recovers the parameters if there exists a permutation $\pi : [k] \rightarrow [k]$, such that the returned parameters $W$ and $v$ satisfy 
\begin{align*}
\max_{j\in [k]} \{\|w_j^* - w_{\pi(j)}\|/\|w_j^*\|\} \leq 0.01 \text{~and~} v_{\pi(j)} = v_j^*.
\end{align*}
 We record the recovery rates and represent them as grey scale in Fig.~\ref{fig:performance}(a). As we can see from Fig.~\ref{fig:performance}(a), the least number of samples required to have 100\% recovery rate is about proportional to the dimension. 

Next we test the tensor initialization. We show the error between the output of the tensor method and the ground truth parameters against the number of samples under different dimensions in Fig~\ref{fig:performance}(b). The pure dark blocks indicate, in at least one of the 10 trials, $\sum_{i=1}^k v^{{(0)}}_i  \neq \sum_{i=1}^k v^{*}_i $, which means $v^{(0)}_i$ is not correctly initialized. Let $\Pi(k)$ denote the set of all possible permutations $\pi : [k]\rightarrow [k]$. The grey scale represents the averaged error, 
\begin{align*}
\min_{\pi \in \Pi(k)}\max_{j\in [k] } \{\|w_j^* - w_{\pi(j)}^{(0)}\|/\|w_j^*\|\},
\end{align*}
over 10 trials. As we can see, with a fixed dimension, the more samples we have the better initialization we obtain. We can also see that to achieve the same initialization error, the sample complexity required is about proportional to the dimension. 

We also compare different initialization methods for gradient descent in Fig.~\ref{fig:performance}(c). We fix $d=10,k=5,n=10000$ and compare three different initialization approaches, (\RN{1}) Let both $v$ and $W$ be initialized from tensor methods, and then do gradient descent for $W$ while $v$ is fixed; (\RN{2}) Let both $v$ and $W$ be initialized from random Gaussian, and then do gradient descent for both $W$ and $v$; (\RN{3}) Let $v = v^*$ and $W$ be initialized from random Gaussian, and then do gradient descent for $W$ while $v$ is fixed. 
As we can see from Fig~\ref{fig:performance}(c), Approach (\RN{1}) is the fastest and Approach (\RN{2}) doesn't converge even if more iterations are allowed. Both Approach (\RN{1}) and (\RN{3}) have linear convergence rate when the objective value is small enough, which verifies our local linear convergence claim.

\section{Conclusion}
As shown in Theorem~\ref{thm:overall}, the tensor initialization followed by gradient descent will provide a globally converging algorithm with linear time/sample complexity in dimension, logarithmic in precision and polynomial in other factors for smooth homogeneous activation functions. Our distilled properties for activation functions include a wide range of non-linear functions and hopefully provide an intuition to understand the role of non-linear activations played in optimization. 
%Our tensor methods currently only work for homogeneous activations. So one of our future works is to extend it to more general activations. 
Deeper neural networks and convergence for SGD will be considered in the future.

\newpage
\addcontentsline{toc}{section}{References}
\bibliography{ref}
\bibliographystyle{alpha}%{icml2017}

\clearpage
\newpage
\appendix

\section{Notation} 
For any positive integer $n$, we use $[n]$ to denote the set $\{1,2,\cdots,n\}$.
For random variable $X$, let $\mathbb{E}[X]$ denote the expectation of $X$ (if this quantity exists).
%Define $W := [ w_1, w_2,\cdots, w_K] \in \R^{K\times d}$ and $ v := [v_1;\cdots;v_K]$. 
%We use $\Expect{X}$ to denote the expectation of a random variable $X$.
For any vector $ x\in \mathbb{R}^n$, we use $\|  x\|$ to denote its $\ell_2$ norm.

We provide several definitions related to matrix $A$.
Let $\det(A)$ denote the determinant of a square matrix $A$. Let $A^\top$ denote the transpose of $A$. Let $A^\dagger$ denote the Moore-Penrose pseudoinverse of $A$. Let $A^{-1}$ denote the inverse of a full rank square matrix. Let $\| A\|_F$ denote the Frobenius norm of matrix $A$. Let $\| A\|$ denote the spectral norm of matrix $A$. Let $\sigma_i(A)$ to denote the $i$-th largest singular value of $A$. We often use capital letter to denote the stack of corresponding small letter vectors, e.g., $W = [ w_1\; w_2\;\cdots\; w_k]$. For two same-size matrices, $A,B \in \mathbb{R}^{d_1\times d_2}$, we use $A\circ B \in \mathbb{R}^{d_1\times d_2}$ to denote element-wise multiplication of these two matrices. 

%Let $T\in \R^{d\times d\times d}$ be a tensor and $T_{ijk}$ be the $i,j,k$-th entry of $T$.  We say a tensor is supersymmetric if $T_{i,j,l}$ is invariant under any permutation of $i,j,k$.
We use $\otimes$ to denote outer product and $\cdot$ to denote dot product. Given two column vectors $ u, v \in \mathbb{R}^n$, then $  u\otimes  v  \in \mathbb{R}^{n\times n}$ and $( u \otimes  v)_{i,j} =  u_i \cdot  v_j$, and $ u^\top  v= \sum_{i=1}^n  u_i  v_i \in \mathbb{R}$. Given three column vectors $ u, v, w\in \mathbb{R}^n$, then $ u\otimes  v\otimes  w\in \mathbb{R}^{n\times n\times n}$ and $( u\otimes  v\otimes  w)_{i,j,k} =  u_i \cdot  v_j \cdot  w_k$. We use $ u^{\otimes r} \in \mathbb{R}^{n^r}$ to denote the vector $ u$ outer product with itself $r-1$ times.

Tensor $T\in \mathbb{R}^{n\times n\times n}$ is symmetric if and only if for any $i,j,k$, $T_{i,j,k} = T_{i,k,j} = T_{j,i,k} = T_{j,k,i}=T_{k,i,j}$ $= T_{k,j,i}$. %We say tensor $T$ has rank $k$, if $k$ is the smallest number such that there exists $3k$ vectors and $T=\sum_{i=1}^k  a_i \otimes \bb_i \otimes \bc_i$. 
Given a third order tensor $T\in \mathbb{R}^{n_1\times n_2\times n_3}$ and three matrices $A\in \mathbb{R}^{n_1 \times d_1},B \in \mathbb{R}^{n_2 \times d_2},C\in \mathbb{R}^{n_3\times d_3}$, we use $T(A,B,C)$
 to denote a $d_1 \times d_2 \times d_3$ tensor where the $(i,j,k)$-th entry is,
 \begin{align*}
\overset{n_1}{\underset{i'=1}{\sum}} \overset{n_2}{\underset{j'=1}{\sum}}  \overset{n_3}{\underset{k'=1}{\sum}}  T_{i',j',k'} A_{i',i}B_{j',j}C_{k',k}. 
 \end{align*}
 %If $T = \sum_{i=1}^r \lambda_i  u_i^{\otimes 3}$ with $\lambda_i \neq 0$ for $i\in [r]$ and linearly independent $\{ u_i\}_{i=1,2,\cdots,r}$, then we say $T$ is rank-$r$.
  We use $\|T\|$ to denote the operator norm of the tensor $T$, i.e., 
\begin{align*}
\|T\| = \underset{\| a\|=1}{\max} |T( a, a, a)|.
\end{align*} %We use $\sigma_r(T)$ to denote $\sigma_r(T) = \underset{{\| a\|=1, a\in \text{span}\{\{ u_i\}_{i\in[r]}\}}}{\min} |T( a, a, a)|$.

For tensor $T\in \mathbb{R}^{n_1\times n_2 \times n_3}$, we use matrix
$T^{(1)}\in \mathbb{R}^{n_1\times n_2n_3}$ to denote the flattening of
tensor $T$ along the first dimension, i.e., $ [T^{(1)}]_{i,(j-1)n_3+k} = T_{i,j,k}, \forall i\in [n_1], j\in [n_2], k\in [n_3]$. Similarly for matrices $T^{(2)} \in \mathbb{R}^{n_2\times n_3n_1}$ and $T^{(3)} \in \mathbb{R}^{n_3\times n_1 n_2}$.  
 %We also use the same notation $T$ to denote the multi-array map from three matrices,
 %We say a tensor $T$ is rank-one if $T =  a\otimes \bb \otimes \bc$, where $T_{ijk} = a_ib_jc_k$. 

  %Throughout the paper, we use $\Otilde(d)$ to denote $O(d\times \text{polylog}(d))$. 

We use $\bone_{f}$ to denote the indicator function, which is $1$ if
$f$ holds and $0$ otherwise. Let $I_d \in \mathbb{R}^{d\times d}$
denote the identity matrix. We use $\phi(z)$ to denote an activation
function. We define $(z)_+:= \max \{ 0,z\}$. We use ${\cal D}$ to
denote a Gaussian distribution ${\cal N}(0,I_d)$ or to denote a joint
distribution of $(X,Y)\in \mathbb{R}^d\times\mathbb{R}$, where the
marginal distribution of $X$ is ${\cal N}(0,I_d)$.

For any function $f$, we define $\widetilde{O}(f)$ to be $f\cdot \log^{O(1)}(f)$. In addition to $O(\cdot)$ notation, for two functions $f,g$, we use the shorthand $f\lesssim g$ (resp. $\gtrsim$) to indicate that $f\leq C g$ (resp. $\geq$) for an absolute constant $C$. We use $f\eqsim g$ to mean $cf\leq g\leq Cf$ for constants $c,C$.

\section{Preliminaries}
In this section, we introduce some lemmata and corollaries that will be used in the proofs.

%\begin{theorem}[Matrix Berinstein (See e.g., Theorem 1.6.2 in \cite{})]

%\end{theorem}
\subsection{Useful Facts}
We provide some facts that will be used in the later proofs.

\begin{fact}\label{fac:inner_prod_bound}
Let $z$ denote a fixed $d$-dimensional vector, then for any $C \geq 1$ and $n\geq 1$, we have
\begin{align*}
\underset{ x\sim {\cal N}(0,I_d) }{ \Pr } [  | \langle x , z \rangle |^2 \leq 5C \| z\|^2 \log n ]  \geq 1-1/(nd^C).
\end{align*}
\end{fact}
\begin{proof}
This follows by Proposition 1.1 in \cite{hkz12}.
\end{proof}

\begin{fact}\label{fac:gaussian_norm_bound}
For any $C\geq 1$ and $n\geq 1$, we have
\begin{align*}
\underset{ x\sim {\cal N}(0,I_d) }{ \Pr } [  \| x \|^2 \leq 5C d \log n ]  \geq 1- 1/(nd^C).
\end{align*}
\end{fact}
\begin{proof}
This follows by Proposition 1.1 in \cite{hkz12}.
\end{proof}
%\begin{corollary}[Corollary 3 in \cite{zhong2016mixed}]\label{cor:inner_prod_bound}
%If $x \sim \mathcal{N}(\boldsymbol{0},I_d)$ and $P\geq 1$ is a constant, then given any fixed $ beta\in \R^d$,  with probability $1- \frac{1}{n}d^{-P}$, we have 
%\[( beta^Tx)^2 \leq  (4P+5) \|  beta \|^2  \log n .\]
%\end{corollary}

%\begin{corollary}[Corollary 4 in \cite{zhong2016mixed}]\label{cor:gaussian_norm_bound}
%If $x \sim \mathcal{N}(\boldsymbol{0},I_d)$ and $P\geq 1$ is a constant, then with probability $1-\frac{1}{n}d^{-P}$, we have 
%\[\|x\|^2  \leq (4P+5) d \log n.\]
%\end{corollary}

%\begin{lemma}[Lemma 7 in \cite{zhong2016mixed}]\label{lem:remove_n}
%If $n\geq c \log^{K+1}(c) K^{4K}  d\log^{K+2} (d)$, where $c,d,K>0$, then $n \geq c d\log d \log^{K+1}(n)$. 
%\end{lemma}

\begin{fact}\label{fac:sigma_k_Wbar}
Given a full column-rank matrix $W=[ w_1,  w_2, \cdots,  w_k] \in \mathbb{R}^{d\times k}$, let $\ov{W} = [\frac{ w_1}{\| w_1\|}$, $\frac{ w_2}{\| w_2\|}$, $\cdots$, $\frac{ w_k}{\| w_k\|}]$. 
 Then, we have: (\RN{1}) for any $i\in [k]$, $\sigma_k(W) \leq \| w_i\| \leq \sigma_1(W)$; (\RN{2}) $ 1/\kappa(W) \leq \sigma_k(\ov{W}) \leq \sigma_1(\ov{W}) \leq \sqrt{k}$. 
\end{fact}
\begin{proof}
Part (\RN{1}). We have,
$$\sigma_k(W) \leq \|W e_i\| = \| w_i\| \leq \sigma_1(W)$$

Part (\RN{2}).
We first show how to lower bound $\sigma_k(\ov{W})$,
\begin{align*}
\sigma_k(\ov{W}) & = ~ \min_{\| s\|=1}\| \ov{W}  s\|  \\
& = ~ \min_{\| s\|=1} \left\|\sum_{i=1}^k \frac{s_i}{\| w_i\|} w_i \right\| & \text{~by~definition~of~}\ov{W}  \\
& \geq ~ \min_{\| s\|=1} \sigma_k(W)  \left( \sum_{i=1}^k (\frac{s_i}{\| w_i\|})^2 \right)^{\frac{1}{2}} & \text{~by~} \|  w_i\| \geq \sigma_k(W)  \\
& \geq ~ \min_{\|  s\|_2= 1} \sigma_k(W) \left( \sum_{i=1}^k (\frac{s_i}{\max_{j\in [k]}\| w_j\|})^2 \right)^{\frac{1}{2}} & \text{~by~} \max_{j\in [k]} \| w_j\|\geq \|  w_i\|  \\
& =  ~ \sigma_k(W) / \max_{j\in [k]} \|  w_j \|  & \text{~by~} \| s\|=1\\
& \geq ~ \sigma_k(W) /\sigma_1(W).  & \text{~by~}\max_{j\in [k]} \| w_j\| \leq \sigma_1(W) \\
& = ~ 1/\kappa(W).
\end{align*}

It remains to upper bound $\sigma_1(\ov{W})$,
\begin{align*}
\sigma_1(\ov{W}) \leq \left( \sum_{i=1}^k \sigma_i^2 ( \ov{W} ) \right)^{\frac{1}{2}} =\| \ov{W} \|_F \leq \sqrt{k}.
\end{align*}
\end{proof}

\begin{fact}\label{fac:UU_VV}
Let $U\in \R^{d\times k}$  and $V\in \R^{d\times k} $ $(k\leq d)$ denote two orthogonal matrices. 
Then $\|UU^\top -VV^\top \| = \|(I-UU^\top ) V \| =\|(I - VV^\top)U \|  = \sqrt{1 - \sigma_k^2(U^\top V)}$.
\end{fact}
\begin{proof}
Let $U_\perp \in \mathbb{R}^{d \times (d-k)}$ and $V_\perp \in \mathbb{R}^{d \times (d-k)}$ be the orthogonal complementary matrices of $U,V \in \mathbb{R}^{d\times k}$ respectively.
\begin{align*}
\|UU^\top -VV^\top \| & = \|(I-VV^\top)UU^\top -VV^\top(I-UU^\top)\|  \\
& = \| V_\perp V_\perp^\top U U^\top -VV^\top U_\perp U_\perp^\top\|  \\
& = \left\| 
\begin{bmatrix} V_\perp & V  \end{bmatrix} 
\begin{bmatrix} V_\perp^\top U & 0 \\ 0 & - V^\top U_\perp \end{bmatrix} 
\begin{bmatrix} U^\top  \\ U_\perp^\top \end{bmatrix} 
\right\|  \\
& = \max ( \|V_\perp^\top U\|,\|V^\top U_\perp\| ).
\end{align*}

We show how to simplify $\|V_\perp^\top U\|$,
\begin{align*}
\|V_\perp^\top U\| = \|(I - VV^\top ) U \| = \sqrt{\|U^\top (I - VV^\top ) U \|} = \max_{\| a\|=1}\sqrt{1 - \|V^\top U a\|^2} = \sqrt{1 - \sigma_k^2(V^\top U)}.
\end{align*}

Similarly we can simplify $\| U_\perp^\top V \|$,
\begin{align*}
\| U_\perp^\top V \| = \sqrt{1 - \sigma^2_k(U^\top V)} = \sqrt{1 - \sigma^2_k(V^\top U)} .
\end{align*} 
\end{proof}

\begin{fact}
Let $C\in \R^{d_1\times d_2},B\in\R^{d_2\times d_3}$ be two matrices. Then $\|CB\| \leq \|C\|\|B\|_F$ and $\|CB\| \geq \sigma_{\min}(C)\|B\|_F$.
\end{fact}
\begin{proof}
For each $i\in [d_3]$, let $ b_i$ denote the $i$-th column of $B$. We can upper bound $\| C B\|$,
\begin{align*}
\|CB \|_F = \left( \sum_{i=1}^{d_2} \|C b_i\|^2 \right)^{1/2} \leq \left( \sum_{i=1}^{d_2} \|C\|^2 \| b_i\|^2 \right)^{1/2} = \| C \| \| B \|_F.
\end{align*}
We show how to lower bound $\| C B \|$,
\begin{align*}
\| C B\| = \left( \sum_{i=1}^{d_2} \|C b_i\|^2 \right)^{1/2}  \geq  \left( \sum_{i=1}^{d_2} \sigma_k^2(C) \| b_i\|^2 \right)^{1/2} = \sigma_{\min}(C) \| B\|_F.
\end{align*}
\end{proof}

\begin{fact}\label{fac:exp_gaussian_dot_three_vectors}
Let $a,b,c \geq 0$ denote three constants, let $u,v,w\in \mathbb{R}^d$ denote three vectors, let ${\cal D}_d$ denote Gaussian distribution ${\cal N}(0,I_d)$ then
\begin{align*}
\underset{x\sim {\cal D}_d}{\E} \left[ |u^\top x|^a |v^\top x|^b |w^\top x|^c \right] \eqsim \|u\|^a \| v\|^b \| w\|^c.
\end{align*}
\end{fact}
\begin{proof}
\begin{align*}
 \underset{x\sim {\cal D}_d}{\E} \left[ |u^\top x|^a |v^\top x|^b |w^\top x|^c \right] \leq & ~ \left( \underset{x\sim {\cal D}_d}{\E} [ |u^\top x|^{2a}]\right)^{1/2} \cdot \left( \underset{x\sim {\cal D}_d}{\E} [ |u^\top x|^{4b}] \right)^{1/4} \cdot  \left( \underset{x\sim {\cal D}_d}{\E} [ |u^\top x|^{4c}  ]\right)^{1/4} \\
\lesssim & ~ \| u \|^a \| v \|^b \| w \|^c,
\end{align*}
where the first step follows by H\"{o}lder's inequality, i.e., $\E[|XYZ|] \leq ( \E[|X|^2])^{1/2} \cdot  ( \E[|Y|^4] )^{1/4} \cdot ( \E[ |Z|^4])^{1/4}$, the third step follows by calculating the expectation and $a,b,c$ are constants.

Since all the three components $|u^\top x|$, $|v^\top x|$, $|w^\top x|$ are positive and related to a common random vector $x$, we can show a lower bound, 
\begin{align*}
\underset{x\sim {\cal D}_d}{\E} \left[ |u^\top x|^a |v^\top x|^b |w^\top x|^c \right] \gtrsim \|u\|^a \| v\|^b \| w\|^c.
\end{align*}
\end{proof}

\subsection{Matrix Bernstein}

In many proofs we need to bound the difference between some population matrices/tensors and their empirical versions. Typically, the classic matrix Bernstein inequality requires the norm of the random matrix be bounded \emph{almost surely} (e.g., Theorem 6.1 in \cite{t11}) or the random matrix satisfies subexponential property (Theorem 6.2 in \cite{t11}) . However, in our cases, most of the random matrices don't satisfy these conditions. So we derive the following lemmata that can deal with random matrices that are not bounded almost surely or follow subexponential distribution, but are bounded with high probability. %Lemma~\ref{lem:modified_bernstein_non_zero} considers the case when the population matrix is non-zero while Lemma~\ref{lem:modified_bernstein_zero_mean} is for the case when population matrix is zero. 

\begin{lemma}[Matrix Bernstein for unbounded case (A modified version of bounded case, Theorem 6.1 in \cite{t11})]% with non-zero mean]
\label{lem:modified_bernstein_non_zero}
Let ${\cal B}$ denote a distribution over $\mathbb{R}^{d_1 \times d_2}$. Let $d = d_1 +d_2$. Let $B_1, B_2, \cdots B_n$ be i.i.d. random matrices sampled from ${\cal B}$. Let $\overline{B} = \mathbb{E}_{B\sim {\cal B}} [B]$ and $\wh{B}  = \frac{1}{n} \sum_{i=1}^n B_i$. For parameters $m\geq 0, \gamma \in (0,1),\nu >0 ,L>0$, if the distribution ${\cal B}$ satisfies the following four properties,
%\Zhao{ (\RN{1}) For any constant $t \geq 1$, there exist some $K(t)>0$ and $R(t)>0$ such that for any $B\sim {\cal B}$, with probability $1-n^{-1}(d_1+d_2)^{-t}$,  $\|B\| \leq R(t)\log^{K(t)}(n)$; } 
 \begin{align*}
\mathrm{(\RN{1})} \quad & \quad \underset{B \sim {\cal B}}{\Pr}\left[ \left\| B \right\| \leq  m \right] \geq 1 - \gamma; \\ %%%m=C \log^c n, gamma=1/( n d^{t} )
\mathrm{(\RN{2})} \quad & \quad \left\| \underset{B \sim {\cal B}}{\mathbb{E}}[B]  \right\| >0; \\
\mathrm{(\RN{3})} \quad & \quad \max \left( \left\| \underset{B \sim {\cal B}}{\mathbb{E}} [ B B^\top ] \right\|, \left\| \underset{B \sim {\cal B}}{\mathbb{E}} [ B^\top B ] \right\| \right) \leq \nu ;\\ 
\mathrm{(\RN{4})} \quad & \quad \max_{\| a\|=\| b\|=1} \left( \underset{B \sim {\cal B}}{\mathbb{E}} \left[ \left( a^\top B  b \right)^2 \right]  \right)^{1/2} \leq L.
\end{align*}

Then we have for any $0<\epsilon <1$ and $t\geq 1$, if
\begin{align*}
n \geq  ( 18 t \log d  ) \cdot ( \nu + \| \ov{B} \|^2+ m \| \ov{B} \| \epsilon )  / ( \epsilon^2 \| \ov{B} \|^2 ) \quad \text{~and~} \quad \gamma \leq (\epsilon \| \ov{B} \| /(2L) )^2
\end{align*} % m = R(t) \log^{K(t)}(n)
with probability  at least $1-1/d^{2t} - n\gamma$,
\begin{equation*}
\| \wh{B} - \ov{B} \| \leq \epsilon \| \ov{B} \|.
\end{equation*}
\end{lemma}

\begin{proof}
%Let $m = R(t)\log^{K(t)}(n)$.
Define the event \[\xi_i = \{ \|B_i\| \leq m \}, \forall i\in [n]. \]
Define $M_i = \bone_{\|B_i\| \leq m}B_i$. Let $\ov{M} = \mathbb{E}_{B\sim {\cal B}} [ \bone_{\|B \| \leq m}B ] $ and $\wh{M} = \frac{1}{n}\sum_{i=1}^n M_i$. By triangle inequality, we have 
\begin{align}\label{eq:whB_minus_ovB}
 \| \wh{B} - \ov{B} \| \leq \|\wh{B} - \wh{M}\| +\|\wh{M} - \ov{M}\| + \| \ov{M} - \ov{B} \|.
\end{align}
  In the next a few paragraphs, we will upper bound the above three terms. 
%\begin{enumerate}
%\item $\widehat{M} = \widehat{B}$ with high probability by the union bound,
%\item $\| \ov{M} - \ov{B} \|$ is bounded because $\Expect{\bone_{\|B_i\|\leq m}}$ is small,
%\item $\|\widehat{M} - M\| $ is bounded by matrix Bernstein inequality.
%\end{enumerate}

{\bf The first term in Eq.~\eqref{eq:whB_minus_ovB}}. Denote $\xi^c$ as the complementary set of $\xi$, thus $\Pr[\xi^c_i] \leq \gamma$.
By a union bound over $i\in [n]$, with probability $1- n \gamma$, $\|B_i\| \leq m$ for all $i\in [n]$. Thus $\wh{M} = \wh{B}$. 

{\bf The second term in Eq.~\eqref{eq:whB_minus_ovB}}. For a matrix $B$ sampled from ${\cal B}$, we use $\xi$ to denote the event that $\xi = \{ \| B \| \leq m\}$. Then, we can upper bound $\|\ov{M}-\ov{B}\|$ in the following way,
%Let $\xi^c = \{x| \|g(x)xx^\top\|\geq m\}$.
\begin{align}%\label{eq:BM}
& ~ \| \ov{M}- \ov{B} \| \notag \\
= & ~ \left\| \underset{B \sim {\cal B}}{\mathbb{E}} [\bone_{\| B \| \leq m} \cdot B] - \underset{B\sim {\cal B}}{\mathbb{E}} [B] \right\| \notag \\
=& ~ \left\| \underset{B \sim {\cal B}}{\mathbb{E}} [ B \cdot \bone_{\xi^c} ] \right\| \notag \\
=& ~ \max_{\| a\|=\| b\| = 1}  \underset{B \sim {\cal B}}{\mathbb{E}} [ a^\top B b \bone_{\xi^c}] \notag \\
\leq & ~  \max_{\| a\|=\| b\|=1}  \underset{B \sim {\cal B}}{\mathbb{E}} [ ( a^\top B b)^2 ]^{1/2} \cdot \underset{B \sim {\cal B}}{\mathbb{E}} [\bone_{\xi^c} ]^{1/2} &\text{~by~H\"{o}lder's inequality} \notag \\
\leq & ~ L \underset{B \sim {\cal B}}{\mathbb{E}} [\bone_{\xi^c} ]^{1/2}  & \text{~by~Property~(\RN{4})} \notag \\
\leq & ~  L \gamma^{1/2}, &\text{~by~}\Pr[\xi^c] \leq \gamma \notag \\
\leq & ~ \frac{1}{2} \epsilon \| \ov{B} \|, & \text{~by~} \gamma \leq (\epsilon \| \ov{B}\| / (2L))^2 \notag 
\end{align}
which implies
%where the first inequality follows by Holder's inequality. Note that if $n\geq 4L^2 \epsilon^{-2}\| \ov{ B } \|^{-2}d^{-t}$, we have $L n^{-1/2} d^{-t/2} \leq \frac{\epsilon}{2} \| \ov{B} \|$ and therefore 
\begin{align*}%\label{eq:M_B_diff}
\| \ov{M}- \ov{B} \|\leq \frac{\epsilon}{2}\| \ov{B} \|.
\end{align*}
Since $\epsilon<1$, we also have $\| \ov{M}-\ov{B} \| \leq \frac{1}{2} \| \ov{B} \|$ and $\frac{3}{2}\| \ov{B} \|\geq \|\ov{M}\|\geq \frac{1}{2}\| \ov{B} \|$.

{\bf The third term in Eq.~\eqref{eq:whB_minus_ovB}}. We can bound $\|\wh{M} - \ov{M}\|$ by Matrix Bernstein's inequality \cite{t11}. 

We define $Z_i = M_i - \ov{M}$. Thus we have $ \underset{B_i\sim {\cal B}}{\mathbb{E}} [Z_i] = 0$, $\|Z_i\| \leq 2m$,  and 
\begin{align*}
 \left\| \underset{B_i\sim {\cal B}}{\mathbb{E}} [ Z_iZ_i^\top ] \right\| = \left\| \underset{B_i\sim {\cal B}}{\mathbb{E}} [ M_iM_i^\top] - \ov{M}\cdot \ov{M}^\top \right\| \leq \nu + \| \ov{M} \|^2 \leq  \nu + 3\| \ov{B} \|^2 .
\end{align*}
  Similarly, we have $\left\|\underset{B_i\sim {\cal B}}{\mathbb{E}} [Z_i^\top Z_i] \right\| \leq   \nu+3\| \ov{ B} \|^2$. 
Using matrix Bernstein's inequality, for any $\epsilon>0$, 
\begin{align*}
& \underset{B_1, \cdots, B_n \sim {\cal B} }{\Pr} \left[ \frac{1}{n} \left\| \sum_{i=1}^{n}Z_i \right\|\geq \epsilon \| \ov{B} \| \right] 
\leq  d \exp\left(-\frac{\epsilon^2 \|\ov{B}\|^2 n /2}{ \nu+3\|\ov{B}\|^2 + 2m \| \ov{B}\| \epsilon /3} \right) .
\end{align*}
%Setting $\epsilon \|B\| \leq \frac{3\|B\|^2 + \nu}{ 2m} $, and 
By choosing
\begin{align*}%\label{sample_complexity_raw}
n \geq ( 3t \log d ) \cdot  \frac{\nu+3\|\ov{B} \|^2+2m \| \ov{B} \| \epsilon /3}{ \epsilon^2 \| \ov{B}\|^2 /2 } ,
\end{align*}
for $t\geq 1$, we have with probability at least $1-1/d^{2t}$,
\begin{align*}%\label{eq:M_bounded}
\left\| \frac{1}{n}\sum_{i=1}^n M_i - \ov{M} \right\| \leq \frac{\epsilon}{2} \|\ov{B}\| 
\end{align*}

%So combining Eq~\eqref{eq:M_bounded} with previous results in Step 1, $\wh{B}=  \wh{M}$, and Step 2, Eq~\eqref{eq:M_B_diff}, 
Putting it all together, we have for $0<\epsilon <1$, if
\begin{equation*}
n \geq ( 18 t \log d)  \cdot ( \nu+ \|\ov{B}\|^2+ m \|\ov{B}\| \epsilon ) / (\epsilon^2 \|\ov{B}\|^2)  \quad \text{~and~} \quad \gamma \leq (\epsilon \| \ov{B}\| / (2L))^2
\end{equation*} 
with probability  at least $1-1/d^{2t}-n\gamma$,
\begin{equation*}
\left\| \frac{1}{n} \sum_{i=1}^n B_i - \underset{B \sim {\cal B} }{\mathbb{E}}[B] \right\| \leq \epsilon \left\| \underset{B \sim {\cal B} }{\mathbb{E}}[B] \right\|.
\end{equation*}

\end{proof}

\begin{corollary}[Error bound for symmetric rank-one random matrices]% with non-zero mean]
\label{cor:modified_bernstein_tail_xx}
Let $x_1, x_2, \cdots x_n$ denote $n$ i.i.d. samples drawn from Gaussian distribution ${\cal N}(0,I_d)$. Let $h(x) : \mathbb{R}^d \rightarrow \mathbb{R}$ be a function satisfying the following properties $\mathrm{(\RN{1})}$, $\mathrm{(\RN{2})}$ and $\mathrm{(\RN{3})}$. 
%\Zhao{
%(\RN{1}) For any constant $t \geq 1$, there exist some $R_h(t)>0$ and $K_h(t)>0$ depending on $t$ such that given any $i \in [n]$ with probability $1-n^{-1}d^{-t}$,  $|h(x_i)| \leq R_h(t) \log^{K_h(t)}(n)$; 
%}
 \begin{align*}
\mathrm{(\RN{1})} ~ & ~ \underset{x\sim {\cal N}(0,I_d)}{\Pr} \left[ |h(x)| \leq m \right] \geq 1- \gamma\\
\mathrm{(\RN{2})} ~ & ~\left\| \underset{x \sim {\cal N}(0,I_d)}{\mathbb{E}} [ h(x) x x^\top ]  \right\| > 0; \\
\mathrm{(\RN{3})} ~ & ~ \left( \underset{x \sim {\cal N}(0,I_d)}{\mathbb{E}} [h^4(x)] \right)^{1/4} \leq L .
\end{align*}

Define function $B(x) = h(x) x x^\top \in \mathbb{R}^{d\times d}$, $\forall i\in[n]$. Let $ \ov{B} = \underset{x \sim {\cal N}(0,I_d)}{\mathbb{E}} [ h(x) x x^\top ]$.
For any $0<\epsilon <1$ and $t\geq 1$, if 
\begin{align*}
n \gtrsim ( t\log d) \cdot (  L^2 d + \| \ov{B}\|^2  +  (m t d \log n ) \| \ov{B}\| \epsilon) / ( \epsilon^2 \|\ov{B}\|^2 ),  \text{~and~} \gamma + 1/(nd^{2t}) \lesssim (\epsilon \| \ov{B}\| / L )^2
\end{align*} 
then
\begin{align*}
\underset{x_1,\cdots, x_n \sim {\cal N}(0,I_d) }{\Pr} \left[ \left\| \ov{B} - \frac{1}{n} \sum_{i=1}^n B(x_i) \right\| \leq \epsilon \|\ov{B}\| \right] \geq 1- 2 / (d^{2t}) - n\gamma.
 \end{align*}
 \end{corollary}

\begin{proof}
We show that the four Properties in Lemma~\ref{lem:modified_bernstein_non_zero} are satisfied. Define function $B(x) = h(x) x x^\top$.

(\RN{1}) $\|B(x)\| = \|h(x)x x^\top\| = |h(x)| \|x\|^2$. 

By using Fact~\ref{fac:gaussian_norm_bound}, we have
\begin{align*}
\underset{x \sim {\cal N}(0,I_d)}{\Pr}[ \|x\|^2 \leq 10 t d \log n ] \geq 1 -1/(nd^{2t})
\end{align*} 
Therefore,
\begin{align*}
\underset{x \sim {\cal N}(0,I_d)}{\Pr} [ \|B(x) \| \leq m \cdot 10 t d \log(n) ] \geq 1-\gamma - 1/(nd^{2t}).
\end{align*}

(\RN{2}) $ \left\| \underset{B \sim {\cal B}}{\E} [ B ] \right\| =  \left\| \underset{x \sim {\cal N}(0,I_d)  }{\E} [ h(x) x x^\top ] \right\|  > 0$. 

(\RN{3})
 \begin{align*}
 & ~ \max \left( \left\| \underset{B \sim {\cal B} }{\E} [ B B^\top ] \right\|, \left\| \underset{B \sim {\cal B} }{\E} [ B^\top B ] \right\| \right) \\
= & ~ \max_{\| a\|=1} \underset{x \sim {\cal N}(0,I_d) }{\E} [(h(x))^2 \|x\|^2 ( a^\top x)^2] \\
\leq & ~ \left( \underset{x \sim {\cal N}(0,I_d) }{\E} [(h(x))^4] \right)^{1/2} \cdot \left(  \underset{x \sim {\cal N}(0,I_d) }{\E} [ \|x\|^8] \right)^{1/4} \cdot \max_{\| a\|=1} \left( \underset{x \sim {\cal N}(0,I_d) }{\E} [ ( a^\top x)^8] \right)^{1/4} \\
\lesssim & ~ L^2 d .
\end{align*}

(\RN{4})
\begin{align*}
 & ~\max_{\| a\|=\| b\|=1} \left(\underset{B \sim {\cal B} }{\E}  [( a^\top B  b)^2] \right)^{1/2}  \\
 = & ~ \max_{\| a\|=1} \left( \underset{x\sim {\cal N}(0,I_d)}{\E}[ h^2(x) ( a^\top x)^4] \right)^{1/2} \\
\leq & ~ \left( \underset{x\sim {\cal N}(0,I_d)}{\E}[ h^4(x)] \right)^{1/4} \cdot \max_{\| a\|=1} \left( \underset{x\sim {\cal N}(0,I_d)}{\E}[ ( a^\top x)^8] \right)^{1/4} \\
\lesssim & ~ L.
\end{align*}

Applying Lemma~\ref{lem:modified_bernstein_non_zero}, we obtain,
for any $0<\epsilon <1$ and $t\geq 1$, if 

\begin{align*}
n \gtrsim ( t\log d) \cdot (  L^2 d + \| \ov{B}\|^2  +  (m t d \log n ) \| \ov{B}\| \epsilon ) / ( \epsilon^2 \|\ov{B}\|^2 )  , \text{~and~} \gamma + 1/(nd^{2t}) \lesssim (\epsilon \| \ov{B}\| / L )^2
\end{align*}
then
\begin{align*}
\underset{x_1,\cdots, x_n \sim {\cal N}(0,I_d) }{\Pr} \left[ \left\| \ov{B} - \frac{1}{n} \sum_{i=1}^n B(x_i) \right\| \leq \epsilon \|\ov{B}\| \right] \geq 1- 2 / (d^{2t}) - n\gamma.
 \end{align*}

\end{proof}

%\section{About Property~\ref{pro:gradient},~\ref{pro:expect},~\ref{pro:hessian}}

%\section{Proof of Proposition~\ref{proposition:act_example}}
\section{Properties of Activation Functions}\label{app:proof_prop1}
\restate{proposition:act_example}
\begin{proof}
We can easily verify that \ReLU~, leaky \ReLU~and squared \ReLU~satisfy Property~\ref{pro:expect} by calculating $\rho(\sigma)$ in Property~\ref{pro:expect}, which is shown in Table~\ref{table:pro2}. Property~\ref{pro:gradient} for  \ReLU~, leaky \ReLU~and squared \ReLU~can be verified since they are non-decreasing with bounded first derivative. \ReLU~and leaky \ReLU~are piece-wise linear, so they satisfy Property~\ref{pro:hessian}(b). Squared \ReLU~is smooth so it satisfies Property~\ref{pro:hessian}(a). 
\begin{table}[H]
\centering
\begin{tabular}{|c|c|c|c|c|c|c|c|} \hline
Activations & \ReLU & Leaky \ReLU & squared \ReLU & erf & \shortstack{ sigmoid\\ ($\sigma=0.1$)} & \shortstack{ sigmoid \\ ($\sigma=1$)} & \shortstack{ sigmoid \\ ($\sigma=10$)}\\ \hline
$\alpha_0(\sigma)$ &  $\frac{1}{2}$ &  $\frac{1.01}{2} $& $\sigma \sqrt{ \frac{2}{\pi}}$ &  $\frac{1}{(2\sigma^2+1)^{1/2}}$ & 0.99 & 0.605706 & 0.079\\
$\alpha_1(\sigma)$ &  $\frac{1}{\sqrt{2\pi}}$ & $\frac{0.99}{\sqrt{2\pi}}$ & $\sigma$ & 0 & 0 & 0 & 0 \\
$\alpha_2(\sigma)$ &  $\frac{1}{2}$ &  $\frac{1.01}{2}$ & $2\sigma \sqrt{\frac{2}{\pi}}$ &  $\frac{1}{(2\sigma^2+1)^{3/2}}$ & 0.97 & 0.24 & 0.00065 \\
$\beta_0(\sigma)$ &   $\frac{1}{2}$ &  $\frac{1.0001}{2}$ & $2\sigma$ &  $\frac{1}{(4\sigma^2+1)^{1/2}}$ & 0.98	& 0.46 & 0.053\\
$\beta_2(\sigma)$ &   $\frac{1}{2}$ &  $\frac{1.0001}{2}$ & $6\sigma$ &  $\frac{1}{(4\sigma^2+1)^{3/2}}$ & 0.94 & 0.11 & 0.00017  \\ \hline
$\rho(\sigma)$ & 0.091 & 0.089	& 0.27$\sigma$ & $\rho_{\text{erf}}(\sigma)$ $^1$ & 1.8E-4 & 4.9E-2 & 5.1E-5 \\ \hline
\end{tabular}
\caption{$\rho(\sigma)$ values for different activation functions. Note that we can calculate the exact values for \ReLU, Leaky \ReLU, squared \ReLU~and erf. We can't find a closed-form value for sigmoid or tanh, but we calculate the numerical values of $\rho(\sigma)$ for $\sigma=0.1,1,10$. $^1$ $\rho_{\text{erf}}(\sigma) = \min\{(4\sigma^2+1)^{-1/2} - (2\sigma^2+1)^{-1}, (4\sigma^2+1)^{-3/2} - (2\sigma^2+1)^{-3}, (2\sigma^2+1)^{-2}\}$}
\label{table:pro2}
\end{table}
Smooth non-decreasing activations with bounded first derivatives automatically satisfy Property~\ref{pro:gradient} and ~\ref{pro:hessian}. For Property~\ref{pro:expect}, since their first derivatives are symmetric, we have $\E[\phi'(\sigma\cdot z)z] = 0$. Then by H\"{o}lder's inequality and $\phi'(z)\geq 0$, we have 
\begin{align*}
&\E_{z\sim \D_1}[\phi'^2(\sigma\cdot z)] \geq \left( \E_{z\sim \D_1}[ \phi'(\sigma\cdot z) ] \right)^2, \\
& \E_{z\sim \D_1 } [ \phi'^2(\sigma\cdot z)z^2] \cdot \E_{z\sim \D_1}[z^2] \geq \left(\E_{z\sim \D_1}[\phi'(\sigma\cdot z)z^2] \right)^2 , \\
&\E_{z\sim \D_1}[ \phi'(\sigma\cdot z)z^2] \cdot \E_{z\sim \D_1}[ \phi'(\sigma\cdot z)] = \E_{z\sim \D_1}[(\sqrt{\phi'(\sigma\cdot z)}z)^2]\cdot \E_{z\sim \D_1}[(\sqrt{\phi'(\sigma\cdot z)})^2] \geq \left(\E_{z\sim \D_1}[\phi'(\sigma\cdot z)z] \right)^2.
\end{align*}
The equality in the first inequality happens when $\phi'(\sigma\cdot z)$ is a constant a.e.. The equality in the second inequality happens when  $|\phi'(\sigma\cdot z)| $ is a constant a.e., which is invalidated by the non-linearity and smoothness condition. The equality in the third inequality holds only when $\phi'(z)=0$ a.e., which leads to a constant function under non-decreasing condition. Therefore, $\rho(\sigma) > 0$ for any smooth non-decreasing non-linear activations with bounded symmetric first derivatives. The statements about linear activations and quadratic activation follow direct calculations. 
\end{proof}

\section{Local Positive Definiteness of Hessian}\label{app:local_pd}

\subsection{Main Results for Positive Definiteness of Hessian}\label{app:convex_main_result}

\subsubsection{Bounding the Spectrum of the Hessian near the Ground Truth}
\begin{theorem}[Bounding the spectrum of the Hessian near the ground truth]\label{thm:lsc_nn}
%Assume $\{ x_j,y_j\}_{j=1,2,\cdots,n}$ obeys Model~\eqref{eq:model}. Let $v_i=v_i^*$ and $ w_i$ is close enough to $ w_i^*$, $\forall i\in [k]$, i.e., $\|W - W^*\| \leq \Theta(\frac{1}{k^2 \kappa^6 \hat \kappa}\sigma_k(W^*))$, where $\kappa = \frac{\sigma_1(W^*)}{\sigma_k(W^*)}$ and $\hat \kappa = \frac{\Pi_{i=1}^k \sigma_i(W^*)}{ (\sigma_{k}(W^*))^{k} }$. Let $W$ be independent of the sample set and $\phi(z) = \max\{0,z\}$. Then when the number of samples, $n\geq \Otilde(P^2d k^4 \kappa^6 {\hat \kappa}^2)$ for any $P\geq 1$, w.p. $1-O(d^{-P})$ the Hessian of $\widehat f(W) $ at $W$ satisfies
%$$ \Theta(\frac{1}{k\kappa \hat \kappa}) I\preceq \nabla^2 \widehat{f}(W) \preceq \Theta(k\kappa^2) I$$ 
%%Given a full-column-rank matrix $W^*\in \mathbb{R}^{d\times k}$, let $\sigma_i$ denote the $i$-th largest singular value of $W^*$. Let $\kappa = \sigma_1/ \sigma_k$, $\lambda=  (\prod_{i=1}^k \sigma_i )/ \sigma_{k}^{k} $. Let $v_{\max}$ denote $\max_{i\in[k]} |v_i^*|$. 
For any $W\in \mathbb{R}^{d\times k}$ with $\|W - W^*\| \lesssim
v_{\min}^4 \rho^2(\sigma_k) / ( k^2\kappa^5 {\lambda}^2 v_{\max}^4 \sigma_1^{4p} ) \cdot \| W^*\|$, let $S$ denote a set of i.i.d. samples from distribution ${\cal D}$ (defined in~(\ref{eq:model})) and let the activation function satisfy Property~\ref{pro:gradient},\ref{pro:expect},\ref{pro:hessian}. Then for any $t\geq 1$, if $|S| \geq d \cdot \poly(\log d,t) \cdot k^2v_{\max}^4 \tau \kappa^8 \lambda^2 \sigma_1^{4p}/(v_{\min}^4 \rho^2(\sigma_k))$, we have with probability at least $1-d^{-\Omega(t)}$,
\begin{align*}
 \Omega( v_{\min}^2 \rho(\sigma_k) / (\kappa^2 \lambda ) ) I\preceq \nabla^2 \widehat{f}_S(W) \preceq O(kv_{\max}^2\sigma_1^{2p}) I.
\end{align*}
\end{theorem}

\begin{proof}
The main idea of the proof follows the following inequalities,
\begin{align*}
 \nabla^2 f_{\cal D}(W^*) - \|\nabla^2 \widehat{f}_S(W) - \nabla^2 f_{\cal D}(W^*)\|I \preceq \nabla^2 \widehat{f}_S(W) \preceq ~ \nabla^2 f_{\cal D}(W^*) + \|\nabla^2 \widehat{f}_S(W) - \nabla^2 f_{\cal D}(W^*)\|I 
%&- \|\nabla^2 f(W) - \nabla^2 \widehat{f}(W)\| I 
\end{align*}
The proof sketch is first to bound the range of the eigenvalues of $\nabla^2 f_{\cal D}(W^*)$ (Lemma~\ref{lem:pd_ground}) and then bound the spectral norm of the remaining error, $\|\nabla^2 \widehat{f}_S(W) - \nabla^2 f_{\cal D}(W^*)\|$.
$\|\nabla^2 \widehat{f}_S(W) - \nabla^2 f_{\cal D}(W^*)\|$ can be further decomposed into two parts, $\|\nabla^2 \widehat{f}_S(W) - H\|$ and $\|H - \nabla^2 f_{\cal D}(W^*)\|$, where $H$ is $\nabla^2 f_{\cal D}(W)$ if $\phi$ is smooth, otherwise $H$ is a specially designed matrix . We can upper bound them when $W$ is close enough to $W^*$ and there are enough samples. In particular, if the activation satisfies Property~\ref{pro:hessian}(a), see Lemma~\ref{lem:smooth_pop_local} for bounding $ \|H - \nabla^2 f_{\cal D}(W^*)\|$ and Lemma~\ref{lem:emp_pop} for bounding $ \|H -
\nabla^2 \widehat{f}_S(W)\|$. If the activation satisfies Property~\ref{pro:hessian}(b), see Lemma~\ref{lem:emp_pop_nonsmooth}. 

Finally we can complete the proof by setting $\delta = O(v_{\min}^2 \rho(\sigma_1) / ( kv_{\max}^2\kappa^2 \lambda \sigma_1^{2p} ) )$
in Lemma~\ref{lem:emp_pop} and Lemma~\ref{lem:emp_pop_nonsmooth}, setting $\|W-W^*\|\lesssim v_{\min}^2 \rho(\sigma_k)/ ( k\kappa^2\lambda v_{\max}^2 \sigma_1^p)$ in Lemma~\ref{lem:smooth_pop_local} and setting $\|W-W^*\|\leq v_{\min}^4 \rho^2(\sigma_k) \sigma_k/ ( k^2\kappa^4 {\lambda}^2 v_{\max}^4 \sigma_1^{4p} )$ in Lemma~\ref{lem:emp_pop_nonsmooth}. 
\end{proof}

\subsubsection{Local Linear Convergence of Gradient Descent}%{Proof of Proposition~\ref{thm:lc_gd}}

Although Theorem~\ref{thm:lsc_nn} gives upper and lower bounds for the spectrum of the Hessian w.h.p., it only holds when the current set of parameters $W$ are independent of samples. When we use iterative methods, like gradient descent, to optimize this objective, the next iterate calculated from the current set of samples will depend on this set of samples. Therefore, we need to do resampling at each iteration. Here we show that for activations that satisfies Properties~\ref{pro:gradient},~\ref{pro:expect} and \ref{pro:hessian}(a), linear convergence of gradient descent is guaranteed. To the best of our knowledge, there is no linear convergence guarantees for general non-smooth objective. So the following proposition also applies to smooth objectives only, which excludes \ReLU. 

\begin{theorem}[Linear convergence of gradient descent, formal version of Theorem~\ref{thm:lc_gd_informal}]\label{thm:lc_gd}
Let $W^c\in \R^{d\times k}$ be the current iterate satisfying 
\begin{align*}
\|W^c - W^*\| \lesssim v_{\min}^4 \rho^2(\sigma_k) / ( k^2\kappa^5 {\lambda}^2 v_{\max}^4 \sigma_1^{4p} ) \| W^*\|.
\end{align*}
 Let $S$ denote a set of i.i.d. samples from distribution ${\cal D}$ (defined in~(\ref{eq:model})) Let the activation function satisfy Property~\ref{pro:gradient},\ref{pro:expect} and \ref{pro:hessian}(a). Define 
\begin{align*}
m_0 = \Theta( v_{\min}^2 \rho(\sigma_k)/ (\kappa^2 \lambda) ) \text{~and~} M_0= \Theta ( kv_{\max}^2 \sigma_1^{2p} ).
\end{align*}

 For any $t\geq 1$, if we choose 
\begin{align}\label{eq:lc_gd_S_bound}
 |S| \geq d\cdot \poly(\log d,t) \cdot k^2v_{\max}^4 \tau \kappa^8 \lambda^2 \sigma_1^{4p}/(v_{\min}^4 \rho^2(\sigma_k))
\end{align}
and
 perform gradient descent with step size $1/M_0$ on $\widehat{f}_S(W^c)$ and obtain the next iterate,
\begin{align*}
\wt{W} = W^c - \frac{1}{M_0} \nabla \widehat{f}_S(W^c),
\end{align*}
then with probability at least $1-d^{-\Omega(t)}$,
\begin{align*}
\| \wt{W} - W^*\|_F^2  \leq  (1- \frac{m_0}{M_0} ) \| W^c-W^*\|_F^2.
\end{align*}
\end{theorem}

\begin{proof}

To prove Theorem~\ref{thm:lc_gd}, we need to show the positive definite properties on the entire line between the current iterate and the optimum by constructing a set of anchor points, which are independent of the samples. Then we apply traditional analysis for the linear convergence of gradient descent. 

In particular, given a current iterate $W^c$, we set $d^{(p+1)/2}$ anchor points $\{W^a\}_{a=1,2,\cdots,d^{(p+1)/2}}$ equally along the line $\xi W^* + (1-\xi)W^c$ for $\xi\in[0,1]$. 

According to Theorem~\ref{thm:lsc_nn}, by setting $t \leftarrow  t+(p+1)/2$, we have with probability at least $1-d^{-(t+(p+1)/2)}$ for each anchor point $\{W^a\}$,
\begin{align*}
m_0 I\preceq \nabla^2 \widehat{f}_S(W^a) \preceq M_0 I.
\end{align*}

Then given an anchor point $W^a$, according to Lemma~\ref{lem:pd_near_anchors},
we have with probability $1-2d^{-(t+(p+1)/2)}$, for any points $W$ between $(W^{a-1}+W^{a})/2$ and $(W^a+W^{a+1})/2$,
\begin{equation}\label{eq:pd_all_points}
m_0 I\preceq \nabla^2 \widehat{f}_S(W) \preceq M_0 I.
\end{equation}

Finally by applying union bound over these $d^{(p+1)/2}$ small intervals,  we have with probability at least $1-d^{-t}$ for any points $W$ on the line between $W^c$ and $W^*$, 
\begin{align*}
m_0 I\preceq \nabla^2 \widehat{f}_S(W) \preceq M_0 I.
\end{align*}

Now we can apply traditional analysis for linear convergence of gradient descent. 

Let $\eta$ denote the stepsize.
\begin{align*}
& ~\| \wt{W} - W^*\|_F^2 \\
= &~ \|W^c  - \eta \nabla \widehat{f}_S(W^c) - W^* \|_F^2 \\
= &~ \| W^c - W^*\|_F^2 - 2\eta \langle \nabla \widehat{f}_S(W^c), (W^c-W^*) \rangle + \eta^2 \|\nabla \widehat{f}_S(W^c)\|_F^2 
\end{align*}
We can rewrite $\wh{f}_S(W^c)$,
\begin{align*}
\nabla \widehat{f}_S(W^c) =  \left( \int_{0}^1 \nabla^2 \widehat{f}_S( W^* + \gamma (W^c-W^*) ) d\gamma \right) \text{vec}(W^c - W^*).
\end{align*}
We define function $\wh{H}_S : \R^{d \times k} \rightarrow \R^{dk\times dk}$ such that
\begin{align*}
\wh{H}_S(W^c-W^*) = \left( \int_{0}^1 \nabla^2 \widehat{f}_S( W^* + \gamma (W^c-W^*) ) d\gamma \right).
\end{align*}
According to Eq.~\eqref{eq:pd_all_points},
\begin{equation}\label{eq:smooth_sc_line}
m_0 I \preceq \wh{H} \preceq M_0 I.
\end{equation} 
We can upper bound $\| \nabla \wh{f}_S(W^c) \|_F^2$,
\begin{align*}
\|\nabla \widehat{f}_S(W^c)\|_F^2 =  \langle \wh{H}_S (W^c-W^*), \wh{H}_S (W^c-W^*)\rangle 
\leq  M_0 \langle W^c-W^*, \wh{H}_S(W^c-W^*) \rangle .
\end{align*}

Therefore, 
\begin{align*}
& ~\|\wt{W} - W^*\|_F^2  \\
\leq & ~ \| W^c - W^*\|_F^2 - (-\eta^2 M_0 + 2\eta)\langle W^c-W^*, \wh{H} (W^c-W^*) \rangle\\
 \leq  & ~ \| W^c - W^*\|_F^2 - (-\eta^2 M_0 + 2\eta)m_0 \| W^c-W^*\|_F^2 \\
  =  & ~ \| W^c - W^*\|_F^2 - \frac{m_0}{M_0} \| W^c-W^*\|_F^2 \\
    \leq  & ~ (1- \frac{m_0}{M_0} ) \| W^c-W^*\|_F^2 
\end{align*}
where the third equality holds by setting $\eta = \frac{1}{M_0}$. 
\end{proof}

\subsection{Positive Definiteness of Population Hessian at the Ground Truth}\label{proof:lemma:pd_ground}

The goal of this Section is to prove Lemma~\ref{lem:pd_ground}.

\begin{lemma}[Positive definiteness of population Hessian at the ground truth]\label{lem:pd_ground}
If $\phi(z)$ satisfies Property~\ref{pro:gradient},\ref{pro:expect} and \ref{pro:hessian}, we have the following property for the second derivative of function $f_{\cal D}(W)$ at $W^*$,
\begin{align*}
 \Omega(v_{\min}^2 \rho(\sigma_k) / ( \kappa^2 \lambda ) ) I \preceq \nabla^2 f_{\cal D}(W^*) \preceq O(kv_{\max}^2 \sigma_1^{2p}) I.
\end{align*}
\end{lemma}
\begin{proof}
The proof directly follows Lemma~\ref{lem:lower_bound_pd_ground} (Section~\ref{app:lower_bound_hessian_non_ortho}) and Lemma~\ref{lem:upper_bound_pd_ground}(Section~\ref{app:upper_bound_hessian_non_ortho}).
\end{proof}
\subsubsection{Lower Bound on the Eigenvalues of Hessian for the Orthogonal Case}

\begin{lemma}\label{lem:lower_bound_ortho}
Let ${\cal D}_1$ denote Gaussian distribution ${\cal N}(0,1)$. Let $ \alpha_0 = \E_{z\sim {\cal D}_1} [\phi'(z)] $, $\alpha_1 = \E_{z \sim {\cal D}_1 } [ \phi'(z)z]$, $\alpha_2 =\E_{z \sim {\cal D}_1 } [\phi'(z)z^2]$,
$ \beta_0 = \E_{z\sim {\cal D}_1} [ \phi'^2(z) ]$ ,$ \beta_2 = \E_{z \sim {\cal D}_1 } [ \phi'^2(z)z^2 ]$. Let $\rho$ denote $ \min\{(\beta_0 - \alpha_0^2-\alpha_1^2), (\beta_2 - \alpha_1^2  - \alpha_2^2) \}$. Let $P = \begin{bmatrix} p_1 & p_2 & \cdots & p_k \end{bmatrix} \in \mathbb{R}^{k\times k}$.
Then we have,
\begin{equation}\label{eq:ortho_lower_bound}
\underset{ u \sim {\cal D}_k }{\E}  \left[ \left(\sum_{i=1}^k p_i^\top u \cdot \phi'(u_i) \right)^2 \right] \geq \rho \|P\|_F^2
\end{equation}
\end{lemma}

\begin{proof}
The main idea is to explicitly calculate the LHS of Eq~\eqref{eq:ortho_lower_bound}, then reformulate the equation and find a lower bound represented by $\alpha_0,\alpha_1,\alpha_2,\beta_0,\beta_2$.

\begin{align*}
& ~ \underset{ u\sim {\cal D}_k }{\mathbb{E}} \left[ \left(\sum_{i=1}^k p_i^\top u \cdot \phi'(u_i) \right)^2 \right]  \\
= & ~ \sum_{i=1}^k \sum_{l=1}^k \underset{ u\sim {\cal D}_k }{\mathbb{E}} [ p_i^\top ( \phi'( u_l) \phi'( u_i) \cdot u u^\top) p_l ]  \\
= & ~ \underbrace{ \sum_{i=1}^k  \underset{ u\sim {\cal D}_k }{\mathbb{E}} [ p_i^\top ( \phi'( u_i)^2 \cdot u u^\top) p_i ] }_{A} + \underbrace{ \sum_{i \neq l} \underset{ u\sim {\cal D}_k }{\mathbb{E}} [ p_i^\top ( \phi'( u_l) \phi'( u_i) \cdot u u^\top) p_l ] }_{B}
\end{align*}
Further, we can rewrite the diagonal term in the following way,
\begin{align*}
A = & ~\sum_{i=1}^k  \underset{ u\sim {\cal D}_k }{\mathbb{E}} [ p_i^\top ( \phi'( u_i)^2 \cdot u u^\top) p_i ] \\
 = & ~ \sum_{i=1}^k  \underset{ u\sim {\cal D}_k }{\mathbb{E}} \left[ p_i^\top \left( \phi'( u_i)^2 \cdot \left( u_i^2 e_i e_i^\top +  \sum_{j\neq i} u_i u_j ( e_i e_j^\top  + e_je_i^\top) + \sum_{j\neq i} \sum_{l \neq i} u_j u_l e_j e_l^\top  \right) \right) p_i \right] \\
 = & ~ \sum_{i=1}^k  \underset{ u\sim {\cal D}_k }{\mathbb{E}} \left[ p_i^\top \left( \phi'( u_i)^2 \cdot \left( u_i^2 e_i e_i^\top + \sum_{j\neq i}  u_j^2 e_j e_j^\top  \right) \right) p_i \right] \\
 = & ~ \sum_{i=1}^k   \left[ p_i^\top \left( \underset{ u\sim {\cal D}_k }{\mathbb{E}} [\phi'( u_i)^2  u_i^2  ] e_i e_i^\top + \sum_{j\neq i} \underset{ u\sim {\cal D}_k }{\mathbb{E}} [\phi'( u_i)^2    u_j^2  ] e_j e_j^\top \right) p_i \right] \\
 = & ~ \sum_{i=1}^k   \left[ p_i^\top \left( \beta_2 e_i e_i^\top + \sum_{j\neq i} \beta_0  e_j e_j^\top \right) p_i \right] \\
 = & ~ \sum_{i=1}^k p_i^\top ( (\beta_2-\beta_0) e_i e_i^\top + \beta_0 I_k ) p_i \\
 = & ~ (\beta_2-\beta_0) \sum_{i=1}^k p_i^\top e_i e_i^\top p_i + \beta_0 \sum_{i=1}^k p_i^\top p_i \\
 = & ~ (\beta_2-\beta_0) \| \diag(P) \|^2 + \beta_0 \| P \|_F^2,
\end{align*}
where the second step follows by rewriting $uu^\top = \overset{k}{\underset{i=1}{\sum}} \overset{k}{\underset{j=1}{\sum}} u_i u_j e_i e_j^\top$, the third step follows by \\$\underset{ u\sim {\cal D}_k }{\mathbb{E}} [ \phi'(u_i)^2  u_i  u_j ] = 0$, $\forall j\neq i$ and $\underset{ u\sim {\cal D}_k }{\mathbb{E}} [ \phi'(u_i)^2 u_j u_l ] = 0$, $\forall j\neq l$, the fourth step follows by pushing expectation, the fifth step follows by $\underset{ u\sim {\cal D}_k }{\mathbb{E}} [\phi'( u_i)^2  u_i^2  ]=\beta_2$ and $\underset{ u\sim {\cal D}_k }{\mathbb{E}} [\phi'( u_i)^2    u_j^2  ] =\underset{ u\sim {\cal D}_k }{\mathbb{E}} [\phi'( u_i)^2] =\beta_0 $, and the last step follows by $\overset{k}{\underset{i=1}{\sum}} p_{i,i}^2 = \| \diag(P) \|^2$ and $\overset{k}{\underset{i=1}{\sum}} p_i^\top p_i = \overset{k}{\underset{i=1}{\sum}} \| p_i \|^2 = \| P \|_F^2$.

We can rewrite the off-diagonal term in the following way,
\begin{align*}
B = & ~ \sum_{i \neq l} \underset{ u\sim {\cal D}_k }{\mathbb{E}} [ p_i^\top ( \phi'( u_l) \phi'( u_i) \cdot u u^\top) p_l ] \\
= & ~ \sum_{i \neq l} \underset{ u\sim {\cal D}_k }{\mathbb{E}} \left[ p_i^\top \left( \phi'( u_l) \phi'( u_i) \cdot \left( u_i^2 e_i e_i^\top + u_l^2 e_l e_l^\top + u_i u_l (e_i e_l^\top + e_l e_i^\top ) + \sum_{j\neq l} u_i u_j e_i e_j^\top \right.\right.\right. \\
+ & ~ \left.\left.\left. \sum_{j\neq i} u_j u_l e_j e_l^\top + \sum_{j \neq i,l} \sum_{j' \neq i,l} u_j u_{j'} e_j e_{j'}^\top  \right) \right) p_l \right] \\
= & ~ \sum_{i \neq l} \underset{ u\sim {\cal D}_k }{\mathbb{E}} \left[ p_i^\top \left( \phi'( u_l) \phi'( u_i) \cdot \left( u_i^2 e_i e_i^\top + u_l^2 e_l e_l^\top+  u_i u_l (e_i e_l^\top + e_l e_i^\top ) + \sum_{j \neq i,l}  u_j^2 e_j e_{j}^\top  \right) \right) p_l \right] \\
= & ~ \sum_{i \neq l}  \left[ p_i^\top \left( \underset{ u\sim {\cal D}_k }{\mathbb{E}}[ \phi'( u_l) \phi'( u_i)  u_i^2] e_i e_i^\top + \underset{ u\sim {\cal D}_k }{\mathbb{E}}[ \phi'( u_l) \phi'( u_i)  u_l^2] e_l e_l^\top \right.\right. \\
+ & ~ \left.\left. \underset{ u\sim {\cal D}_k }{\mathbb{E}}[ \phi'( u_l) \phi'( u_i)  u_i u_l] (e_i e_l^\top + e_l e_i^\top ) + \sum_{j \neq i,l}  \underset{ u\sim {\cal D}_k }{\mathbb{E}}[ \phi'( u_l) \phi'( u_i) u_j^2] e_j e_{j}^\top   \right) p_l \right] \\
= & ~ \sum_{i \neq l}  \left[ p_i^\top \left( \alpha_0 \alpha_2 (e_i e_i^\top + e_l e_l^\top) + \alpha_1^2 (e_i e_l^\top + e_l e_i^\top ) + \sum_{j \neq i,l}  \alpha_0^2 e_j e_{j}^\top   \right) p_l \right] \\
= & ~ \sum_{i \neq l}  \left[ p_i^\top \left( (\alpha_0 \alpha_2 -\alpha_0^2) (e_i e_i^\top + e_l e_l^\top) + \alpha_1^2 (e_i e_l^\top + e_l e_i^\top ) +  \alpha_0^2 I_k   \right) p_l \right] \\
= & ~ \underbrace{ (\alpha_0 \alpha_2 - \alpha_0^2) \sum_{i\neq l} p_i^\top (e_ie_i^\top + e_l e_l^\top) p_l }_{B_1} + \underbrace{ \alpha_1^2 \sum_{i\neq l} p_i^\top (e_i e_l^\top + e_l e_i^\top) p_l }_{B_2} + \underbrace{ \alpha_0^2 \sum_{i\neq l} p_i^\top p_l }_{B_3},
\end{align*}

where the third step follows by $\underset{u\sim {\cal D}_k }{\E}[ \phi'(u_l) \phi'(u_i) u_i u_j] =0$ and $\underset{u\sim {\cal D}_k }{\E}[ \phi'(u_l) \phi'(u_i) u_{j'} u_j] =0$ for $j'\neq j$.

For the term $B_1$, we have
\begin{align*}
B_1 & = ~ (\alpha_0 \alpha_2 - \alpha_0^2) \sum_{i\neq l} p_i^\top (e_ie_i^\top + e_l e_l^\top) p_l \\
& = ~ 2 (\alpha_0 \alpha_2 - \alpha_0^2) \sum_{i\neq l} p_i^\top e_ie_i^\top  p_l \\
& = ~ 2 (\alpha_0 \alpha_2 - \alpha_0^2) \sum_{i=1}^k p_i^\top e_ie_i^\top  \left( \sum_{l=1}^k p_l - p_i \right) \\
& = ~ 2 (\alpha_0 \alpha_2 - \alpha_0^2) \left( \sum_{i=1}^k p_i^\top e_ie_i^\top \sum_{l=1}^k p_l -   \sum_{i=1}^k p_i^\top e_ie_i^\top  p_i \right) \\
& = ~ 2 (\alpha_0 \alpha_2 - \alpha_0^2) ( \diag(P)^\top \cdot P \cdot {\bf 1} - \| \diag(P) \|^2 )
\end{align*}

For the term $B_2$, we have
\begin{align*}
B_2 = & ~ \alpha_1^2 \sum_{i\neq l} p_i^\top (e_i e_l^\top + e_l e_i^\top) p_l \\
= & ~ \alpha_1^2 \left( \sum_{i\neq l} p_i^\top e_i e_l^\top p_l +\sum_{i\neq l} p_i^\top e_l e_i^\top p_l \right) \\
= & ~ \alpha_1^2 \left( \sum_{i=1}^k \sum_{l=1}^k p_i^\top e_i e_l^\top p_l - \sum_{j=1}^k p_j^\top e_j e_j^\top p_j + \sum_{i=1}^k \sum_{l=1}^k p_i^\top e_l e_i^\top p_l - \sum_{j=1}^k p_j^\top e_j e_j^\top p_j \right) \\
= & ~ \alpha_1^2 ( (\diag(P)^\top {\bf 1} )^2 - \| \diag(P) \|^2 +\langle P, P^\top \rangle - \| \diag(P) \|^2 )
\end{align*}

For the term $B_3$, we have
\begin{align*}
B_3 = & ~ \alpha_0^2 \sum_{i\neq l} p_i^\top p_l \\
= & ~ \alpha_0^2 \left( \sum_{i=1}^k p_i^\top \sum_{l=1}^k p_l - \sum_{i=1}^k p_i^\top p_i \right)\\
= & ~ \alpha_0^2 \left( \left\| \sum_{i=1}^k p_i \right\|^2 - \sum_{i=1}^k \| p_i \|^2 \right) \\
= & ~ \alpha_0^2 ( \| P \cdot {\bf 1 } \|^2 - \| P \|_F^2 )
\end{align*}

%\begin{align*}
%= & ~ \sum_{i=1}^k [ p_i^\top ( \beta_0 I +(\beta_2-\beta_0) e_i e_i^\top ) p_i ]  \\
%+ & ~ \sum_{i\neq l} \left[ p_i^\top ( \alpha_0^2 I +(\alpha_2\alpha_0-\alpha_0^2)( e_i e_i^\top +  e_l e_l^\top)  + \alpha_1^2( e_i e_l^\top +  e_l e_i^\top)) p_l \right]  
%\end{align*}
Let $\diag(P)$ denote a length $k$ column vector where the $i$-th entry is the $(i,i)$-th entry of $P\in \mathbb{R}^{k\times k}$. Furthermore, we can show $A+B$ is,
\begin{align*}
 & ~ A + B \\
= & ~ A + B_1 + B_2 + B_3 \\
= & ~ \underbrace{ (\beta_2-\beta_0) \| \diag(P) \|^2 + \beta_0 \| P \|_F^2 }_{A} + \underbrace{ 2 (\alpha_0 \alpha_2 - \alpha_0^2) ( \diag(P)^\top \cdot P \cdot {\bf 1} - \| \diag(P) \|^2 ) }_{B_1} \\
+ & ~ \underbrace{ \alpha_1^2 ( (\diag(P)^\top \cdot {\bf 1} )^2 - \| \diag(P) \|^2 +\langle P, P^\top \rangle - \| \diag(P) \|^2 ) }_{B_2} + \underbrace{ \alpha_0^2 ( \| P \cdot {\bf 1 } \|^2 - \| P \|_F^2 ) }_{B_3} \\
%= & ~ \alpha_0^2 \left( \left\|\sum_{i=1}^k p_i \right\|^2 - \|P\|_F^2 \right) + 2 (\alpha_2\alpha_0 - \alpha_0^2) \left( \left(\sum_{i=1}^k  e_i e_i^\top p_i \right)^\top \left(\sum_{l=1}^k p_l \right) - \|\diag(P)\|^2 \right) \\
%+ & ~ \alpha_1^2 ((\diag(P)^\top \bone)^2 - \|\diag(P)\|^2 + \langle P, P^\top \rangle - \|\diag(P)\|^2) + \beta_0 \|P\|_F^2 + (\beta_2 - \beta_0)\|\diag(P)\|^2 \\
= & ~ \underbrace{ \|\alpha_0 P \cdot \bone + (\alpha_2 - \alpha_0) \diag(P)\|^2 }_{C_1} + \underbrace{ \alpha_1^2 (\diag(P)^\top \cdot \bone)^2 }_{C_2} + \underbrace{ \frac{\alpha^2_1}{2}\|P+P^\top - 2\diag(\diag(P))\|_F^2 }_{C_3}\\
+ & ~ \underbrace{ (\beta_0 - \alpha_0^2-\alpha_1^2) \|P  - \diag(\diag(P))\|_F^2 }_{C_4}  ~+ \underbrace{ (\beta_2 - \alpha_1^2  - \alpha_2^2) \|\diag(P)\|^2 }_{C_5} \\ 
\geq & ~  (\beta_0 - \alpha_0^2-\alpha_1^2) \|P - \diag(\diag(P))\|_F^2  ~+ (\beta_2 - \alpha_1^2  - \alpha_2^2) \|\diag(P)\|^2 \\ 
\geq & ~ \min\{  (\beta_0 - \alpha_0^2-\alpha_1^2), (\beta_2 - \alpha_1^2  - \alpha_2^2)  \} \cdot (  \|P - \diag(\diag(P))\|_F^2 + \|\diag(P)\|^2 )\\ 
= & ~ \min\{  (\beta_0 - \alpha_0^2-\alpha_1^2), (\beta_2 - \alpha_1^2  - \alpha_2^2)  \} \cdot (  \|P - \diag(\diag(P))\|_F^2 + \|\diag(\diag(P))\|^2 )\\ 
\geq & ~ \min\{(\beta_0 - \alpha_0^2-\alpha_1^2), (\beta_2 - \alpha_1^2  - \alpha_2^2) \}\cdot \|P\|_F^2 \\
=  & ~ \rho \|P\|_F^2,
\end{align*}
where the first step follows by $B=B_1+B_2+B_3$, and the second step follows by the definition of $A,B_1,B_2,B_3$ the third step follows by $A+B_1+B_2+B_3=C_1+C_2+C_3+C_4+C_5$, the fourth step follows by $C_1, C_2, C_3 \geq 0$, the fifth step follows $a\geq \min(a,b)$, the sixth step follows by $\| \diag(P) \|^2 = \| \diag(\diag(P)) \|^2$, the seventh step follows by triangle inequality, and the last step follows the definition of $\rho$.
\end{proof}

\begin{claim}
$A+B_1+B_2+B_3=C_1+C_2+C_3+C_4+C_5$.
\end{claim}
\begin{proof}
The key properties we need are, for two vectors $a,b$, $\| a + b\|^2 = \|a\|^2 + 2\langle a, b\rangle +\|b\|^2$; for two matrices $A,B$, $\| A+B \|_F^2 = \|A\|_F^2 + 2\langle A, B \rangle + \| B \|_F^2$.
Then, we have
\begin{align*}
  & ~ C_1 + C_2 +C_3+C_4 +C_5 \\
= & ~ \underbrace{ (\|\alpha_0 P \cdot \bone\|)^2 + 2 (\alpha_0\alpha_2 -\alpha_0^2) \langle  P \cdot \bone,  \diag(P) \rangle + (\alpha_2 - \alpha_0)^2 \| \diag(P)\|^2 }_{C_1} + \underbrace{ \alpha_1^2 ( \diag(P)^\top \cdot \bone)^2 }_{C_2} \\
+ & ~ \underbrace{ \frac{\alpha_1^2}{2} ( 2\|P\|_F^2  + 4 \| \diag(\diag(P)) \|_F^2 + 2 \langle P, P^\top \rangle - 4 \langle P, \diag(\diag(P)) \rangle - 4 \langle P^\top, \diag(\diag(P)) \rangle ) }_{C_3} \\
+ & ~ \underbrace{ (\beta_0-\alpha_0^2 -\alpha_1^2) (\| P\|_F^2 -2 \langle P, \diag(\diag(P)) \rangle + \| \diag(\diag(P)) \|_F^2) }_{C_4} + \underbrace{ (\beta_2 -\alpha_1^2 - \alpha_2^2) \| \diag (P) \|^2 }_{C_5}\\
= & ~ \underbrace{ \alpha_0^2 \| P \cdot \bone\|^2 + 2 (\alpha_0\alpha_2 -\alpha_0^2) \langle  P \cdot \bone,  \diag(P) \rangle + (\alpha_2 - \alpha_0)^2 \| \diag(P)\|^2 }_{C_1} + \underbrace{ \alpha_1^2 ( \diag(P)^\top \cdot \bone)^2 }_{C_2} \\
+ & ~ \underbrace{ \frac{\alpha_1^2}{2} ( 2\|P\|_F^2  + 4 \|\diag(P) \|^2 + 2 \langle P, P^\top \rangle - 8 \| \diag(P) \|^2 ) }_{C_3} \\
+ & ~ \underbrace{ (\beta_0-\alpha_0^2 -\alpha_1^2) (\| P\|_F^2 -2 \|\diag(P)\|^2 + \| \diag(P) \|^2 ) }_{C_4} + \underbrace{ (\beta_2 -\alpha_1^2 - \alpha_2^2) \| \diag (P) \|^2 }_{C_5} \\
= & ~ \alpha_0^2 \| P \cdot \bone\|^2 + 2(\alpha_0\alpha_2 -\alpha_0^2) \diag(P)^\top \cdot P \cdot \bone + \alpha_1^2 (\diag(P)^\top \cdot \bone)^2 + \alpha_1^2 \langle P, P^\top \rangle \\
+ & ~ (\beta_0 -\alpha_0^2 ) \| P\|_F^2 + \underbrace{ ( (\alpha_2-\alpha_0)^2 -2\alpha_1^2 -\beta_0 +\alpha_0^2 + \alpha_1^2 +\beta_2 -\alpha_1^2 -\alpha_2^2 ) }_{ \beta_2 - \beta_0 -2 (\alpha_2\alpha_0 - \alpha_0^2+ \alpha_1^2)} \| \diag(P) \|^2 \\
= & ~ \underbrace{ 0 }_{\text{part~of~}A}+ \underbrace{ 2 (\alpha_2\alpha_0 - \alpha_0^2)\cdot \diag(P)^\top P \cdot \bone }_{\text{part~of~}B_1}  + \underbrace{ \alpha_1^2 \cdot ((\diag(P)^\top \bone)^2 + \langle P, P^\top \rangle) }_{\text{part~of~}B_2} + \underbrace{ \alpha_0^2 \cdot \|P \cdot \bone\|^2 }_{\text{part~of~}B_3} \\
+ & ~ \underbrace{ (\beta_0 -\alpha_0^2) \cdot \|P\|_F^2 }_{\text{proportional~to~}\|P\|_F^2} + \underbrace{ (\beta_2 - \beta_0 -2 (\alpha_2\alpha_0 - \alpha_0^2+ \alpha_1^2))\cdot \|\diag(P)\|^2 }_{\text{proportional~to~}\|\diag(P)\|^2} \\
= & ~ \underbrace{ (\beta_2-\beta_0) \| \diag(P) \|^2 + \beta_0 \| P \|_F^2 }_{A} + \underbrace{ 2 (\alpha_0 \alpha_2 - \alpha_0^2) ( \diag(P)^\top \cdot P \cdot {\bf 1} - \| \diag(P) \|^2 ) }_{B_1} \\
+ & ~ \underbrace{ \alpha_1^2 ( (\diag(P)^\top \cdot {\bf 1} )^2 - \| \diag(P) \|^2 +\langle P, P^\top \rangle - \| \diag(P) \|^2 ) }_{B_2} + \underbrace{ \alpha_0^2 ( \| P \cdot {\bf 1 } \|^2 - \| P \|_F^2 ) }_{B_3}\\
= & A + B_1 + B_2 + B_3
\end{align*}
where the second step follows by $\langle P, \diag(\diag(P))\rangle = \| \diag(P) \|^2$ and $\| \diag(\diag(P)) \|_F^2 = \| \diag(P) \|^2$.
\end{proof}

\subsubsection{Lower Bound on the Eigenvalues of Hessian for Non-orthogonal Case}\label{app:lower_bound_hessian_non_ortho}

First we show the lower bound of the eigenvalues. The main idea is to reduce the problem to a $k$-by-$k$ problem and then lower bound the eigenvalues using orthogonal weight matrices.

\begin{lemma}[Lower bound]\label{lem:lower_bound_pd_ground}
If $\phi(z)$ satisfies Property~\ref{pro:gradient},\ref{pro:expect} and \ref{pro:hessian}, we have the following property for the second derivative of function $f_{\cal D}(W)$ at $W^*$,
\begin{align*}
 \Omega(v_{\min}^2 \rho(\sigma_k) / ( \kappa^2 \lambda ) ) I \preceq \nabla^2 f_{\cal D}(W^*).
\end{align*}
\end{lemma}

\begin{proof}
Let $a\in \mathbb{R}^{dk}$ denote vector $ \begin{bmatrix}  a_1^\top &  a_2^\top & \cdots &  a_k^\top \end{bmatrix}^\top$, let $b\in \mathbb{R}^{dk}$ denote vector $ \begin{bmatrix}  b_1^\top &  b_2^\top & \cdots &  b_k^\top \end{bmatrix}^\top $ and let $c \in \mathbb{R}^{dk}$ denote vector $\begin{bmatrix} c_1^\top & c_2^\top & \cdots & c_k^\top \end{bmatrix}^\top $.
The smallest eigenvalue of the Hessian can be calculated by 
\begin{align}
\nabla^2 f(W^*) & \succeq \min_{ \| a \|=1} a^\top \nabla^2 f(W^*) a ~I_{dk} = \min_{ \| a\|=1} \underset{x\sim {\cal D}_d}{\E} \left[ \left(\sum_{i=1}^k v_i^*  a_i^\top x \cdot \phi'( w_i^{*\top}x ) \right)^2 \right] ~ I_{dk} \label{eq:sm_eig_cal}
\end{align}
Note that 
\begin{align}
 & ~ \min_{ \| a \| = 1} \underset{x\sim {\cal D}_d}{\E} \left[ \left(\sum_{i=1}^k (v_i^*  a_i)^\top x \cdot \phi'( w_i^{*\top}x ) \right)^2 \right]  \notag \\
= & ~\min_{ \| a \| \neq 0} \underset{x\sim {\cal D}_d}{\E} \left[ \left(\sum_{i=1}^k (v_i^*  a_i)^\top x \cdot \phi'( w_i^{*\top}x ) \right)^2  \right]  / \| a \|^2 \notag \\
 = & ~ \min_{  \sum_{i}  \|  b_i/v_i^*\|^2\neq 0} \underset{x\sim {\cal D}_d}{\E} \left[ \left(\sum_{i=1}^k  b_i^\top x \cdot \phi'( w_i^{*\top}x ) \right)^2  \right] / \left( \sum_{i=1}^k \|  b_i/v_i^*\|^2 \right) & \text{~by~} a_i=b_i/v_i^* \notag \\
  = & ~ \min_{ \sum_i \|  b_i\|^2\neq 0} \underset{x\sim {\cal D}_d}{\E} \left[ \left(\sum_{i=1}^k  b_i^\top x \cdot \phi'( w_i^{*\top}x )\right)^2 \right] / \left( \sum_{i=1}^k \|  b_i/v_i^*\|^2 \right) \notag \\
\geq &  ~ v_{\min}^2 \min_{ \sum_i \|  b_i\|^2\neq 0} \underset{x\sim {\cal D}_d}{\E} \left[ \left(\sum_{i=1}^k  b_i^\top x \cdot \phi'( w_i^{*\top}x ) \right)^2 \right] / \left( \sum_{i=1}^k \|  b_i\|^2 \right)  & \text{~by~} v_{\min} = \min_{i\in [k]} |v_i^*|  \notag \\
= & ~ v_{\min}^2 \min_{ \| a\|=1} \underset{x\sim {\cal D}_d}{\E} \left[ \left(\sum_{i=1}^k  a_i^\top x \cdot \phi'( w_i^{*\top}x ) \right)^2 \right] \label{eq:min_eigenvalue}
\end{align}
%Let ${\cal B}(k,d)$ denote the set $\{  a | \sum_{i=1}^k \|  a_i \|^2 =1, \forall i\in [k],  a_i \in \mathbb{R}^d \}$. 
Let $U\in \mathbb{R}^{d\times k}$ be the orthonormal basis of $W^*$ and let $V=\begin{bmatrix} v_1 & v_2 & \cdots & v_k \end{bmatrix}=U^\top W^* \in \mathbb{R}^{k\times k}$. 
%To clarify the notation, here $ v_i^*$ (bold font) is a vector and $v_i^*$ (plain font) is a scalar, which is the ground truth parameter in the second layer of our model. 
Also note that $V$ and $W^*$ have same singular values and $W^* = UV$.  We use $U_{\bot} \in \mathbb{R}^{d \times (d-k)}$ to denote the complement of $U$. For any vector $ a_j \in \mathbb{R}^{d}$, there exist two vectors $ b_j\in \mathbb{R}^k$ and $ c_j \in \mathbb{R}^{d-k}$ such that
\begin{align*}
 a_j = U  b_j + U_{\bot}  c_j.
\end{align*}
We use ${\cal D}_d$ to denote Gaussian distribution ${\cal N}(0,I_d)$, ${\cal D}_{d-k}$ to denote Gaussian distribution ${\cal N}(0,I_{d-k})$, and ${\cal D}_k$ to denote Gaussian distribution ${\cal N}(0,I_k)$. Then we can rewrite formulation~\eqref{eq:min_eigenvalue} (removing $v_{\min}^2$) as
\begin{align*}
& \underset{x\sim {\cal D}_d }{\mathbb{E}} \left[ \left(\sum_{i=1}^k  a_i^\top x \cdot \phi'( w_i^{*\top}x ) \right)^2 \right] 
=  \underset{x\sim {\cal D}_d }{\mathbb{E}} \left[ \left(\sum_{i=1}^k ( b_i^\top U^\top+ c_i^\top U_\perp^{\top})x \cdot \phi'( w_i^{*\top}x ) \right)^2 \right] =  A + B + C \\
\end{align*}
where 
\begin{align*}
A = & ~\underset{x \sim {\cal D}_d}{ \mathbb{E} } \left[  \left(\sum_{i=1}^k  b_i^\top U^\top x \cdot \phi'( w_i^{*\top}x ) \right)^2 \right], \\
B = & ~\underset{x \sim {\cal D}_d}{ \mathbb{E} } \left[ \left( \sum_{i=1}^k  c_i^\top U_\perp^{\top} x \cdot \phi'( w_i^{*\top}x ) \right)^2 \right],\\
C = & ~ \underset{x \sim {\cal D}_d}{ \mathbb{E} } \left[ 2 \left(\sum_{i=1}^k  b_i^\top U^\top x \cdot \phi'( w_i^{*\top} x ) \right) \cdot \left(\sum_{i=1}^k  c_i^\top U_\perp^{\top} x \cdot \phi'( w_i^{*\top}x ) \right) \right].
\end{align*}
We calculate $A,B,C$ separately. First, we can show
\begin{align*}
A & = \underset{x\sim {\cal D}_d }{\mathbb{E}} \left[  \left(\sum_{i=1}^k  b_i^\top U^\top x \cdot \phi'( w_i^{*\top}x ) \right)^2   \right] = \underset{ z\sim {\cal D}_k }{\mathbb{E}} \left[ \left(\sum_{i=1}^k  b_i^\top  z \cdot \phi'( v_i^{\top} z ) \right)^2   \right].
\end{align*}
Second, we can show
\begin{align*}
B & = \underset{x\sim {\cal D}_d }{\mathbb{E}} \left [ \left(\sum_{i=1}^k  c_i^\top U_\perp^{\top} x \cdot \phi'( w_i^{*\top}x ) \right)^2 \right]\\
& = \underset{ s\sim {\cal D}_{d-k}, z\sim {\cal D}_k }{\mathbb{E}} \left[\left(\sum_{i=1}^k  c_i^\top  s \cdot \phi'( v_i^{*\top} z ) \right)^2\right]\\
& = \underset{ s\sim {\cal D}_{d-k}, z\sim {\cal D}_k }{\mathbb{E}} [ (y^\top  s  )^2] & \text{~by~defining~}y=\sum_{i=1}^k \phi'( v_i^{*\top} z )  c_i \in \mathbb{R}^{d-k}\\
& = \underset{ z\sim {\cal D}_k }{\mathbb{E}} \left[ \underset{ s\sim {\cal D}_{d-k} }{\mathbb{E}} [( y^\top  s  )^2] \right]\\
& = \underset{ z\sim {\cal D}_k }{\mathbb{E}} \left[\underset{ s\sim {\cal D}_{d-k} }{\mathbb{E}} \left[ \sum_{j=1}^{d-k} s_j^2 y_j^2 \right] \right] & \text{~by~} \mathbb{E}[s_j s_{j'}]=0 \\
& = \underset{ z\sim {\cal D}_k }{\mathbb{E}} \left[  \sum_{j=1}^{d-k} y_j^2 \right] & \text{~by~}s_j \sim {\cal N}(0,1)\\
& =  \underset{ z\sim {\cal D}_k }{\mathbb{E}} \left[ \left\|\sum_{i=1}^k \phi'( v_i^{*\top} z ) c_i \right\|^2 \right] & \text{~by~definition~of~}y
\end{align*}
Third, we have $C=0$ since $U_\perp^\top x$ is independent of $ w_i^{*\top}x$ and $U^\top x$. Thus, putting them all together, 
\begin{align*}
& \underset{x\sim {\cal D}_d }{\mathbb{E}} \left[ \left(\sum_{i=1}^k  a_i^\top x \cdot \phi'( w_i^{*\top}x ) \right)^2 \right]= \underset{ z\sim {\cal D}_k }{\mathbb{E}} \left[  \left( \sum_{i=1}^k  b_i^\top  z \cdot \phi'( v_i^{\top} z ) \right)^2 \right]  + \underset{ z\sim {\cal D}_k }{\mathbb{E}} \left[ \left\|\sum_{i=1}^k \phi'( v_i^{\top} z ) c_i \right\|^2 \right] \\
\end{align*}

\begin{comment}
\begin{align*}
%= & \min_{  a \in {\cal B}(k,d)} \underset{x\sim {\cal D}(d) }{\mathbb{E}}[  (\sum_{i=1}^k ( b_i^\top U^\top x \cdot \phi'( w_i^{*\top}x ))^2  \\
%&+ \sum_{i=1}^k  c_i^\top U_\perp^{\top} x \cdot \phi'( w_i^{*\top}x ))^2 ] \\
= &  \mathbb{E}[  (\sum_{i=1}^k  b_i^\top U^\top x \cdot \phi'( w_i^{*\top}x ))^2   \\
& + (\sum_{i=1}^k  c_i^\top U_\perp^{\top} x \cdot \phi'( w_i^{*\top}x ))^2 \\
& + 2(\sum_{i=1}^k  b_i^\top U^\top x \cdot \phi'( w_i^{*\top}x ))(\sum_{i=1}^k  c_i^\top U_\perp^{\top} x \cdot \phi'( w_i^{*\top}x ))] \\
= &  \underset{x\sim {\cal D}(d) }{\mathbb{E}}[  (\sum_{i=1}^k  b_i^\top U^\top x \cdot \phi'( w_i^{*\top}x ))^2   ] \\
& + \underset{x\sim {\cal D}(d) }{\mathbb{E}} [(\sum_{i=1}^k  c_i^\top U_\perp^{\top} x \cdot \phi'( w_i^{*\top}x ))^2] \\
= &  \underset{ z\sim {\cal D}(k) }{\mathbb{E}}[  (\sum_{i=1}^k  b_i^\top  z \cdot \phi'( v_i^{*\top} z ))^2   ] \\
& + \underset{ s\sim {\cal D}(d-k), z\sim {\cal D}(k) }{\mathbb{E}} [(\sum_{i=1}^k  c_i^\top  s \cdot \phi'( v_i^{*\top} z ))^2] \\
= &  \underset{ z\sim {\cal D}(k) }{\mathbb{E}}[  (\sum_{i=1}^k  b_i^\top  z \cdot \phi'( v_i^{*\top} z ))^2   ] \\
& + \underset{ s\sim {\cal D}(d-k), z\sim {\cal D}(k) }{\mathbb{E}} [(\sum_{i=1}^k  c_i^\top  s \cdot \phi'( v_i^{*\top} z ))^2] \\
= &  \underset{ z\sim {\cal D}(k) }{\mathbb{E}}[  (\sum_{i=1}^k  b_i^\top  z \cdot \phi'( v_i^{*\top} z ))^2   ] \\
& + \underset{ z\sim {\cal D}(k) }{\mathbb{E}} [\|\sum_{i=1}^k \phi'( v_i^{*\top} z ) c_i\|^2] \\
\end{align*}
\end{comment}

Let us lower bound $A$,
\begin{align*}
A = &  \underset{ z\sim {\cal D}_k }{\mathbb{E}} \left[  \left(\sum_{i=1}^k  b_i^\top  z \cdot \phi'(v_i^{\top} z ) \right)^2 \right]  \\
= &  \int (2\pi)^{-k/2}  \left(\sum_{i=1}^k  b_i^\top  z \cdot \phi'( v_i^{\top} z ) \right)^2  e^{-\| z\|^2/2}d  z  \\
\underset{\xi_1}{=} &  \int (2\pi)^{-k/2}  \left(\sum_{i=1}^k  b_i^\top V^{\dagger\top}  s \cdot \phi'(s_i) \right)^2  e^{-\|V^{\dagger\top}  s\|^2/2} \cdot |\det(V^\dagger)| d s  \\
\underset{\xi_2}{\geq}  &  \int (2\pi)^{-k/2}  \left(\sum_{i=1}^k  b_i^\top V^{\dagger\top}  s \cdot \phi'(s_i) \right)^2  e^{-\sigma^2_1(V^{\dagger}) \| s\|^2/2}  \cdot |\det(V^\dagger)|  d  s  \\
\underset{\xi_3}{=}  &  \int (2\pi)^{-k/2}  \left(\sum_{i=1}^k  b_i^\top V^{\dagger\top} u/\sigma_1(V^{\dagger}) \cdot \phi'(u_i/\sigma_1(V^{\dagger})) \right)^2  e^{- \|u\|^2/2} |\det(V^\dagger)|/\sigma_1^k(V^{\dagger})  d u  \\
= &  \int (2\pi)^{-k/2}  \left(\sum_{i=1}^k p_i^\top u \cdot \phi'(\sigma_k\cdot u_i) \right)^2  e^{- \|u\|^2/2} \frac{1}{\lambda}  d u  \\
= &  \frac{1}{\lambda} \underset{ u\sim {\cal D}_k }{\mathbb{E}} \left[ \left(\sum_{i=1}^k p_i^\top u \cdot \phi'(\sigma_k \cdot u_i) \right)^2 \right]  
\end{align*}
where $V^\dagger \in \mathbb{R}^{k\times k}$ is the inverse of $V\in \mathbb{R}^{k\times k}$, i.e., $V^\dagger V=I$, $p_i^\top =  b_i^\top V^{\dagger\top} /\sigma_1(V^{\dagger}) $ and $\sigma_k = \sigma_k(W^*)$. $\xi_1$ replaces $ z$ by $ z = V^{\dagger\top}  s$, so $ v_i^{*\top} z = s_i$. $\xi_2$ uses the fact $\|V^{\dagger\top} s\| \leq \sigma_1(V^\dagger) \| s\|$. $\xi_3$ replaces $ s$ by $ s = u/\sigma_1(V^\dagger)$. Note that $\phi'(\sigma_k \cdot u_i)$'s are independent of each other, so we can simplify the analysis. 
% Then applying Lemma~\ref{lem:pop_smallest_eigenvalue} completes the we we  proof. 
In particular, Lemma~\ref{lem:lower_bound_ortho} gives a lower bound in this case in terms of $p_i$.
Note that $\|p_i\| \geq \| b_i\|/\kappa $. Therefore,
\begin{align*}
  \underset{ z\sim {\cal D}_k }{\mathbb{E}} \left[ \left(\sum_{i=1}^k  b_i^\top  z \cdot \phi'(v_i^{\top} z ) \right)^2   \right] \geq \rho(\sigma_k) \frac{1}{\kappa^2 \lambda} \|b\|^2.
\end{align*}

For $B$, similar to the proof of Lemma~\ref{lem:lower_bound_ortho}, we have,
\begin{align*}
B = & ~\underset{ z\sim {\cal D}_k }{\mathbb{E}} \left[ \left\|\sum_{i=1}^k \phi'( v_i^{\top} z ) c_i \right\|^2 \right] \\
= &  ~\int (2\pi)^{-k/2}  \left\|\sum_{i=1}^k \phi'( v_i^{\top} z ) c_i \right\|^2 e^{-\| z\|^2/2}d  z  \\
= & ~  \int (2\pi)^{-k/2}  \left\|\sum_{i=1}^k \phi'(\sigma_k\cdot u_i ) c_i \right \|^2 e^{-\|V^{\dagger\top}  u/\sigma_1(V^\dagger)\|^2/2} \cdot  \det(V^{\dagger} /\sigma_1(V^\dagger)) d u  \\
= & ~  \int (2\pi)^{-k/2}  \left\|\sum_{i=1}^k \phi'(\sigma_k\cdot u_i ) c_i \right \|^2 e^{-\|V^{\dagger\top}  u/\sigma_1(V^\dagger)\|^2/2} \cdot  \frac{1}{\lambda} d u  \\
\geq & ~  \int (2\pi)^{-k/2}  \left\|\sum_{i=1}^k \phi'(\sigma_k\cdot u_i ) c_i \right \|^2 e^{-\|  u\|^2/2} \cdot  \frac{1}{\lambda} d u  \\
= & ~ \frac{1}{\lambda} \underset{ u\sim {\cal D}_k}{\mathbb{E}} \left[\left\|\sum_{i=1}^k \phi'(\sigma_k\cdot u_i) c_i \right\|^2 \right]  \\
= & ~ \frac{1}{\lambda} \left( \sum_{i=1}^k  \underset{ u\sim {\cal D}_k }{\mathbb{E}} [ \phi'(\sigma_k\cdot u_i) \phi'(\sigma_k\cdot u_i)  c_i^\top  c_i] +  \sum_{i\neq l}  \underset{ u\sim {\cal D}_k }{\mathbb{E}} [ \phi'(\sigma_k\cdot u_i) \phi'(\sigma_k\cdot u_l)  c_i^\top  c_l] \right)\\
= & ~ \frac{1}{\lambda} \left(   \underset{ z\sim {\cal D}_1 }{\mathbb{E}} [ \phi'(\sigma_k\cdot u_i)^2]  \sum_{i=1}^k \| c_i \|^2 +   \left( \underset{ z\sim {\cal D}_1 }{\mathbb{E}} [ \phi'(\sigma_k\cdot z)  ] \right)^2  \sum_{i\neq l} c_i^\top  c_l \right)\\
= & ~ \frac{1}{\lambda} \left( \left(\underset{ z\sim {\cal D}_1 }{\E} [\phi'(\sigma_k\cdot z)] \right)^2 \left\|\sum_{i=1}^k  c_i\right\|_2^2 +  \left (\underset{ z\sim {\cal D}_1 }{\E} [\phi'(\sigma_k\cdot z)^2 ]  -  \left( \underset{ z\sim {\cal D}_1 }{\E}[\phi'(\sigma_k\cdot z) ] \right)^2 \right) \|c\|^2 \right) \\
\geq & ~ \frac{1}{\lambda} \left(\underset{ z\sim {\cal D}_1 }{\E} [\phi'(\sigma_k\cdot z)^2 ]  -  \left( \underset{ z\sim {\cal D}_1 }{\E}[\phi'(\sigma_k\cdot z) ] \right)^2 \right) \| c\|^2 \\
\geq & ~ \rho(\sigma_k) \frac{1}{ \lambda}\|c\|^2, 
\end{align*}
where the first step follows by definition of Gaussian distribution, the second step follows by replacing $ z$ by $ z = V^{\dagger\top} u/\sigma_1(V^\dagger)$, and then $v_i^\top z =  u_i /\sigma_1(V^\dagger) =  u_i \sigma_k(W^*)$, the third step follows by $\| u\|^2 \geq \| \frac{1}{\sigma_1(V^\dagger)} {V^\dagger}^\top u \|^2$ , the fourth  step follows by $\det(V^\dagger /\sigma_1(V^\dagger)) = \det(V^\dagger) / \sigma_1^k(V^\dagger) = 1/\lambda$, the fifth step follows by definition of Gaussian distribution, the ninth step follows by $x^2\geq 0$ for any $x\in \mathbb{R}$,  and the last step follows  by Property~\ref{pro:expect}.

Note that $1 = \|a\|^2 = \|b\|^2+\|c\|^2 $. Thus, we finish the proof for the lower bound. 
\end{proof}

 \subsubsection{Upper Bound on the Eigenvalues of Hessian for Non-orthogonal Case}\label{app:upper_bound_hessian_non_ortho}

\begin{lemma}[Upper bound]\label{lem:upper_bound_pd_ground}
If $\phi(z)$ satisfies Property~\ref{pro:gradient},\ref{pro:expect} and \ref{pro:hessian}, we have the following property for the second derivative of function $f_{\cal D}(W)$ at $W^*$,
\begin{align*}
 \nabla^2 f_{\cal D}(W^*) \preceq O(kv_{\max}^2 \sigma_1^{2p}) I.
\end{align*}
\end{lemma}
\begin{proof}

Similarly, we can calculate the upper bound of the eigenvalues by
\begin{align*}
&~ \| \nabla^2 f(W^*) \| \\
 = &~\max_{ \| a\|=1} a^\top \nabla^2 f(W^*) a\\
 = &~v_{\max}^2 \max_{ \| a \|=1} \underset{x\sim {\cal D}_k }{\mathbb{E}} \left[ \left(\sum_{i=1}^k  a_i^\top x \cdot  \phi'( w_i^{*\top}x ) \right)^2 \right] \\
 = &~v_{\max}^2 \max_{ \| a \|=1} \sum_{i=1}^k \sum_{l=1}^k \underset{x \sim {\cal D}_k }{\mathbb{E}}[  a_i^\top x \cdot  \phi'( w_i^{*\top}x ) \cdot   a_l^\top x \cdot  \phi'( w_l^{*\top}x )  ] \\
\leq & ~ v_{\max}^2 \max_{ \| a \|=1} \sum_{i=1}^k \sum_{l=1}^k \left(\underset{x \sim {\cal D}_k }{\mathbb{E}}[ ( a_i^\top x)^4] \cdot \underset{x \sim {\cal D}_k }{\mathbb{E}}[(\phi'( w_i^{*\top}x ))^4]  \cdot \underset{x \sim {\cal D}_k}{\mathbb{E}}[ ( a_l^\top x)^4] \cdot \underset{x \sim {\cal D}_k }{ \mathbb{E}}[(\phi'( w_l^{*\top}x ))^4]   \right)^{1/4} \\
\lesssim  & ~ v_{\max}^2  \max_{ \| a\|=1} \sum_{i=1}^k \sum_{l=1}^k \| a_i\| \cdot \| a_l\| \cdot \| w_i^*\|^p \cdot \| w_l^*\|^p \\
\leq  & ~ v_{\max}^2  \max_{ \| a\|=1} \sum_{i=1}^k \sum_{l=1}^k \| a_i\| \cdot \| a_l\| \cdot \sigma_1^{2p} \\
 \leq & ~ k v_{\max}^2  \sigma_1^{2p},
\end{align*}
where the first inequality follows H\"{o}lder's inequality, the second inequality follows by  Property~\ref{pro:gradient}, the third inequality follows by $\| w_i^* \| \leq \sigma_1(W^*)$, and the last inequality by Cauchy-Schwarz inequality.

\end{proof}

\subsection{Error Bound of Hessians near the Ground Truth for Smooth Activations}
The goal of this Section is to prove Lemma~\ref{lem:emp_pop_smooth}
\begin{lemma}[Error bound of Hessians near the ground truth for smooth activations]\label{lem:emp_pop_smooth}
Let $\phi(z)$ satisfy Property~\ref{pro:gradient}, Property~\ref{pro:expect} and Property~\ref{pro:hessian}(a). Let $W$ satisfy $\| W - W^* \| \leq \sigma_k/2$. Let $S$ denote a set of i.i.d. samples from the distribution defined in \eqref{eq:model}. Then for any $t\geq 1$ and $0 < \epsilon < 1/2$, if 
\begin{align*}
|S| \geq \epsilon^{-2} d \kappa^2 \tau \cdot \poly(\log d, t)
\end{align*}
then we have, with probability at least $1-1/d^{\Omega(t)}$,
\begin{align*}
\| \nabla^2 \wh{f}_S (W) - \nabla^2 f_{\D} (W^*) \| \lesssim v_{\max}^2 k \sigma_1^p (\epsilon \sigma_1^p + \| W - W^* \|). 
\end{align*}
\end{lemma}
\begin{proof}
This follows by combining Lemma~\ref{lem:smooth_pop_local} and Lemma~\ref{lem:emp_pop} directly.
\end{proof}
\subsubsection{Second-order Smoothness near the Ground Truth for Smooth Activations}\label{proof:lemma:smooth_pop_local}
The goal of this Section is to prove Lemma~\ref{lem:smooth_pop_local}. 

\begin{fact}
Let $w_i$ denote the $i$-th column of $W\in \mathbb{R}^{d\times k}$, and $w_i^*$ denote the $i$-th column of $W^*\in \mathbb{R}^{d\times k}$. If $\|W-W^*\| \leq \sigma_k(W^*)/2$, then for all $i\in [k]$,
\begin{align*}
\frac{1}{2}\| w_i^*\| \leq \| w_i\| \leq \frac{3}{2}\| w_i^*\|.
\end{align*}
\end{fact}

\begin{proof}
Note that if $ \|W-W^*\| \leq \sigma_k(W^*)/2 $, we have $\sigma_k(W^*)/2 \leq \sigma_i(W) \leq \frac{3}{2}\sigma_1(W^*)$ for all $i\in[k]$ by Weyl's inequality.  By definition of singular value, we have $\sigma_k(W^*) \leq \| w_i^*\|\leq \sigma_1(W^*)$. By definition of spectral norm, we have $\| w_i -  w_i^*\| \leq \|W-W^*\|$.
Thus, we can lower bound $\| w_i \|$,
\begin{align*}
\| w_i \| \leq \| w_i^* \| + \| w_i - w_i^* \| \leq \| w_i^* \| + \| W - W^* \| \leq \| w_i^* \| + \sigma_k /2 \leq \frac{3}{2} \| w_i^*\|. 
\end{align*}
Similarly, we have $\| w_i \| \geq \frac{1}{2} \| w_i^*\|$.
\end{proof}

\begin{lemma}[Second-order smoothness near the ground truth for smooth activations]\label{lem:smooth_pop_local}
If $\phi(z)$ satisfies Property~\ref{pro:gradient}, Property~\ref{pro:expect} and Property~\ref{pro:hessian}(a), 
 then for any $W$ with $ \|W-W^*\| \leq \sigma_k/2 $, we have
$$\| \nabla^2 f_{\cal D}(W) - \nabla^2 f_{\cal D}(W^*)\| \lesssim k^2  v_{\max}^{2} \sigma_1^p \|W-W^*\|.$$
\end{lemma}

\begin{proof}

Let $\Delta =\nabla^2 f_{\cal D}(W) - \nabla^2 f_{\cal D}(W^*)$. For each $(i,l)\in [k] \times [k]$, let $\Delta_{i,l} \in \mathbb{R}^{d \times d}$ denote the $(i,l)$-th block of $\Delta $. 
Then, for $i\neq l$, we have
\begin{align*}
\Delta_{i,l} = & ~ \underset{x\sim {\cal D}_d}{\E} \left[ v^*_i v^*_l \left( \phi'( w_i^\top x)  \phi'( w_l^\top x) -\phi'( w_i^{*\top} x)  \phi'( w_l^{*\top} x) \right)  x x^\top \right],
\end{align*}
and for $i= l$, we have {%\small
\begin{align}\label{eq:diag_population} 
 & ~ \Delta_{i,i} \notag \\
= & ~ \underset{x\sim {\cal D}_d}{\E} \left[ \left(  \sum_{r=1}^k v_r^* \phi( w_{r}^\top x)  - y \right) v_i^* \phi''( w_i^\top x) x x^\top + v^{*2}_i \left( \phi'^2( w_i^\top x) -\phi'^2 ( w_i^{*\top} x)  \right) x x^\top \right] \notag \\
= & ~ \underset{x\sim {\cal D}_d}{\E} \left[ \left(  \sum_{r=1}^k v_r^* \phi( w_{r}^\top x)  - y \right) v_i^* \phi''( w_i^\top x) x x^\top \right] + \underset{x\sim {\cal D}_d}{\E} \left[ v^{*2}_i \left( \phi'^2( w_i^\top x) -\phi'^2 ( w_i^{*\top} x)  \right) x x^\top \right]. 
\end{align}
}

In the next a few paragraphs, we first show how to bound the off-diagonal term, and then show how to bound the diagonal term.  

First, we consider off-diagonal terms. 
\begin{align}\label{eq:smooth_pop_local_off_diagonal}
 & ~ \| \Delta_{i,l} \| \notag \\
= & ~ \left\| \underset{x\sim {\cal D}_d }{\E} \left[  v^*_i v^*_l \left( \phi'( w_i^\top x)  \phi'( w_l^\top x) -\phi'( w_i^{*\top} x)  \phi'( w_l^{*\top} x) \right)  x x^\top \right] \right\|  \notag \\
\leq & ~ v_{\max}^2\max_{\| a\|=1} \underset{x\sim {\cal D}_d}{ \E} \left[   \left|\phi'( w_i^\top x)  \phi'( w_l^\top x) -\phi'( w_i^{*\top} x)  \phi'( w_l^{*\top} x) \right| \cdot (x^\top  a)^2 \right] \notag \\
\leq & ~ v_{\max}^2 \max_{\| a\|=1}  \underset{x\sim {\cal D}_d}{ \E} \left[   \left( |\phi'( w_i^\top x)-\phi'( w_i^{*\top} x)| |\phi'( w_l^\top x)|  + |\phi'( w_i^{*\top} x)| |\phi'( w_l^\top x) -  \phi'( w_l^{*\top} x)| \right)  (x^\top a)^2 \right] \notag \\
= & ~ v_{\max}^2 \max_{\| a\|=1} \left( \underset{x\sim {\cal D}_d}{ \E} \left[   |\phi'( w_i^\top x)-\phi'( w_i^{*\top} x)| |\phi'( w_l^\top x)|  (x^\top  a)^2 \right] \right. \notag \\ 
& ~ + \left. \underset{x\sim {\cal D}_d}{ \E} \left[ |\phi'( w_i^{*\top} x)| |\phi'( w_l^\top x) -  \phi'( w_l^{*\top} x)|   (x^\top a)^2 \right] \right) \notag \\
\leq & ~ v_{\max}^2 \max_{\| a\|=1}\left( \underset{x\sim {\cal D}_d}{\E} \left[   L_2 |( w_i- w_i^*)^\top x| |\phi'( w_l^\top x)|  (x^\top a)^2 \right] + \underset{x\sim {\cal D}_d}{\E} \left[ |\phi'( w_i^{*\top} x)| L_2|( w_l- w_l^*)^\top x|   (x^\top a)^2 \right] \right) \notag \\
\leq & ~ v_{\max}^2 \max_{\| a\|=1}\left( \underset{x\sim {\cal D}_d}{\E} \left[   L_2 |( w_i- w_i^*)^\top x|  L_1  | w_l^\top x|^p  (x^\top a)^2 \right] + \underset{x\sim {\cal D}_d}{\E} \left[ L_1 |  w_i^{*\top} x|^p  L_2|( w_l- w_l^*)^\top x|   (x^\top a)^2 \right] \right) \notag \\
\leq & ~ v_{\max}^2 L_1 L_2 \max_{\| a\|=1}\left(  \underset{x\sim {\cal D}_d}{\E} \left[    |( w_i- w_i^*)^\top x|   | w_l^\top x|^p  (x^\top a)^2 \right]+  \underset{x\sim {\cal D}_d}{\E} \left[ |( w_l- w_l^*)^\top x| |  w_i^{*\top} x|^p     (x^\top a)^2 \right]  \right) \notag \\
\lesssim & ~ v_{\max}^2 L_1 L_2 \max_{\| a\|=1} ( \| w_i -w_i^* \|  \| w_l \|^p \| a \|^2 +  \| w_l - w_l^* \|  \| w_i^* \|^p \| a\|^2) \notag \\
\lesssim & ~ v_{\max}^2 L_1 L_2 \sigma_1^p(W^*) \|W-W^*\|
\end{align}
where the first step follows by definition of $\Delta_{i,l}$, the second step follows by definition of spectral norm and $v_i^* v_l^* \leq v_{\max}^2$, the third step follows by triangle inequality, the fourth step follows by linearity of expectation, the fifth step follows by Property~\ref{pro:hessian}(a), i.e., $|\phi'(w_i^\top x) - \phi'(w_i^{*\top} x) | \leq L_2 | (w_i - w_i^*)^\top x|$, the sixth step follows by Property~\ref{pro:gradient}, i.e., $\phi'(z) \leq L_1 |z|^p$, the seventh step follows by Fact~\ref{fac:exp_gaussian_dot_three_vectors}, and the last step follows by $\|a\|^2=1$, $\|w_i -w_i^* \| \leq \| W - W^*\|$, $\|w_i\| \leq \frac{3}{2} \|w_i^*\|$, and $\| w_i^*\| \leq \sigma_1(W^*)$.

  Note that the proof for the off-diagonal terms also applies to bounding the second-term in the diagonal block Eq.~\eqref{eq:diag_population}. 
Thus we only need to show how to bound the first term in the diagonal block Eq.~\eqref{eq:diag_population}. 
\begin{align}\label{eq:smooth_pop_local_diagonal}
& ~ \left\| \underset{x \sim {\cal D}_d }{\E} \left[ \left( \sum_{r=1}^k v_r^* \phi( w_{r}^\top x)  - y \right) v_i^* \phi''( w_i^\top x) x x^\top \right] \right\| \notag \\
= & ~ \left\| \underset{x \sim {\cal D}_d }{\E} \left[ \left( \sum_{r=1}^k v_r^* ( \phi( w_{r}^\top x)  - \phi(w_r^{*\top} x) ) \right) v_i^* \phi''( w_i^\top x) x x^\top \right] \right\| \notag \\
\leq &  ~v_{\max}^{2} \sum_{r=1}^k \max_{\| a\|=1} \underset{x \sim {\cal D}_d }{\E} [ |\phi( w_{r}^\top x)-\phi( w_{r}^{*\top} x)|  |\phi''(w_i^\top x) | (x^\top  a)^2] \notag \\
\leq &  ~v_{\max}^{2} \sum_{r=1}^k \max_{\| a\|=1} \underset{x \sim {\cal D}_d }{\E} [ |\phi( w_{r}^\top x)-\phi( w_{r}^{*\top} x)| L_2 (x^\top  a)^2] \notag \\
\leq & ~ v_{\max}^{2} L_2 \sum_{r=1}^k \max_{\| a\|=1} \underset{x \sim {\cal D}_d }{\E} \left[  \max_{z\in[ w_r^\top x, w_r^{*\top}x]}|\phi'(z)| \cdot |( w_r- w_r^*)^{\top} x| \cdot  (x^\top  a)^2 \right] \notag \\
\leq & ~ v_{\max}^{2} L_2 \sum_{r=1}^k  \max_{\| a\|=1} \underset{x \sim {\cal D}_d }{\E} \left[  \max_{z\in[ w_r^\top x, w_r^{*\top}x]}L_1 |z|^p  \cdot |( w_r- w_r^*)^{\top} x| \cdot  (x^\top a)^2 \right] \notag \\
\leq & ~ v_{\max}^{2} L_1 L_2 \sum_{r=1}^k  \max_{\| a\|=1} \underset{x \sim {\cal D}_d }{\E} [  ( | w_r^\top x|^p+| w_r^{*\top}x|^p ) \cdot    |( w_r- w_r^*)^{\top} x| \cdot  (x^\top a)^2] \notag \\
\lesssim & ~ v_{\max}^{2} L_1L_2 \sum_{r=1}^k  [ (\| w_r\|^p+\| w_r^{*}\|^p)  \| w_r- w_r^*\|] \notag \\
\lesssim & ~ k v_{\max}^{2}L_1L_2 \sigma_1^p(W^*) \|W-W^*\|,
\end{align}
where the first step follows by $y=\sum_{r=1}^k v_r^* \phi(w_r^{*\top} x)$, the second step follows by definition of spectral norm and $v_r^* v_i^* \leq |v_{\max}|^2$, the third step follows by Property~\ref{pro:hessian}(a), i.e., $|\phi''(w_i^\top x)|\leq L_2$, the fourth step follows by $|\phi(w_r^\top x) - \phi(w_r^{*\top}x) \leq \max_{z\in [ w_r^\top x, w_r^{*\top } x]} |\phi'(z)| | (w_r-w_r^{*} )^\top x |$, the fifth step follows Property~\ref{pro:gradient}, i.e., $|\phi'(z)|\leq L_1 |z|^p$, the sixth step follows by $ \max_{z\in [ w_r^\top x, w_r^{*\top } x]} |z|^p \leq ( |w_r^\top x|^p + |w_r^{*\top} x|^p )$, the seventh step follows by Fact~\ref{fac:exp_gaussian_dot_three_vectors}.

Putting it all together, we can bound the error by
\begin{align*}
& ~ \|\nabla^2 f_{\cal D}(W) - \nabla^2 f_{\cal D}(W^*)\| \\
= & ~ \max_{\| a\|=1} a^\top ( \nabla^2 f_{\cal D}(W) - \nabla^2 f_{\cal D}(W^*) ) a \\
= & ~ \max_{ \| a\|=1} \sum_{i=1}^k \sum_{l=1}^k  a_i^\top \Delta_{i,l}  a_l  \\
= & ~ \max_{ \| a\|=1} \left( \sum_{i=1}^k   a_i^\top \Delta_{i,i}  a_i  +  \sum_{i \neq l}  a_i^\top \Delta_{i,l}  a_l \right)\\
 \leq & ~ \max_{ \| a\|=1} \left( \sum_{i=1}^k \|\Delta_{i,i}\| \| a_i\|^2  + \sum_{i\neq l} \|\Delta_{i,l}\| \| a_i\|  \| a_l\| \right) \\
 \leq & \max_{\| a \|=1} \left( \sum_{i=1}^k  C_1\| a_i\|^2  + \sum_{i\neq l}  C_2 \| a_i\|  \| a_l\| \right) \\
 =   & \max_{\| a \|=1} \left( C_1 \sum_{i=1}^k \| a_i\|^2  + C_2 \left( \left( \sum_{i=1}^k \| a_i\| \right)^2 -  \sum_{i=1}^k \| a_i\|^2 \right)  \right) \\
  \leq   & \max_{\| a \|=1} \left( C_1 \sum_{i=1}^k \| a_i\|^2  + C_2 \left( k \sum_{i=1}^k \| a_i\|^2 -  \sum_{i=1}^k \| a_i\|^2 \right)  \right) \\
  = & \max_{\| a \|=1} ( C_1 + C_2 (k-1)) \\
 \lesssim & ~ k^2 v_{\max}^{2} L_1 L_2 \sigma_1^p(W^*) \|W-W^*\|.
\end{align*}
where the first step follows by definition of spectral norm and $a$ denotes a vector $\in \mathbb{R}^{dk}$, the first inequality follows by $\| A\| = \max_{\|x\| \neq 0, \| y \| \neq 0} \frac{ x^\top A y }{ \| x\| \|y \| }$, the second inequality follows by $\| \Delta_{i,i}\| \leq C_1$ and $\| \Delta_{i,l} \| \leq C_2$, the third inequality follows by Cauchy-Scharwz inequality, the eighth step follows by $\sum_{i=1} \| a_i \|^2 =1$, and the last step follows by Eq~\eqref{eq:smooth_pop_local_off_diagonal} and \eqref{eq:smooth_pop_local_diagonal}.
\end{proof}

\subsubsection{Empirical and Population Difference for Smooth Activations}\label{proof:lemma:emp_pop}

The goal of this Section is to prove Lemma~\ref{lem:emp_pop}. For each $i\in [k]$, let $\sigma_i$ denote the $i$-th largest singular value of $W^*\in \mathbb{R}^{d\times k}$.

Note that Bernstein inequality requires the spectral norm of each random matrix to be bounded almost surely. 
However, since we assume Gaussian distribution for $x$, $\|x\|^2$ is not bounded almost surely. The main idea is to do truncation and then use Matrix Bernstein inequality. Details can be found in Lemma~\ref{lem:modified_bernstein_non_zero} and Corollary~\ref{cor:modified_bernstein_tail_xx}.

\begin{lemma}[Empirical and population difference for smooth activations]\label{lem:emp_pop}
Let $\phi(z)$ satisfy Property~\ref{pro:gradient},\ref{pro:expect} and \ref{pro:hessian}(a).
% is satisfied because of its smoothness.
Let $W$ satisfy $\|W-W^*\|\leq \sigma_k/2$. Let $S$ denote a set of i.i.d. samples from distribution ${\cal D}$ (defined in~(\ref{eq:model})). Then for any $t\geq 1$ and $0<\epsilon<1/2$,  if 
\begin{align*}
|S| \geq \epsilon^{-2} d \kappa^2 \tau \cdot \poly(\log d, t)
\end{align*}
%$$|S| \gtrsim  \epsilon^{-2}  \kappa^2 \tau t^{p+2}d   \log^{p+3} d$$
then we have, with probability at least $1-d^{-\Omega(t)}$, {\small
\begin{align*} %%% kappa is actually kappa(W^*)
 \| \nabla^2\widehat{f}_S(W) - \nabla^2 f_{\cal D}(W) \| \lesssim & v^{2}_{\max} k \sigma_1^{p} (\epsilon \sigma_1^{p}+ \|W-W^*\|).
\end{align*} }
\end{lemma}

\begin{proof}
Define $\Delta = \nabla^2 f_{\cal D}(W) -  \nabla^2\widehat{f}_S(W)$. Let's first consider the diagonal blocks. 
Define 
\begin{align*}
 \Delta_{i,i} = & \underset{x\sim {\cal D}_d}{\E} \left[ \left( \overset{k}{ \underset{r=1}{\sum} } v_r^* \phi( w_r^\top x)  - y \right) v_i^* \phi''( w_i^\top x) x x^\top + v^{*2}_i \phi'^2( w_i^\top x)  x x^\top \right] \\
& ~ - \left( \frac{1}{n}\sum_{j=1}^n \left( \overset{k}{ \underset{r=1}{\sum} } v_r^* \phi( w_r^\top x_j)  - y \right) v_i^* \phi''( w_i^\top x_j) x_j x_j^\top  + v^{*2}_i \phi'^2( w_i^\top x_j)  x_j x_j^\top \right).
\end{align*}

Let's further decompose $\Delta_{i,i}$ into $\Delta_{i,i} = \Delta_{i,i}^{(1)}+\Delta_{i,i}^{(2)}$, where

\begin{align*}
 & ~\Delta_{i,i}^{(1)} \\
 = &~ \underset{x\sim {\cal D}_d}{\E} \left[ \left( \overset{k}{ \underset{r=1}{\sum} } v_r^* \phi( w_{r}^\top x)  - y \right) v_i^* \phi''( w_i^\top x) x x^\top \right]  - \frac{1}{n}\sum_{j=1}^n \left( \overset{k}{ \underset{r=1}{\sum} } v_r^* \phi( w_{r}^\top x_j)  - y \right) v_i^* \phi''( w_i^\top x_j) x_j x_j^\top \\
=  & ~   v_i^* \sum_{r=1}^k \left( v_r^*  \underset{x\sim {\cal D}_d}{\E} \left[( \phi( w_{r}^\top x)  - \phi( w_r^{*\top} x)) \phi''( w_i^\top x) x x^\top \right] \right.  \\
- &  \frac{1}{n}\sum_{j=1}^n \left. (  \phi( w_{r}^\top x_j)  -  \phi( w_r^{*\top} x_j))  \phi''( w_i^\top x_j) x_j x_j^\top \right),
\end{align*}
and
\begin{align}\label{eq:delta_ii_2}
 \Delta_{i,i}^{(2)} =  \underset{x\sim {\cal D}_d }{\E}[ v^{*2}_i \phi'^2 ( w_i^\top x) x x^\top ] - \frac{1}{n}\sum_{j=1}^n [v^{*2}_i \phi'^2 ( w_i^\top x_j) x_j x_j^\top ].
\end{align}

The off-diagonal block is
\begin{align*}
 \Delta_{i,l} = v^{*}_iv^{*}_l \left( \underset{x\sim \D_d}{\E} [  \phi'( w_i^\top x) \phi'( w_l^\top x) x x^\top ]  - \frac{1}{n}\sum_{j=1}^n  \phi'( w_i^\top x_j)\phi'( w_l^\top x_j) x_j x_j^\top  \right)
\end{align*}

Combining Claims.~\ref{cla:bound_emp_diag_1}, \ref{cla:bound_emp_diag_2} and \ref{cla:bound_emp_off_diag}, and taking a union bound over $k^2$ different $\Delta_{i,j}$, we obtain if 
$n\geq \epsilon^{-2} \kappa^2(W^*) \tau d   \cdot \poly(t, \log d) $
for any $\epsilon \in (0,1/2)$, with probability at least $1-1/d^{t}$,
\begin{align*}
& \| \nabla^2\widehat{f}_S(W) - \nabla^2 f(W) \| 
\lesssim  v^{2}_{\max} k \sigma_1^{p}(W^*)(\epsilon \sigma_1^{p}(W^*) + \|W-W^*\|) 
\end{align*}

\end{proof}

\begin{claim}\label{cla:bound_emp_diag_1}
For each $i\in [k]$, if $n \geq d \poly(\log d, t)$
 \begin{align*}
 \| \Delta_{i,i}^{(1)} \| \lesssim k v_{\max}^2 \sigma_1^p(W^*) \| W - W^*\|
 \end{align*}
 holds with probability $1-1/d^{4t}$.
\end{claim}
\begin{proof}
 %Let's first consider $\Delta_{i,i}^{(1)}$. We don't try to bound the difference between the empirical diagonals and the population diagonals. Instead we try to bound the magnitude of both since $W$ and $W^*$ are close to each other. 

 For each $r\in [k]$, we define function $\wh{B}_r : \mathbb{R}^{d} \rightarrow \mathbb{R}^{d\times d}$,
 \begin{align*}
  \wh{B}_{r}(x) = L_1L_2 \cdot (|  w_r^\top x|^p+|  w_r^{*\top} x|^p) \cdot |( w_r - w_r^*)^\top x| \cdot x x^\top.
 \end{align*}
According to Properties~\ref{pro:gradient},\ref{pro:expect} and \ref{pro:hessian}(a), we have for each $ x\in S$,
$$-\wh{B}_{r}(x) \preceq (  \phi( w_r^\top x)  -  \phi( w_r^{*\top} x))  \phi''( w_i^\top x) x x^\top \preceq \wh{B}_{r}(x) $$
Therefore, 
\begin{align*}
\Delta_{i,i}^{(1)} \preceq v_{\max}^{2}  \sum_{r=1}^k \left(\underset{x\sim {\cal D}_d}{\E}[ \wh{B}_{r}(x)] + \frac{1}{|S|}\sum_{x\in S} \wh{B}_{r}(x) \right).
\end{align*}

Let $h_r(x) = L_1L_2 |  w_{r}^\top x|^p \cdot |( w_r - w_r^*)^\top x| $. Let $\ov{B}_r  = \E_{x\sim {\cal D}_d } [ h_r(x)x x^\top ] $. Define function $B_{r}(x)=  h_r(x)x x^\top$. 

(\RN{1}) Bounding $|h_r(x)|$.

 According to Fact~\ref{fac:inner_prod_bound}, we have for any constant $t \geq1$,
with probability $1-1/( nd^{4t})$,
\begin{align*}
|h_r(x)| =  L_1L_2 |  w_{r}^\top x|^p |( w_r - w_r^*)^\top x| 
\lesssim  L_1L_2  \| w_r\|^p\| w_r- w_r^*\|  (t \log  n)^{(p+1)/2}.
\end{align*}

(\RN{2}) Bounding $\| \ov{B}_r \|$.

 \begin{align*}
\|\ov{B}_r\|  & \geq \underset{x\sim {\cal D}_d}{\E} \left[L_1L_2 |  w_r^\top x|^p |( w_r - w_r^*)^\top x| \left(\frac{( w_r- w_r^*)^\top x}{\| w_r- w_r^*\|} \right)^2 \right]  \gtrsim L_1L_2 \| w_r\|^{p}\| w_r -  w_r^*\|,
\end{align*}
where the first step follows by definition of spectral norm, and last step follows by Fact~\ref{fac:exp_gaussian_dot_three_vectors}. Using Fact~\ref{fac:exp_gaussian_dot_three_vectors}, we can also prove an upper bound $\| \ov{B}_r \|$, $\| \ov{B}_r \| \lesssim L_1L_2  \| w_r\|^{p}\| w_r -  w_r^*\|$.

(\RN{3}) Bounding $(\E_{x\sim \D_d}[h^4(x)])^{1/4}$

Using Fact~\ref{fac:exp_gaussian_dot_three_vectors}, we have
\begin{align*}
\left( \underset{x\sim{\cal D}_d}{\E} [  h^4(x) ] \right)^{1/4} = L_1L_2 \left( \underset{x\sim{\cal D}_d}{\E} \left[ \left(|  w_r^\top x|^p |( w_r - w_r^*)^\top x| \right)^4 \right] \right)^{1/4} 
\lesssim L_1 L_2 \| w_r\|^{p}\| w_r- w_r^*\|.
\end{align*}

By applying Corollary~\ref{cor:modified_bernstein_tail_xx}, if $ n \geq \epsilon^{-2} d  \poly( \log d, t)$, then with probability $1-1/d^{4t}$,

\begin{align}\label{eq:abs_g}
& ~ \left\| \underset{x\sim{\cal D}_d}{\E} \left[|  w_{r}^\top x|^p \cdot |( w_r - w_r^*)^\top x| \cdot x x^\top \right]  -\frac{1}{|S|} \left( \sum_{x\in S} |  w_{r}^\top x_j|^p \cdot |( w_r - w_r^*)^\top x| \cdot xx^\top \right)  \right\| \notag \\
= & ~ \left\| \ov{B}_r - \frac{1}{|S|} \sum_{x\in S} B_{r} (x) \right\| \notag \\
 \leq & ~ \epsilon\| \ov{B}_r\| \notag \\
 \lesssim & ~  \epsilon \| w_r\|^{p}\| w_r -  w_r^*\|.
\end{align}

If $\epsilon \leq 1/2$, we have
\begin{align*}
\|\Delta_{i,i}^{(1)} \| & \lesssim \sum_{i=1}^k v_{\max}^2 \| \ov{B}_r \| \lesssim k v_{\max}^2 \sigma_1^p(W^*) \|W-W^*\| 
\end{align*}
\end{proof}

\begin{claim}\label{cla:bound_emp_diag_2}
For each $i\in [k]$, if $n\geq \epsilon^{-2} d \tau \poly(\log d, t)$ , then 
\begin{align*}
\| \Delta_{i,i}^{(2)}\| \lesssim \epsilon v_{\max}^2 \sigma_1^{2p}
\end{align*}
holds with probability $1-1/d^{4t}$.
\end{claim}
\begin{proof}
Recall the definition of $\Delta_{i,i}^{(2)}$.
\begin{align*}
 \Delta_{i,i}^{(2)} = \underset{x\sim {\cal D}_d}{\E} [ v^{*2}_i \phi'^2( w_i^\top x) x x^\top ]  - \frac{1}{|S|}\sum_{x\in S} [v^{*2}_i \phi'^2( w_i^\top x) x x^\top ] 
\end{align*}
Let $h_i(x) = \phi'^2 ( w_i^\top x)$. Let $\ov{B}_i = \E_{x\sim {\cal D}_d} [ h_i(x) x x^\top ]$ Define function $B_i(x)  = h_i(x)xx^\top$. 

(\RN{1}) Bounding $|h_i(x)|$.

 For any constant $t \geq 1$, $(\phi'( w_i^\top x) )^2 \leq L_1^2 | w_i^\top x|^{2p} \lesssim L_1^2 \| w_i\|^{2p} t^p \log^p (n)$ with probability $1-1/(nd^{4t})$ according to Fact~\ref{fac:inner_prod_bound}. 

(\RN{2}) Bounding $\| \ov{B}_i \|$. 
\begin{align*}
 \left \| \underset{x\sim {\cal D}_d}{\E} [ \phi'^{2}( w_i^\top x)  x x^\top ] \right\| 
= & ~ \max_{\| a\| = 1}  \underset{x\sim {\cal D}_d}{\E} \left[ \phi'^{2}( w_i^\top x) (x^\top a)^2 \right] \\
= & ~ \max_{\| a\| = 1}  \underset{x\sim {\cal D}_d}{\E} \left[ \phi'^{2}( w_i^\top x)  \left(\alpha \ov{ w}_i^\top x  + \beta x^\top v \right)^2 \right] \\
= & ~ \max_{\alpha^2 +\beta^2= 1,\| v\|=1}  \underset{x\sim {\cal D}_d}{\E} \left[  \phi'^{2}( w_i^\top x) \left((\alpha \ov{ w}_i^\top x )^2 + (\beta x^\top v)^2 \right) \right] \\
= & ~ \max_{\alpha^2 +\beta^2= 1}  \left(\alpha^2 \underset{z\sim {\cal D}_1}{\E} [  \phi'^{2} (\| w_i\| z)  z^2] + \beta^2 \underset{z\sim {\cal D}_1}{\E} [ \phi'^{2}(\| w_i\| z) ] \right)\\
= & ~ \max \left( \underset{x\sim {\cal D}_1}{\E} [  \phi'^{2}(\| w_i\| z)  z^2],  \underset{x\sim {\cal D}_1}{\E} [ \phi'^{2}(\| w_i\| z)  ] \right)
\end{align*}
where $\ov{ w}_i =  w_i/\| w_i\|$ and $v$ is a unit vector orthogonal to $ w_i$ such that $ a = \alpha \ov{ w}_i + \beta v$. Now from Property~\ref{pro:expect}, we have
\begin{align*}
  \rho(\| w_i\|)\leq \left\| \underset{x \sim {\cal D}_d}{\E}  [ \phi'^{2}( w_i^\top x) x x^\top ] \right\| \lesssim  L_1^2 \| w_i\|^{2p}.
\end{align*}

(\RN{3}) Bounding $(\E_{x\sim \D_d} [h_i^4(x)] )^{1/4}$.
\begin{align*}
\left( \E_{x\sim \D_d} [h_i^4(x)] \right)^{1/4} = \left( \underset{x \sim {\cal D}_d}{\E} [ \phi'^8( w_i^\top x) ] \right)^{1/4} \lesssim L_1^2 \| w_i\|^{2p}.
\end{align*}

By applying Corollary ~\ref{cor:modified_bernstein_tail_xx}, we have, for any $0<\epsilon <1$, if $ n\geq 
\epsilon^{-2} d   \frac{\| w_i\|^{4p}}{\rho^2(\| w_i\|)} \poly(\log d, t)
$ 
the following bound holds 
\begin{align*}
\left\| \ov{B}_i - \frac{1}{|S|} \sum_{x\in S} B_i(x) \right\| \leq \epsilon \| \ov{B}_i \|,
\end{align*}
 with probability at least $1-1/d^{4t}$.

Therefore, if 
$n\geq \epsilon^{-2} d \tau \poly(\log d, t)$,
where $\tau=\frac{(3\sigma_1/2)^{4p}}{\min_{\sigma\in [\sigma_k/2,3\sigma_1/2 ]} \rho^2(\sigma)} $, we have with probability $1-1/d^{4t}$
\begin{align*}%\label{eq:bound_emp_diag_2}
 \|\Delta_{i,i}^{(2)}\| \lesssim \epsilon v^{2}_{\max}  \sigma_1^{2p}
\end{align*}
\end{proof}

\begin{claim}\label{cla:bound_emp_off_diag}
For each $i\in [k], j\in [k], i \neq j$, if $n \geq \epsilon^{-2} \kappa^2 \tau d \poly(\log d,t)$, then
\begin{align*}
\| \Delta_{i,j} \| \leq \epsilon v_{\max}^2 \sigma_1^{2p}(W^*)
\end{align*}
holds with probability $1-1/d^{4t}$.
\end{claim}
\begin{proof}
Recall the definition of off-diagonal blocks $\Delta_{i,l}$,
\begin{align}\label{eq:delta_il}
 \Delta_{i,l} = v^{*}_iv^{*}_l \left( \underset{x\sim \D_d}{\E} [  \phi'( w_i^\top x) \phi'( w_l^\top x) x x^\top ]  - \frac{1}{|S|}\sum_{x\in S}  \phi'( w_i^\top x)\phi'( w_l^\top x) x x^\top  \right)
\end{align}

Let $h_{i,l}(x) = \phi'( w_i^\top x) \phi'( w_l^\top x)$. Define function $B_{i,l}(x) = h_{i,l}(x) xx^\top$. Let $\ov{B}_{i,l}  = \E_{x\sim \D_d}[ h_{i,l}(x)xx^\top ]$. 

(\RN{1}) Bounding $|h_{i,l}(x)|$.

 For any constant $t \geq 1$, we have with probability $1-1/(nd^{4t})$
\begin{align*}
 |h_{i,l}(x)| = & ~|\phi'( w_i^\top x) \phi'( w_l^\top x) | \\
 \leq & ~  L_1^2 \| w_i^\top x\|^{p}\| w_l^\top x\|^{p} \\
 \leq & ~ L_1^2 \| w_i\|^{p}\| w_l\|^{p}  (t \log n)^p \\
  \lesssim & ~ L_1^2 \sigma_1^{2p} ( t \log n )^p
\end{align*}
 where the third step follows by Fact~\ref{fac:inner_prod_bound}. 

(\RN{2}) Bounding $\|\ov{B}_{i,l}\|$.

Let $U \in \mathbb{R}^{d \times 2}$ be the orthogonal basis of $\text{span}\{ w_i, w_l\}$ and $U_\perp \in \R^{d\times (d-2)}$ be the complementary matrix of $U$. Let matrix  $V \in \mathbb{R}^{2\times 2}$ denote $U^\top[ w_i\; w_l]$, then $UV = [ w_i\; w_l] \in \mathbb{R}^{d\times 2}$. Given any vector $ a\in \mathbb{R}^d$, there exist vectors $b\in \mathbb{R}^2$ and $c\in \mathbb{R}^{d-2}$ such that $ a = U b+U_{\perp} c$. We can simplify $\|\ov{B}_{i,l}\|$ in the following way,
\begin{align*}
\|\ov{B}_{i,l}\| = & ~ \left\| \underset{x\sim \D_d}{\E} [\phi'( w_i^\top x) \phi'( w_l^\top x) x x^\top ]\right\| \\
= & ~ \max_{\| a\| = 1}  \underset{x\sim \D_d}{\E} [ \phi'( w_i^\top x) \phi'( w_l^\top x)  (x^\top a)^2 ] \\
= & ~ \max_{\| b\|^2+\| c\|^2 = 1}   \underset{x\sim \D_d}{\E} [  \phi'( w_i^\top x) \phi'( w_l^\top x)  ( b^\top U^\top x +  c^\top U_{\perp}^\top x)^2 ] \\
= & ~ \max_{\| b\|^2+\| c\|^2 = 1}   \underset{x\sim \D_d}{\E} [  \phi'( w_i^\top x) \phi'( w_l^\top x)  (( b^\top U^\top x)^2 + ( c^\top U_{\perp}^\top x)^2 )] \\
= & ~ \max_{\| b\|^2+\| c\|^2 = 1} \left( \underbrace{ \underset{ z\sim \D_2}{\E}[  \phi'( v_1^\top  z) \phi'( v_2^\top  z)  ( b^\top  z)^2] }_{A_1} + \underbrace{ \underset{ z\sim \D_2, s\sim \D_{d-2}}{\E}[\phi'( v_1^\top  z) \phi'( v_2^\top  z)  ( c^\top  s)^2 ] }_{A_2} \right)
\end{align*}

We can lower bound the term $A_1$,
\begin{align*}
A_1 = & ~ \underset{z\sim \D_2}{\E} [ \phi'( v_1^\top  z) \phi'( v_2^\top  z) ( b^\top  z)^2] \\
= & ~ \int (2\pi)^{-1}    \phi'( v_1^\top  z) \phi'(v_2^\top  z)  ( b^\top  z)^2 e^{-\| z\|^2/2}d z \\
=& ~ \int (2\pi)^{-1}   \phi'(s_1) \phi'(s_2)  ( b^\top V^{\dagger\top}  s)^2 e^{-\|V^{\dagger\top}  s\|^2/2} \cdot |\det(V^\dagger)| d s  \\
\geq & ~ \int (2\pi)^{-1}   (\phi'(s_1) \phi'(s_2) ) ( b^\top V^{\dagger\top}  s)^2 e^{-\sigma^2_1(V^{\dagger}) \| s\|^2/2}  \cdot |\det(V^\dagger)|  d  s  \\
= & ~ \int (2\pi)^{-1}  (\phi'(u_1/\sigma_1(V^{\dagger})) \phi'(u_2/\sigma_1(V^{\dagger})) ) \cdot ( b^\top V^{\dagger\top} u/\sigma_1(V^{\dagger}))^2 e^{- \|u\|^2/2} |\det(V^\dagger)|/\sigma_1^2(V^{\dagger})  d u  \\
= & ~ \frac{\sigma_2(V)}{\sigma_1(V)} \underset{u\sim \D_2}{\E}  \left[(p^\top u)^2 \phi'(\sigma_2(V) \cdot u_1)\phi'(\sigma_2(V) \cdot u_2) \right]  \\
= & ~ \frac{\sigma_2(V)}{\sigma_1(V)} \underset{u\sim \D_2}{\E}  \left[ \left((p_1u_1)^2+(p_2u_2)^2+2p_1p_2u_1u_2 \right) \phi'(\sigma_2(V) \cdot u_1)\phi'(\sigma_2(V) \cdot u_2) \right] \\
= & ~ \frac{\sigma_2(V)}{\sigma_1(V)} \left(\|p\|^2 \underset{z\sim \D_1} {\E} [\phi'(\sigma_2(V) \cdot z) z^2]\cdot \underset{z\sim \D_1}{\E} [\phi'(\sigma_2(V) \cdot z)] \right.\\ 
+ & ~ \left. (( p^\top\bone)^2 -\|p\|^2) \underset{z\sim \D_1}{\E} [\phi'(\sigma_2(V) \cdot z) z]^2\right)
\end{align*}
where  $p =  V^{\dagger} b\cdot \sigma_2(V) \in \mathbb{R}^2$. 

Since we are maximizing over $b\in \mathbb{R}^2$, we can choose $ b$ such that $\|p\| = \| b\|$. Then
\begin{align*}
 A_1 = & ~ \underset{z\sim \D_2}{\E}[ \phi'( v_1^\top  z) \phi'( v_2^\top  z)  ( b^\top  z)^2] \\
 \geq & ~ \frac{\sigma_2(V)}{\sigma_1(V)}  \| b\|^2 \left( \underset{z\sim \D_1}{\E} [\phi'(\sigma_2(V) \cdot z) z^2]\cdot \underset{z\sim \D_1}{\E} [\phi'(\sigma_2(V) \cdot z)] - \underset{z\sim \D_1}{\E} [\phi'(\sigma_2(V) \cdot z) z]^2 \right)
\end{align*}
For the term $A_2$, similarly we have 
\begin{align*}
A_2 = & ~ \underset{z\sim \D_2, s \sim \D_{d-2} }{\E} [\phi'( v_1^\top  z) \phi'( v_2^\top  z)  ( c^\top  s)^2 ] \\
= & ~\underset{z\sim \D_2 }{\E} [\phi'( v_1^\top  z) \phi'( v_2^\top  z) ] \underset{  s \sim \D_{d-2} }{\E}[ ( c^\top  s)^2 ] \\
= & ~\| c \|^2 \underset{z\sim \D_2 }{\E} [\phi'( v_1^\top  z) \phi'( v_2^\top  z) ] \\
\geq & ~ \| c\|^2 \frac{\sigma_2(V)}{\sigma_1(V)} \left( \underset{z\sim \D_1}{\E} [\phi'(\sigma_2(V) \cdot z)] \right)^2
\end{align*}

For simplicity, we just set $\| b\|=1$ and $\| c\|=0$ to lower bound,
\begin{align*}
& ~ \left\| \underset{x\sim \D_d}{\E} \left[\phi'( w_i^\top x) \phi'( w_l^\top x) x x^\top \right] \right\| \\
\geq & ~ \frac{\sigma_2(V)}{\sigma_1(V)} \left( \underset{z\sim \D_1}{\E} [\phi'(\sigma_2(V) \cdot z) z^2]\cdot \underset{z\sim \D_1}{\E} [\phi'(\sigma_2(V) \cdot z)] - \left( \underset{z\sim \D_1}{\E} [\phi'(\sigma_2(V) \cdot z) z] \right)^2 \right) \\
\geq & ~ \frac{\sigma_2(V)}{\sigma_1(V)} \rho( \sigma_2(V))\\
\geq & ~ \frac{1}{\kappa(W^*)}\rho(\sigma_2(V))
\end{align*}
where the second step is from Property~\ref{pro:expect} and the fact that $\sigma_k/2 \leq \sigma_2(V)\leq \sigma_1(V) \leq 3\sigma_1/2 $.

For the upper bound of $\| \E_{x\sim \D_d}[\phi'( w_i^\top x) \phi'( w_l^\top x) x x^\top ]\|$,
we have
\begin{align*}
 \underset{z \sim \D_2 }{\E}[ \phi'( v_1^\top  z) \phi'( v_2^\top  z)  ( b^\top  z)^2] 
  \leq & ~ L_1^2 \underset{z \sim \D_2 }{\E}[ | v_1^\top  z |^p \cdot | v_2^\top  z |^p \cdot | b^\top  z|^2] 
  \\
\lesssim & ~ L_1^2 \| v_1\|^p \| v_2\|^p \| b\|^2 \\
\lesssim & ~ L_1^2 \sigma_1^{2p}
\end{align*}
where the first step follows by Property~\ref{pro:gradient}, the second step follows by Fact~\ref{fac:exp_gaussian_dot_three_vectors}, and the last step follows by $\| v_1 \| \leq \sigma_1$, $\| v_2\| \leq \sigma_1$ and $\| b \| \leq 1$.
Similarly, we can upper bound,
\begin{align*}
\underset{z \sim \D_2,  s\sim \D_{d-2} }{\E} [\phi'( v_1^\top  z) \phi'( v_2^\top  z)  ( c^\top  s)^2 ] = \| c\|^2 \underset{z \sim \D_2 }{\E} [\phi'( v_1^\top  z) \phi'( v_2^\top  z) ]  \lesssim  L_1^2 \sigma_1^{2p}
\end{align*}

Thus, we have
\begin{align*}
\left\| \E_{x\sim \D_d}[\phi'( w_i^\top x) \phi'( w_l^\top x) x x^\top ] \right\| \lesssim L_1^2 \sigma_1^{2p} \lesssim \sigma_1^{2p}. 
\end{align*}

(\RN{3}) Bounding $(\E_{x\sim \D_d}[ h_{i,l}^4(x)] )^{1/4}$.

\begin{align*}
\left(\E_{x\sim \D_d}[ h_{i,l}^4(x)] \right)^{1/4} = \left( \E_{x\sim \D_d}  [ \phi'^4( w_i^\top x) \cdot \phi'^4 ( w_l^\top x) ] \right)^{1/4} \lesssim L_1^2 \| w_i\|^{p}\| w_l\|^{p} \lesssim L_1^2 \sigma_1^{2p} .
\end{align*}

Therefore, applying Corollary~\ref{cor:modified_bernstein_tail_xx}, we have, if 
$n\geq \epsilon^{-2} \kappa^2(W^*) \tau d \poly(\log d, t)$, then
\begin{align*}
  \|\Delta_{i,j}\|  ~ \leq  ~ \epsilon v^{2}_{\max}  \sigma_1^{2p}(W^*).
\end{align*}
holds with probability at least $1-1/d^{4t}$.
\end{proof}

\subsection{Error Bound of Hessians near the Ground Truth for Non-smooth Activations}\label{proof:lemma:emp_pop_nonsmooth}

The goal of this Section is to prove Lemma~\ref{lem:emp_pop_nonsmooth},

\begin{lemma}[Error bound of Hessians near the ground truth for non-smooth activations]\label{lem:emp_pop_nonsmooth}
Let $\phi(z)$ satisfy Property~\ref{pro:gradient},\ref{pro:expect} and \ref{pro:hessian}(b). %is satisfied because $\phi''(z) = 0$ except for $e$ points. 
Let $W$ satisfy $\|W-W^*\|\leq \sigma_k/2$. Let $S$ denote a set of
i.i.d. samples from the distribution defined in~(\ref{eq:model}). Then for any $t \geq 1$ and $0<\epsilon<1/2$, if 
\begin{align*}
 |S| \geq \epsilon^{-2} \kappa^2 \tau d  \poly( \log d, t) 
\end{align*}
with probability at least $1-d^{-\Omega(t)}$,
\begin{align*}
\| \nabla^2\widehat{f}_S(W) - \nabla^2 f_{\cal D}(W^*) \| \lesssim v^{2}_{\max} k \sigma_1^{2p}(\epsilon  +( \sigma_k^{-1}\cdot \|W-W^*\|)^{1/2 }).
\end{align*}
\end{lemma}

\begin{proof}
As we noted previously, when Property~\ref{pro:hessian}(b) holds, the diagonal blocks of the empirical Hessian can be written as, with probability $1$,
 \begin{align*}
\frac{\partial^2 \widehat{f}_S(W)}{\partial  w_i^2} = & ~\frac{1}{|S|}\sum_{(x,y) \in S} (v^*_i \phi'( w_i^\top x ))^2 x x^\top
\end{align*}
We construct a matrix $H\in \R^{dk\times dk}$ with $i,l$-th block as
\begin{align*}
H_{i,l} = v^*_iv^*_l  \E_{x\sim \D_d} \left[\phi'( w_i^\top x) \phi'( w_l^\top x)xx^\top \right] \quad \in \mathbb{R}^{d\times d}, \quad \forall i\in [k], l\in [k].
\end{align*}

Note that $H \neq \nabla^2 f_{\D}(W)$. However we can still bound $\|H - \nabla^2 \widehat{f}_S(W)\|$ and $\|H - \nabla^2 f_{\D}(W^*)\|$ when we have enough samples and $\|W-W^*\|$ is small enough. The proof for $\|H - \nabla^2 \widehat{f}_S(W)\|$ basically follows the proof of Lemma~\ref{lem:emp_pop} as $\Delta_{ii}^{(2)}$ in Eq.~\eqref{eq:delta_ii_2} and $\Delta_{il}$ in Eq.~\eqref{eq:delta_il} forms the blocks of $H - \nabla^2 \widehat{f}_S(W)$ and we can bound them without smoothness of $\phi(\cdot)$. 

Now we focus on $H-\nabla^2 f_{\D}(W^*)$. We again consider each block.
\begin{align*}
\Delta_{i,l} =  \E_{x\sim \D_d} \left[  v^*_i v^*_l (\phi'( w_i^\top x)  \phi'( w_l^\top x)-\phi'( w_i^{*\top} x)  \phi'( w_l^{*\top} x))  x x^\top \right] .
\end{align*}
We used the boundedness of $\phi''(z)$ to prove Lemma~\ref{lem:smooth_pop_local}. Here we can't use this condition. Without smoothness, we will stop at the following position. 
\begin{align}
& ~ \left\| \E_{x\sim \D_d}[  v^*_i v^*_l (\phi'( w_i^\top x)  \phi'( w_l^\top x) -\phi'( w_i^{*\top} x)  \phi'( w_l^{*\top} x))  x x^\top ] \right\|  \notag \\
\leq & ~|v^*_i v^*_l| \max_{\| a\|=1}  \E_{x\sim \D_d } \left[   |\phi'( w_i^\top x)  \phi'( w_l^\top x)  -\phi'( w_i^{*\top} x)  \phi'( w_l^{*\top} x)|   (x^\top a)^2 \right]   \notag  \\
\leq & ~ |v^*_i v^*_l | \max_{\| a\|=1} \E_{x\sim \D_d} \left[   |\phi'( w_i^\top x)-\phi'( w_i^{*\top} x)| |\phi'( w_l^\top x)| \right. \notag  \\
& ~ + \left. |\phi'( w_i^{*\top} x)| |\phi'( w_l^\top x) -  \phi'( w_l^{*\top} x)|   (x^\top a)^2 \right]    \notag \\
= & ~ |v^*_i v^*_l| \max_{\| a\|=1} \left( \E_{x\sim \D_d} \left[   |\phi'( w_i^\top x)-\phi'( w_i^{*\top} x)| |\phi'( w_l^\top x)|  (x^\top a)^2 \right] \right. \notag \\
& ~ + \left. \E_{x\sim \D_d} \left[ |\phi'( w_i^{*\top} x)| |\phi'( w_l^\top x) -  \phi'( w_l^{*\top} x)|   (x^\top a)^2 \right] \right) . \label{eq:decomp_exp_off_diag}
\end{align}
where the first step follows by definition of spectral norm, the second step follows by triangle inequality, and the last step follows by linearity of expectation.

Without loss of generality, we just bound the first term in the above formulation. Let $U$ be the orthogonal basis of $\text{span}( w_i, w_i^*, w_l)$. If $ w_i, w_i^*, w_l$ are independent, $U$ is $d$-by-$3$. Otherwise it can be $d$-by-$\rank(\text{span}( w_i, w_i^*, w_l))$. Without loss of generality, we assume $U = \text{span}( w_i, w_i^*, w_l)$ is $d$-by-3. Let $[ v_i\;  v_i^*\; v_l] = U^\top [ w_i \;  w_i^* \;  w_l] \in \mathbb{R}^{3 \times 3}$, and $[ u_i \; u_i^* \; u_l ]= U_{\bot}^\top [ w_i \;  w_i^* \;  w_l] \in \mathbb{R}^{(d-3) \times 3} $. Let $ a = U b+U_\perp  c$, where $U_\perp \in \R^{d\times (d-3)}$ is the complementary matrix of $U$. 
\begin{align}\label{eq:reduce_non_smooth}
& ~\E_{x\sim \D_d} \left[ |\phi'( w_i^\top x)-\phi'( w_i^{*\top} x)| |\phi'( w_l^\top x)|  (x^\top a)^2  \right] \notag \\
= & ~\E_{x\sim \D_d} \left[ |\phi'( w_i^\top x)-\phi'( w_i^{*\top} x)| |\phi'( w_l^\top x)|  (x^\top  (Ub + U_{\bot} c) )^2 \right] \notag \\
\lesssim & ~ \E_{x\sim \D_d} \left[   |\phi'( w_i^\top x)-\phi'( w_i^{*\top} x)| |\phi'( w_l^\top x)| \left(  (x^\top U  b)^2 +(x^\top U_\perp  c )^2 \right) \right] \notag \\
= & ~ \E_{x\sim \D_d} \left[   |\phi'( w_i^\top x)-\phi'( w_i^{*\top} x)| |\phi'( w_l^\top x)|    (x^\top U  b)^2  \right] \notag \\
+ & ~ \E_{x\sim \D_d} \left[   |\phi'( w_i^\top x)-\phi'( w_i^{*\top} x)| |\phi'( w_l^\top x)|   (x^\top U_\perp  c )^2  \right] \notag \\
= & ~ \E_{z\sim \D_3} \left[   |\phi'( v_i^\top z)-\phi'( v_i^{*\top} z)| |\phi'( v_l^\top z)|   (z^\top  b)^2  \right] \notag \\
+ & ~ \E_{y\sim \D_{d-3}} \left[   |\phi'( u_i^\top y)-\phi'( u_i^{*\top} y)| |\phi'( u_l^\top y)|   (y^\top  c )^2  \right]
\end{align}
where the first step follows by $a = Ub + U_{\bot} c$, the last step follows by $(a+b)^2 \leq 2a^2 + 2b^2$.

%Let's consider the first term. The second term is similar. 
%\begin{align}\label{eq:reduce_non_smooth}
%\E_{x\sim \D_d}[   |\phi'( w_i^\top x)-\phi'( w_i^{*\top} x)| |\phi'( w_l^\top x)|  (x^\top U  b)^2] = \E_{z \sim \D_3} [   |\phi'( v_i^\top  z)-\phi'( v_i^{*\top}  z)| |\phi'( v_l^\top  z)|  ( z^\top  b)^2] 
%\end{align}

By Property~\ref{pro:hessian}(b), we have $e$ exceptional points which have $\phi''(z) \neq 0$. Let these $e$ points be $p_1,p_2,\cdots,p_e$. Note that if $ v_i^\top  z$ and $ v_i^{*\top}  z$ are not separated by any of these exceptional points, i.e., there exists no $j\in[e]$ such that $ v_i^\top  z \leq p_j \leq  v_i^{*\top}  z$ or $ v_i^{*\top}  z \leq p_j \leq  v_i^\top  z $, then we have $\phi'( v_i^\top  z) = \phi'( v_i^{*\top}  z)$ since $\phi''(s)$ are zeros except for $\{p_j\}_{j=1,2,\cdots,e}$. So we consider the probability that $ v_i^\top  z, v_i^{*\top}  z$ are separated by any exception point. We use $\xi_j$ to denote the event that $ v_i^\top  z, v_i^{*\top}  z$ are separated by an exceptional point $p_j$. By union bound, $1- \sum_{j=1}^e\Pr [ \xi_j]$ is the probability that $ v_i^\top  z, v_i^{*\top}  z$ are not separated by any exceptional point. 
The first term of Equation~\eqref{eq:reduce_non_smooth} can be bounded as,
\begin{align*}
& ~ \E_{z\sim \D_3} \left[   |\phi'( v_i^\top  z)-\phi'( v_i^{*\top}  z)| |\phi'( v_l^\top  z)|  ( z^\top  b)^2 \right] \\
= & ~ \E_{z\sim \D_3} \left[\bone_{\cup_{j=1}^e \xi_j}|\phi'( v_i^\top  z) + \phi'( v_i^{*\top}  z)| |\phi'( v_l^\top  z)|  ( z^\top  b)^2 \right] \\
\leq & ~ \left( \E_{z\sim \D_3} \left[\bone_{\cup_{j=1}^e \xi_j} \right] \right)^{1/2 } \left(\E_{z\sim \D_3} \left[ (\phi'( v_i^\top  z) + \phi'( v_i^{*\top}  z))^2 \phi'( v_l^\top  z)^2  ( z^\top  b)^4 \right] \right)^{1/2} \\
\leq & ~ \left(\sum_{j=1}^e \Pr_{z\sim \D_3} [ \xi_j ] \right)^{1/2 } \left(\E_{z\sim \D_3} \left[(\phi'( v_i^\top  z)  + \phi'( v_i^{*\top}  z))^2 \phi'( v_l^\top  z)^2  ( z^\top  b)^4 \right] \right)^{1/2} \\
\lesssim & ~ \left(\sum_{j=1}^e \Pr_{z \sim \D_3}[\xi_j] \right)^{1/2 } (\| v_i\|^p + \| v_i^*\|^p)\| v_l\|^p\| b\|^2
\end{align*}
where the first step follows by if $ v_i^\top  z, v_i^{*\top}  z$ are not separated by any exceptional point then $\phi'( v_i^\top  z) = \phi'( v_i^{*\top} z)$ and the last step follows by H\"{o}lder's inequality and Property~\ref{pro:gradient}.

It remains to upper bound $\Pr_{z\sim \D_3}[\xi_j]$. First note that if $ v_i^\top  z, v_i^{*\top}  z$ are separated by an exceptional point, $p_j$, then $ | v_i^{*\top}  z - p_j| \leq | v_i^\top  z- v_i^{*\top}  z|  \leq  \| v_i- v_i^{*}\| \|  z\| $. Therefore, 
\begin{align*}
\Pr_{z\sim \D_3}[\xi_j] \leq \Pr_{z\sim \D_3} \left[ \frac{| v_i^\top  z-p_j|}{\| z\|} \leq \| v_i- v_i^*\| \right].
\end{align*}

 Note that $(\frac{ v_i^{*\top}  z}{\| z\| \| v_i^*\|}+1)/2$ follows Beta(1,1) distribution which is uniform distribution on $[0,1]$. 
\begin{align*}
& ~\Pr_{z\sim \D_3} \left[\frac{| v_i^{*\top}  z - p_j|}{\| z\|\| v_i^*\| }\leq \frac{\| v_i- v_i^*\|}{\| v_i^*\|} \right] 
\leq \Pr_{z\sim \D_3} \left[ \frac{| v_i^{*\top}  z|}{\| z\|\| v_i^*\| }\leq \frac{\| v_i- v_i^*\|}{\| v_i^*\|} \right]
\lesssim \frac{\| v_i- v_i^*\|}{\| v_i^*\|} 
\lesssim  \frac{ \|W-W^*\| }{\sigma_k(W^*)},
\end{align*}
where the first step is because we can view $\frac{ v_i^{*\top}  z}{\| z\|}$ and $\frac{p_j}{\| z\|}$ as two independent random variables: the former is about the direction of $ z$ and the later is related to the magnitude of $ z$. 
Thus, we have 
\begin{align}
\E_{z\in \D_3} [   |\phi'( v_i^\top  z)-\phi'( v_i^{*\top}  z)| |\phi'( v_l^\top  z)|  ( z^\top  b)^2] 
\lesssim  (e \|W-W^*\|/\sigma_k(W^*))^{1/2 } \sigma_1^{2p}(W^*) \| b\|^2 . \label{eq:decomp_off_first_part}
\end{align}

Similarly we have 
\begin{align} \label{eq:decomp_off_second_part}
\E_{y\in \D_{d-3}}[   |\phi'( u_i^\top y)-\phi'( u_i^{*\top} y)| |\phi'( u_l^\top y)| (y^\top   c )^2 ] 
\lesssim  (e \|W-W^*\|/\sigma_k(W^*))^{1/2 } \sigma_1^{2p}(W^*) \| c\|^2. 
\end{align}
Finally combining Eq.~\eqref{eq:decomp_exp_off_diag}, Eq.~\eqref{eq:decomp_off_first_part} and Eq.~\eqref{eq:decomp_off_second_part}, we have 
$$\| H - \nabla^2f_{\D}(W^*)\| \lesssim kv_{\max}^2(e \|W-W^*\|/\sigma_k(W^*))^{1/2 } \sigma_1^{2p}(W^*) ,$$ which completes the proof.

\end{proof}

%\subsubsection{Linear convergence}

\subsection{Positive Definiteness for a Small Region}
Here we introduce a lemma which shows that the Hessian of any $W$, which may be dependent on the samples but is very close to an anchor point, is close to the Hessian of this anchor point.

\begin{lemma}\label{lem:pd_near_anchors}
Let $S$ denote a set of samples from Distribution ${\cal D}$ defined in Eq.~\eqref{eq:model}. Let $W^a\in \R^{d\times k}$ be a point (respect to function $\wh{f}_S(W)$), which is independent of the samples $S$, satisfying $\|W^a - W^*\| \leq \sigma_k/2$. Assume $\phi$ satisfies Property~\ref{pro:gradient}, \ref{pro:expect} and \ref{pro:hessian}(a). % is satisfied because of the smoothness. 
 Then for any $t\geq 1$, if
\begin{align*}
|S| \geq d \poly( \log d,t), % t^{p+3}   d \log^{p+4} d,
\end{align*} 
with probability at least $1-d^{-t}$, for any $W$ (which is not necessarily to be independent of samples) satisfying $\|W^a - W\|\leq \sigma_k/4$, we have 
$$\|\nabla^2 \widehat{f}_S (W) - \nabla^2 \widehat{f}_S (W^a) \|\leq k v^2_{\max}\sigma_1^{p} (\|W^a-W^*\|+\|W - W^a\| d^{(p+1)/2}).$$
\end{lemma}

\begin{proof}
Let $\Delta = \nabla^2 \widehat{f}_S (W) - \nabla^2 \widehat{f}_S (W^a) \in \mathbb{R}^{dk \times dk}$, then $\Delta$ can be thought of as $k^2$ blocks, and each block has size $d\times d$. 
The off-diagonal blocks are,
\begin{align*}
 \Delta_{i,l} = v^{*}_iv^{*}_l  \frac{1}{|S|}\sum_{x\in S} \left( \phi'( w_i^\top x)\phi'( w_l^\top x)- \phi'( w_i^{a\top} x) \phi'( w_l^{a\top} x) \right) x x^\top  
\end{align*}
 For diagonal blocks,
\begin{align*}
 \Delta_{i,i} = & \frac{1}{|S|}\sum_{x\in S} \left( \left( \overset{k}{ \underset{q=1}{\sum} } v_q^* \phi( w_{q}^\top x)  - y \right) v_i^* \phi''( w_i^\top x) x x^\top  + v^{*2}_i  \phi'^2( w_i^\top x) x x^\top \right) \\
& - \frac{1}{|S|}\sum_{x\in S} \left( \left( \overset{k}{ \underset{q=1}{\sum} } v_q^* \phi( w_{q}^{a\top} x)  - y \right) v_i^* \phi''( w_i^{a\top} x) x x^\top  + v^{*2}_i \phi'^2( w_i^{a\top} x)  x x^\top \right) \\
\end{align*}

We further decompose $\Delta_{i,i}$ into $\Delta_{i,i} = \Delta_{i,i}^{(1)}+\Delta_{i,i}^{(2)}$, where
\begin{align*}
 \Delta_{i,i}^{(1)} =  v_i^* \frac{1}{|S|}\sum_{x\in S} \left( \left( \overset{k}{ \underset{q=1}{\sum} } v_q^* \phi( w_{q}^\top x )  - y \right)  \phi''( w_i^\top x )  - \left( \overset{k}{ \underset{q=1}{\sum} } v_q^* \phi( w_{q}^{a\top} x )  - y \right) \phi''( w_i^{a\top} x ) \right) x x^\top ,
\end{align*}
and
\begin{align}\label{eq:delta_ii_2_hat}
 \Delta_{i,i}^{(2)} = &  v^{*2}_i   \frac{1}{|S|}\sum_{ (x,y) \in S}  \phi'^2( w_i^\top x)  x x^\top - \phi'^2( w_i^{a\top} x) x x^\top .
\end{align}

We can further decompose $\Delta_{i,i}^{(1)}$ into $\Delta_{i,i}^{(1,1)}$ and $\Delta_{i,i}^{(1,2)}$,
\begin{align*}
\Delta_{i,i}^{(1)} = & ~ v_i^* \frac{1}{|S|}\sum_{x\in S}  \left( \overset{k}{ \underset{q=1}{\sum} } v_q^* \phi( w_{q}^\top x)  - y \right)  \phi''( w_i^\top x)  - \left( \sum_{q=1}^k v_q^* \phi( w_{q}^{a\top} x)  - y \right) \phi''( w_i^{a\top} x)  xx^\top \\
= & ~ v_i^* \frac{1}{|S|}\sum_{x\in S}  \left( \sum_{q=1}^k  v_q^* \phi( w_{q}^\top x)  -\overset{k}{{\underset{q=1}{\sum} }}  v_q^* \phi( w_{q}^{a\top} x) \right)  \phi''( w_i^\top x) xx^\top \notag  \\
+ & ~ v_i^* \frac{1}{|S|}\sum_{x\in S}  \left( \sum_{q=1}^k v_q^* \phi( w_{q}^{a\top} x)  - y \right)(\phi''( w_i^\top x)- \phi''( w_i^{a\top} x))     x x^\top  \notag  \\
= & ~ v_i^* \frac{1}{|S|}\sum_{x\in S}   \sum_{q=1}^k  v_q^* ( \phi( w_{q}^\top x)  -\phi( w_{q}^{a\top} x) )  \phi''( w_i^\top x) xx^\top \notag  \\
+ & ~ v_i^* \frac{1}{|S|}\sum_{x\in S}   \sum_{q=1}^k v_q^* ( \phi( w_{q}^{a\top} x)  - \phi(w_q^{*\top }  x) ) (\phi''( w_i^\top x)- \phi''( w_i^{a\top} x))     x x^\top  \notag  \\
= & ~ \Delta_{i,i}^{(1,1)} + \Delta_{i,i}^{(1,2)}.
\end{align*}

Combining Claim~\ref{cla:SW_SWa_bound_Delta_ii_11}  and Claim~\ref{cla:SW_SWa_bound_Delta_ii_12} , we have if
\begin{align*}
n \geq  d \poly(\log d, t) %  t^{p+3} d \log^{p+4} d,
\end{align*} 
with probability at least $1-1/d^{4t}$,
\begin{align}\label{eq:decomp_delta_ii_1}
 \|\Delta_{i,i}^{(1)}\| \lesssim k v^2_{\max}\sigma_1^{p}  (\|W^a-W^*\|+ \| W^a - W \| d^{(p+1)/2}).
 \end{align}

Therefore, combining Eq.~\eqref{eq:decomp_delta_ii_1}, Claim~\ref{cla:decomp_delta_ii_2} and Claim~\ref{cla:decomp_delta_il}, we complete the proof.
\end{proof}

\begin{claim}\label{cla:SW_SWa_bound_Delta_ii_11}%{cla:decomp_delta_ii_11}
For each $i\in [k]$, if $n\geq d \poly(\log d ,t)$, then
\begin{align*}
\| \Delta_{i,i}^{(1,1)} \| \lesssim k v_{\max}^2 \sigma_1^p \| W^a - W \| d^{(p+1)/2}
\end{align*}
\end{claim}
\begin{proof}

Define function $h_1(x)= \|x\|^{p+1} $ and $h_2(x) = | w_{q}^{*\top} x|^p | ( w_q^* - w_{q}^{a})^\top x|$. %Note that $\|\Delta_{ii}^{(1)}\| \leq \|\frac{1}{n} \sum_{j=1}^n \sum_{q=1}^k v_q^* L_1L_2\delta_0 \| w_q^a\| h_1(x_j)x_jx_j^\top + L_1L_2 v_q^* h_2(x_j) x_jx_j^\top\|$
Note that $h_1$ and $h_2$ don't contain $W$ which maybe depend on the samples. Therefore, we can use the modified matrix Bernstein inequality (Corollary~\ref{cor:modified_bernstein_tail_xx}) to bound $\Delta_{i,i}^{(1)}$. 
%First we bound $\|\frac{1}{n}\sum_{j=1}^n \|x_j\|^{p+1} x_jx_j^\top - \Expect{\|x\|^{p+1} xx^\top}\|$.
%We check if $h_1$ satisfies the conditions in Corollary~\ref{cor:modified_bernstein_tail_xx}. 

(\RN{1}) Bounding $|h_1(x)|$.

 By Fact~\ref{fac:gaussian_norm_bound}, we have $h_1(x) \lesssim ( t d \log  n)^{(p+1)/2}$ with probability at least $1-1/(nd^{4t})$. %$R_h(t) = (t d )^{(p+1)/2}$ and $K_h(t) = (p+1)/2$. 

(\RN{2}) Bounding $\|\E_{x\sim \D_d} [ \|x\|^{p+1}xx^\top ]\| $.

 Let $g(x) = (2\pi)^{-d/2} e^{-\|x\|^2/2}$. Note that $x g(x) \mathrm{d} x = -\mathrm{d} g(x)$.
\begin{align*}
\E_{x\sim \D_d} \left[ \|x\|^{p+1}xx^\top \right] = & ~ \int \|x\|^{p+1}g(x) xx^\top \mathrm{d} x \\
= & ~ -\int \|x\|^{p+1} \mathrm{d} (g(x)) x^\top\\
= & ~ - \int \|x\|^{p+1}  \mathrm{d} (g(x)x^\top) +  \int \|x\|^{p+1}  g(x)I_d \mathrm{d} x \\
= & ~ \int d(\|x\|^{p+1}) g(x) x^\top +  \int \|x\|^{p+1}  g(x)I_d \mathrm{d} x \\
= & ~ \int (p+1)\|x\|^{p-1} g(x) x x^\top \mathrm{d} x+  \int \|x\|^{p+1}  g(x)I_d \mathrm{d} x \\
\succeq & ~ \int \|x\|^{p+1}  g(x) I_d \mathrm{d} x \\
= & ~ \E_{x\sim \D_d} [ \|x\|^{p+1} ] I_d .
\end{align*}
Since $\|x\|^2$ follows $\chi^2$ distribution with degree $d$, $\E_{x\sim \D_d} [ \|x\|^q ] = 2^{q/2}\frac{\Gamma((d+q)/2)}{\Gamma(d/2)}$ for any $q\geq 0$. So, $ d^{q/2} \lesssim \E_{x\sim \D_d} [ \|x\|^q ] \lesssim d^{q/2}$. Hence, $\|\E_{x\sim \D_d} [ h_1(x) x x^\top ] \| \gtrsim d^{(p+1)/2}$. Also 
\begin{align*}
\left\|\E_{x\sim \D_d} \left[ h_1(x) x x^\top \right] \right\| \leq & ~ \max_{\| a\|=1} \E_{x\sim \D_d} \left[ h_1(x) (x^\top a)^2 \right]   \\
\leq & ~ \max_{\| a\|=1} \left( \E_{x\sim \D_d} \left[ h_1^2(x) \right] \right)^{1/2} \left( \E_{x\sim \D_d} \left[ (x^\top a)^4 \right] \right)^{1/2} \\
\lesssim & ~ d^{(p+1)/2}.
\end{align*}

(\RN{3}) Bounding $ (\E_{x\sim \D_d} [ h^4_1(x)  ] )^{1/4}$.

\begin{align*}
\left(\E_{x\sim \D_d} [ h^4_1(x)  ] \right)^{1/4} \lesssim d^{(p+1)/2} .
\end{align*}

Define function  $B(x)= h(x) x x^\top \in \mathbb{R}^{d\times d}$, $\forall i\in[n]$. Let $\ov{B} = \E_{x \sim \D_d} [ h(x) x x^\top ]$.
%Let $B_i = g(x_i)x_ix_i^\top$ and $B = \Expect{B_i}$.
Therefore by applying Corollary~\ref{cor:modified_bernstein_tail_xx}, we obtain for any $0< \epsilon <1$, if 
\begin{align*}
n \geq \epsilon^{-2} d \poly(\log d, t) %\delta^{-2} t^{p/2+5/2} \cdot d \log^{p/2+7/2}(d),
\end{align*} 
with probability at least $1-1/d^{t}$,
\begin{align*}
\left\|\frac{1}{|S|}\sum_{x\in S} \|x\|^{p+1} xx^\top - \E_{x\sim \D_d} \left[ \|x\|^{p+1} xx^\top \right] \right\| \lesssim \delta d^{(p+1)/2}.
\end{align*}

Therefore we have with probability at least $1-1/d^{t}$,
\begin{align}\label{eq:d_norm_bound}
\left\|\frac{1}{|S|}\sum_{x\in S} \|x\|^{p+1} xx^\top \right\| \lesssim  d^{(p+1)/2}.
\end{align}
\end{proof}

\begin{claim}\label{cla:SW_SWa_bound_Delta_ii_12}
For each $i\in [k]$, if $n\geq d \poly(\log d, t)$, then
\begin{align*}
\| \Delta_{i,i}^{(1,2)} \| \lesssim k \sigma_1^p \| W^a - W^* \|,
\end{align*}
holds with probability at least $1-1/d^{4t}$.
\end{claim}
\begin{proof}
Recall the definition of $\Delta_{i,i}^{(1,2)}$,
\begin{align*}
\Delta_{i,i}^{(1,2)} = v_i^* \frac{1}{|S|}\sum_{x\in S}   \sum_{q=1}^k v_q^* ( \phi( w_{q}^{a\top} x)  - \phi(w_q^{*\top }  x) ) (\phi''( w_i^\top x)- \phi''( w_i^{a\top} x))     x x^\top
\end{align*}
In order to upper bound the $\| \Delta_{i,i}^{(1,2)} \| $, it suffices to upper bound the spectral norm of this quantity,
\begin{align*}
\frac{1}{|S|} \sum_{x\sim S} ( \phi( w_{q}^{a\top} x)  - \phi(w_q^{*\top }  x) ) (\phi''( w_i^\top x)- \phi''( w_i^{a\top} x))     x x^\top.
\end{align*}
By Property \ref{pro:gradient}, we have 
\begin{align*}
| \phi( w_{q}^{a\top} x)  - \phi(w_q^{*\top }  x) | \lesssim L_1 ( |w_q^{a\top } x|^p + |w_q^{*\top }x| ) |  ( w_q^* - w_q^a )^\top x |.
\end{align*}
By Property \ref{pro:hessian}, we have $ |\phi''( w_i^\top x)- \phi''( w_i^{a\top} x) | \leq 2L_2$.

For the second part $ | w_{q}^{*\top} x|^p | ( w_q^* - w_{q}^{a})^\top x| x x^\top$, according to Eq.~\eqref{eq:abs_g}, we have with probability $1-d^{-t}$, if $n\geq d \poly(\log d , t)$, 
\begin{align*}
 \left\|\E_{x\sim \D_d} \left[| w_{q}^{*\top} x |^p | ( w_q^* - w_{q}^{a})^\top x | x x^\top \right] - \frac{1}{|S|}\sum_{x\in S}  | w_{q}^{*\top} x|^p | ( w_q^* - w_{q}^{a})^\top x| xx^\top \right\| \lesssim \delta \| w_{q}^*\|^p \| w_q^* -  w_q^{a}\| .
\end{align*}

Also, note that 
\begin{align*}
 \left\|\E_{x\sim \D_d} \left[ | w_{q}^{*\top} x|^p | ( w_q^* - w_{q}^{a})^\top x| xx^\top \right] \right\| \lesssim  \| w_{q}^*\|^p \| w_q^* -  w_q^{a}\| .
 \end{align*}
Thus, we obtain
\begin{align}\label{eq:norm_average_bound}
\left\| \frac{1}{|S|} \sum_{x\in S} | w_{q}^{*\top} x |^p | ( w_q^* - w_{q}^{a})^\top x |x x^\top\right\| \lesssim \|W^a-W^*\| \sigma_1^p.
\end{align}
\end{proof}

\begin{claim}\label{cla:decomp_delta_ii_2}
For each $i\in [k]$, if $n \geq d \poly(\log d ,t )$, then 
\begin{align*}
 \|\Delta_{i,i}^{(2)}\| \lesssim k v^2_{\max}\sigma_1^{p} \|W- W^a \| d^{(p+1)/2} 
\end{align*}
holds with probability $1-1/d^{4t}$.
\end{claim}
\begin{proof}
We have
\begin{align*}
\|\Delta_{i,i}^{(2)} \|  \leq & ~ v^{*2}_i   \left\|\frac{1}{|S|}\sum_{x\in S}  \left( (\phi'( w_i^\top x_j) -\phi'( w_i^{a\top} x))  \cdot (\phi'( w_i^\top x) +\phi'( w_i^{a\top} x)) \right) x x^\top  \right\| \\
\leq  & ~ v^{*2}_i   \left\| \frac{1}{|S|}\sum_{x\in S}  \left( L_2|( w_i- w_i^{a}) ^\top x|  \cdot L_1( | w_i^\top x |^p +| w_i^{a\top} x |^p) \right) x x^\top  \right\| \\
\leq  & ~ v^{*2}_i  \| W- W^a \| \left\| \frac{1}{|S|}\sum_{x\in S} \left( L_2\|x\| \cdot L_1( \| w_i\|^p \| x\|^p +| w_i^{a\top} x |^p) \right) x x^\top  \right\|.
\end{align*}
Applying Corollary~\ref{cor:modified_bernstein_tail_xx} finishes the proof.

%Next, we bound $\|\Delta_{i,i}^{(2)}\|$.
% if
%\begin{equation*}
%n \gtrsim t^{p+3} d \log^{p+4} d,
%\end{equation*} 
%with probability at least $1-d^{-t}$,
%\begin{equation}\label{eq:decomp_delta_ii_2}
% \|\Delta_{ii}^{(2)}\| \lesssim k v^2_{\max}\sigma_1^{p} \delta_0 d^{(p+1)/2} 
% \end{equation}
\end{proof}

\begin{claim}\label{cla:decomp_delta_il}
For each $i\in [k], j\in [k], i\neq j$, if $n \geq d \poly(\log d , t)$, then
\begin{align*}
\| \Delta_{i,l} \| \lesssim v_{\max}^2 \sigma_1^p \| W^a - W \|
\end{align*}
holds with probability $1-d^{4t}$.
\end{claim}
\begin{proof}
Recall the definition of $\Delta_{i,l}$,
\begin{align*}
 \Delta_{i,l} = & ~ v^{*}_iv^{*}_l  \frac{1}{|S|} \sum_{x\in S} \left( \phi'( w_i^\top x )\phi'( w_l^\top x)- \phi'( w_i^{\top} x ) \phi'( w_l^{a\top} x ) \right. \\
 & \left. + \phi'( w_i^{\top} x) \phi'( w_l^{a\top} x ) - \phi'( w_i^{a\top} x ) \phi'( w_l^{a\top} x) \right) x  x^\top   \\
 = & ~  v^{*}_iv^{*}_l  \frac{1}{|S|} \sum_{x\in S} \left( \phi'( w_i^\top x )\phi'( w_l^\top x)- \phi'( w_i^{\top} x ) \phi'( w_l^{a\top} x ) \right) \\
 & ~ +  v^{*}_iv^{*}_l  \frac{1}{|S|} \sum_{x\in S} \left( \phi'( w_i^{\top} x) \phi'( w_l^{a\top} x ) - \phi'( w_i^{a\top} x ) \phi'( w_l^{a\top} x) \right) x  x^\top \\
 \preceq & ~ |v^{*}_i v^{*}_l|  \frac{1}{|S|}\sum_{x\in S} \left( L_1 \| w_i\|^p L_2 \| w_l -  w_l^a\| \| x \|^{p+1}+ L_2 \| w_i -  w_i^{a}\| \| x \| L_1\| w_l^{a\top} x\|^p \right) x x^\top  
\end{align*}
Applying Corollary~\ref{cor:modified_bernstein_tail_xx} completes the proof.
%Using the same techniques as before, we obtain similar results, if
%\begin{equation*}
%n \gtrsim t^{p+3} d \log^{p+4} d,
%\end{equation*} 
%with probability at least $1-d^{-t}$,
%\begin{equation}\label{eq:decomp_delta_il}
% \|\Delta_{il}\| \lesssim v^2_{\max}\sigma_1^{p} \delta_0 d^{(p+1)/2} .
% \end{equation}
\end{proof}

%Now we are ready to show the linear convergence of gradient descent for smooth activations. 

\section{Tensor Methods}\label{app:tensor}

%\subsection{Main Results for Tensor Methods}

\subsection{Tensor Initialization Algorithm}\label{app:details_tensor_alg}
We describe the details of each procedure in Algorithm~\ref{alg:tensor} in this section. 

{\bf a) Compute the subspace estimation from $\wh{P}_2$ (Algorithm~\ref{alg:power_method})}. Note that the eigenvalues of $P_2$ and $\wh{P}_2$ are not necessarily nonnegative. However, only $k$ of the eigenvalues will have large magnitude. So we can first compute the top-k eigenvectors/eigenvalues of both $C\cdot I + \wh{P}_2 $ and $C\cdot I - \wh{P}_2 $, where $C$ is large enough such that $C \geq 2\|P_2\|$. Then from the $2k$ eigen-pairs, we pick the top-$k$ eigenvectors with the largest eigenvalues in magnitude, which is executed in \textsc{TopK} in Algorithm~\ref{alg:power_method}. For the outputs of \textsc{TopK}, $k_1,k_2$ are the numbers of picked largest eigenvalues from $C\cdot I + \wh{P}_2 $ and $C\cdot I - \wh{P}_2 $ respectively and $\pi_1(i)$ returns the original index of $i$-th largest eigenvalue from $C\cdot I + \wh{P}_2 $ and similarly $\pi_2$ is for $C\cdot I - \wh{P}_2 $. Finally orthogonalizing the picked eigenvectors leads to an estimation of the subspace spanned by $\{ w_1^*\; w_2^*\;\cdots\;  w_k^*\}$. Also note that forming $\wh{P}_2$ takes $O(n\cdot d^2)$ time and each step of the power method doing a multiplication between a $d\times d$ matrix and a $d\times k$ matrix takes $k\cdot d^2$ time by a naive implementation. Here we reduce this complexity from $O((k+n)d^2)$ to $O(knd)$. The idea is to compute each step of the power method without explicitly forming $\wh{P}_2$. We take $P_2 = M_2$ as an example; other cases are similar. In Algorithm~\ref{alg:power_method}, let the step $\wh{P}_2V$ be calculated as $ \wh{P}_2 V=\frac{1}{|S|} \sum_{(x,y)\in S} y ( x  ( x^\top V) -V)$. 
%Computationally, we can approximate the eigenvectors of $\wh{P}_2$ with largest magnitude eigenvalues using power methods in linear time in dimension. 
Now each iteration only needs $O(knd)$ time. Furthermore,
the number of iterations required will be a small number, since the power
method has a linear convergence rate and as an initialization method we don't need a very accurate
solution. The detailed algorithm is shown in Algorithm~\ref{alg:power_method}. The approximation error bound of $\wh{P}_2$ to $P_2$ is provided in Lemma~\ref{lem:tensor_2_m}. Lemma~\ref{lem:subspace_estimation} provides the theoretical bound for Algorithm~\ref{alg:power_method}.

{\bf b) Form and decompose the 3rd-moment $\wh{R}_3$ (Algorithm 1 in \cite{kuleshov2015tensor})}. We apply the non-orthogonal tensor factorization algorithm, Algorithm 1 in \cite{kuleshov2015tensor}, to decompose $\wh{R}_3$. According to Theorem 3 in \cite{kuleshov2015tensor}, when $\wh{R}_3$ is close enough to $R_3$, the output of the algorithm, $ \wh{u}_i$ will close enough to $s_i V^\top \ov w_i^*$, where $s_i$ is an unknown sign. Lemma~\ref{lem:tensor_w_error} provides the error bound for $\|\wh{R}_3-R_3\|$.

{\bf c) Recover the magnitude of $ w_i^*$ and the signs $s_i, v_i^*$ (Algorithm~\ref{alg:recover}). } For Algorithm~\ref{alg:recover}, we only consider homogeneous functions. Hence we can assume $v_i^*\in \{-1,1\}$ and there exist some universal constants $c_j$ such that $m_{j,i} = c_j \| w_i^*\|^{p+1}$ for $j=1,2,3,4$, where $p+1$ is the degree of homogeneity. Note that different activations have different situations even under Assumption~\ref{assumption:non_zero_moment}. In particular, if $M_4=M_2=0$, $\phi(\cdot)$ is an odd function and we only need to know $s_iv_i^*$. If $M_3=M_1=0$, $\phi(\cdot)$ is an even function, so we don't care about what $s_i$ is. %On the other hand, $P_3$ can be equal to $M_3$ or $M_4(I,I,I, \alpha)$. Therefore, we consider four different cases, {\it (1)} $P_3=M_3$ with $M_4=M_2=0$; {\it (2)} $P_3=M_3$ with $M_4\neq 0$ or $M_2\neq 0$; {\it (3)} $P_3 = M_4(I,I,I, \alpha)$ with $M_1=M_3=0$; {\it (4)} $P_3 = M_4(I,I,I, \alpha)$ with $M_1\neq 0$ or $M_3\neq 0$.  

Let's describe the details for Algorithm~\ref{alg:recover}. First define two quantities $Q_1$ and $Q_2$,
\begin{align}
& Q_1 = M_{l_1}(I,\underbrace{ \alpha,\cdots, \alpha}_{(l_1 - 1)~ \alpha\text{'s}}) = \sum_{i=1}^k v_i^* c_{l_1} \| w_i^*\|^{p+1} ( \alpha^\top \ov{ w}_i^*)^{{l_1-1}} \ov{ w}_i^{*},\label{eq:def_Q1} \\
& Q_2 = M_{l_2}(V,V,\underbrace{ \alpha,\cdots, \alpha}_{(l_2 - 2)~ \alpha\text{'s}}) = \sum_{i=1}^k v_i^* c_{l_2} \| w_i^*\|^{p+1} ( \alpha^\top \ov{ w}_i^*)^{{l_2-2}} (V^\top\ov{ w}_i^{*})(V^\top\ov{ w}_i^{*})^\top,\label{eq:def_Q2} 
\end{align}
where $l_1\geq 1$ such that $M_{l_1} \neq 0$ and $l_2 \geq 2$ such that $M_{l_2} \neq 0$. There are possibly multiple choices for $l_1$ and $l_2$. We will discuss later on how to choose them. 
Now we solve two linear systems. 
\begin{align}\label{eq:zr_star}
z^*  = \argmin_{z \in \mathbb{R}^k } \left\| \sum_{i=1}^k z_i s_i\ov{ w}_i^*  - Q_1 \right\|, \text{~and~}
 r^*  = \argmin_{ r \in \mathbb{R}^k } \left\|\sum_{i=1}^k r_i V^\top \ov  w_i^* (V^\top \ov  w_i^*)^\top - Q_2 \right\|_F. 
\end{align}
The solutions of the above linear systems are
\begin{align*}
 z^*_i & = v_i^*s_i^{l_1}   c_{l_1} \| w_i^*\|^{p+1} ( \alpha^\top s_i\ov{ w}_i^*)^{{l_1-1}}, \text{~and~}
r^*_i = v_i^*s_i^{l_2} c_{l_2} \| w_i^*\|^{p+1} ( \alpha^\top s_i\ov{ w}_i^*)^{{l_2-2}} .
\end{align*}
We can approximate $s_i \ov  w_i^* $ by $V \wh{u}_i$ and approximate $Q_1$ and $Q_2$ by their empirical versions $\wh{Q}_1$ and $\wh{Q}_2$ respectively. Hence, in practice, we solve
\begin{equation}\label{eq:linear_systems}
\wh{z}  = \argmin_{z \in \mathbb{R}^k } \left\| \sum_{i=1}^k z_i  V  \wh{u}_i - \wh{Q}_1 \right\|, \text{~and~}
 \wh{r}  = \argmin_{ r \in \mathbb{R}^k } \left\|\sum_{i=1}^k r_i  \wh{u}_i  \wh{u}_i^\top - \wh{Q}_2 \right\|_F
\end{equation}
So we have the following approximations,
\begin{align*}
  \wh{z}_i \approx v_i^*s_i^{l_1} c_{l_1} \| w_i^*\|^{p+1} ( \alpha^\top V \wh{u}_i)^{{l_1-1}}, \text{~and~}
\wh{r}_i \approx v_i^*s_i^{l_2} c_{l_2} \| w_i^*\|^{p+1} ( \alpha^\top V \wh{u}_i)^{{l_2-2}}, \forall i \in [k] .
\end{align*}
In Lemma~\ref{lem:solu_linsys1} and Lemma~\ref{lem:solu_linsys2}, we provide robustness of the above two linear systems, i.e., the solution errors, $\|\wh{z}-z^*\|$ and $\| \wh{r} -  r^*\|$, are bounded under small perturbations of $\ov w_i^*,Q_1$ and $Q_2$.
Recall that the final goal is to approximate $\| w_i^*\|$ and the signs $v_i^*,s_i$. 
Now we can approximate $\| w_i^*\|$ by $(|  \wh{z}_i / (c_{l_1}  ( \alpha^\top V \wh{u}_i)^{{l_1-1}}) |)^{1/(p+1)} $. 
To recover $v_i^*,s_i$, we need to note that if $l_1$ and $l_2$ are both odd or both even, we can't recover both $v_i^*$ and $s_i$. So we consider the following situations,
\begin{enumerate}
\item If $M_1=M_3=0$, we choose $l_1 = l_2 = \min\{j\in \{2,4\} | M_j \neq 0 \}$. Return $v_i^{(0)} = \sign(\wh{r}_i c_{l_2} )  $ and $s_i^{(0)}$ being $-1$ or $1$. \label{13is0}
\item If $M_2=M_4=0$, we choose $l_1 = \min\{j\in \{1,3\} | M_j \neq 0 \}$, $l_2 = 3$. Return $v_i^{(0)}$ being $-1$ or $1$ and $s_i^{(0)} = \sign(v_i^{(0)} \wh{z}_i c_{l_1})$. \label{24is0}
\item Otherwise, we choose $l_1 = \min\{j\in \{1,3\} | M_j \neq 0 \}$, $l_2 = \min\{j\in \{2,4\} | M_j \neq 0 \}$.  Return $v_i^{(0)} = \sign(\wh{r}_i c_{l_2} ) $ and $s_i^{(0)} = \sign(v_i^{(0)} \wh{z}_i c_{l_1})$.
\end{enumerate}
The 1st situation corresponds to part 3 of Assumption~\ref{assumption:non_zero_moment},%\ref{assume:13}, 
where $s_i$ doesn't matter, and the 2nd situation corresponds to part 4 of  Assumption~\ref{assumption:non_zero_moment},
%\ref{assume:24},
 where only $s_iv_i^*$ matters. So we recover $\| w_i^*\|$ to some precision and $v_i^*,s_i$ exactly provided enough samples. The recovery of $ w_i^*$ and $v_i^*$ then follows. 

%The main idea of \textsc{RecMagSign} is to use the empirical 1st-order moment and the reduced ($k$-by-$k$) empirical 2nd-order moment to solve linear systems to estimate the magnitude of $ w_i^*$ and the signs $s_i, v_i^*$. Now we consider a simple example to show the essential idea of \textsc{RecMagSign}. Let $M_2\neq 0$ and $\wh{Q}_2$ be the empirical version of $M_2(V,V)$, then we can obtain $\| w_i^*\|$ and the sign information of $v_i^*$ from $ \wh{r}$ where $ \wh{r} = \argmin_{ r}\|\wh{Q}_2 - \sum_{i=1}^k r_i  \wh{u}_i  \wh{u}_i^\top\|$, which is an overdetermined linear system and whose output $\wh{r}_i$ approximates $c_{2} v_i^* \| w_i^*\|^{p+1}$. More discussions and details about \textsc{RecMagSign} can be found in Appendix~\ref{app:algo} and Algorithm~\ref{alg:recover}. 
%Sample complexity and computational complexity are discussed as follows. 

\begin{algorithm}[t]
\caption{Power Method via Implicit Matrix-Vector Multiplication}
\label{alg:power_method}
\begin{algorithmic}[1]
%\ENSURE Estimation of $\{v_i^{(0)}, w^{(0)}_i\}_{i=1}^k$
\Procedure{\textsc{PowerMethod}}{$\wh{P}_2,k$} 
\State $C\leftarrow 3\|\wh{P}_2\|$, $T \leftarrow$ a large enough constant.
\State Initial guess $\wh{V}_1^{(0)} \in \R^{d\times k},\wh{V}_2^{(0)} \in \R^{d\times k}$
\For{$t=1 \to T$}
%\State $U_1 \leftarrow \textsc{Multiple}(c,+1,\wh{V}_1^{(t-1)}, S)$ \Comment{$U_1\approx (c\cdot I + \wh{P}_2) \wh{V}_1^{(t-1)}$ }
\State $\wh{V}_1^{(t)} \leftarrow \textsc{QR}(C\wh{V}_1^{(t-1)} +\wh{P}_2 \wh{V}_1^{(t-1)} )$ \Comment{$\wh{P}_2 \wh{V}_1^{(t-1)}$ is not calculated directly, see Sec.~\ref{app:details_tensor_alg}(a)}
%\State $U_2 \leftarrow \textsc{Multiple}(c,-1,\wh{V}_2^{(t-1)}, S)$ \Comment{$U_2\approx (c\cdot I - \wh{P}_2) \wh{V}_2^{(t-1)}$ }
\State $\wh{V}_2^{(t)} \leftarrow \textsc{QR}(C\wh{V}_2^{(t-1)} - \wh{P}_2 \wh{V}_2^{(t-1)})$
\EndFor
\For {$j=1, 2$} 
	\State $\wh{V}_j^{(T)} \leftarrow \begin{bmatrix} \wh{v}_{j,1} & \wh{v}_{j,2} & \cdots & \wh{v}_{j,k} \end{bmatrix}$
	\For{$i=1\to k$}
		\State $\lambda_{j,i} \leftarrow |  \wh{v}_{j,i}^\top \wh{P}_2  \wh{v}_{j,i}|$ \Comment{Calculate the absolute of eigenvalues}
	\EndFor
\EndFor  
\State $\pi_1, \pi_2, k_1, k_2  \leftarrow \textsc{TopK}(\lambda,k)$ \Comment{$\pi_j : [k_j] \rightarrow [k]$ and $k_1+k_2=k$, see Sec.~\ref{app:details_tensor_alg}(a)}
\For {$j=1, 2$}
	\State $V_j \leftarrow \begin{bmatrix} \wh{v}_{j,\pi_j(1)} & \wh{v}_{1,\pi_j(2)} & \cdots & \wh{v}_{j,\pi_j(k_j)} \end{bmatrix}$
\EndFor
\State $\wt{V}_2 \leftarrow \textsc{QR}((I-V_1V_1^\top)V_2)$
\State $V \leftarrow [V_1, \wt{V}_2]$
\State \Return $V $
\EndProcedure
%\Procedure{\textsc{Multiple}}{$c,\sigma,V,S$} \Comment{$\approx (c \cdot I +\sigma \cdot \wh{P}_2) V$}
%\State $u\leftarrow \frac{1}{|S|} \underset{(x,y)\in S}{\sum} y \cdot ( \sigma\cdot( xx^\top V - V) + c \cdot V)$
%\State \Return $u$
%\EndProcedure
\end{algorithmic}
\end{algorithm}

\begin{algorithm}[t]
\caption{Recovery of the Ground Truth Parameters of the Neural Network, i.e., $w_i^*$ and $v_i^*$}
\label{alg:recover}
\begin{algorithmic}[1]
%\ENSURE Estimation of $\{v_i^{(0)}, w^{(0)}_i\}_{i=1}^k$
\Procedure{\textsc{RecMagSign}}{$V,\{ \wh{u}_i\}_{i\in [k]}, S$} 
%\Comment{Theorem~\ref{}}
\If{$M_1=M_3=0$}
\State $l_1 \leftarrow l_2 \leftarrow \min\{j\in \{2,4\} | M_j \neq 0 \}$
\ElsIf{$M_2=M_4=0$}
\State $l_1 \leftarrow \min\{j\in \{1,3\} | M_j \neq 0 \}$, $l_2 \leftarrow 3$
\Else
\State $l_1 \leftarrow \min\{j\in \{1,3\} | M_j \neq 0 \}$, $l_2 \leftarrow \min\{j\in \{2,4\} | M_j \neq 0 \}$.
\EndIf
\State $S_1, S_2\leftarrow \textsc{Partition}(S,2)$ \Comment{$|S_1|, |S_2| = \wt{\Omega}(d)$} %Divide $S$ into $S = S_1\cup S_2$ each with size $\wt{\Omega}(d)$. 
\State Choose $\alpha$ to be a random unit vector
%\State Compute the empirical versions of the following two quantities, $\wh{Q}_1,\wh{Q}_2$, from $S_1,S_2$ respectively.
\State $\wh{Q}_1 \leftarrow \E_{S_1} [Q_1]$ \Comment{$\wh{Q}_1$ is the empirical version of $Q_1$(defined in Eq.\eqref{eq:def_Q1})} 
\State $\wh{Q}_2 \leftarrow \E_{S_2} [Q_2]$ \Comment{$\wh{Q}_2$ is the empirical version of $Q_2$(defined in Eq.\eqref{eq:def_Q2})} 
%\begin{align*}
%& Q_1 := M_{l_1}(I,\underbrace{ \alpha,\cdots, \alpha}_{(l_1 - 1)~ \alpha\text{'s}}) ,\;Q_2 := M_{l_2}(V,V,\underbrace{ \alpha,\cdots, \alpha}_{(l_2 - 2)~ \alpha\text{'s}}).
%\end{align*}
\State $\wh{z} \leftarrow \argmin_{z} \left\| \sum_{i=1}^k z_i  V  \wh{u}_i - \wh{Q}_1 \right\| $
\State $\wh{r} \leftarrow \argmin_{ r} \left\|\sum_{i=1}^k r_i  \wh{u}_i  \wh{u}_i^\top - \wh{Q}_2 \right\|_F$
\State $v_i^{(0)} \leftarrow \sign(\wh{r}_i c_{l_2} ) $
\State $s_i^{(0)} \leftarrow \sign(v_i^{(0)} \wh{z}_i c_{l_1})$
\State $ w_i^{(0)} \leftarrow s_i^{(0)} (|  \wh{z}_i / (c_{l_1}  ( \alpha^\top V \wh{u}_i)^{{l_1-1}}) |)^{1/(p+1)}  V \wh{u}_i$
\State \Return $v_i^{(0)}$,$ w_i^{(0)}$
\EndProcedure
\end{algorithmic}
\end{algorithm}

%The estimation of the 2nd-order moment $P_2$ requires only $O(d)$ samples to achieve a certain precision. 
{\bf Sample complexity:} We use matrix Bernstein inequality to bound the error between $\wh{P}_2$ and $P_2$, which requires $\wt{\Omega}(d)$ samples (Lemma~\ref{lem:tensor_2_m}). To bound the estimation error between $R_3$ and $\wh{R}_3$, we flatten the tensor to a matrix and then use matrix Bernstein inequality to bound the error, which requires $\wt{\Omega}(k^3)$ samples (Lemma~\ref{lem:tensor_w_error}). In Algorithm~\ref{alg:recover}, we also need to approximate a $\R^d$ vector and a $\R^{k\times k}$ matrix, which also requires $\wt{\Omega}(d)$. Thus, taking $\wt{O}(d) + \wt{O}(k^3)$ samples is sufficient.

{\bf Time complexity:} In Part a), by using a specially designed power method, we only need $O(knd)$ time to compute the subspace estimation $V$. Part b) needs $O(knd)$ to form $\wh{R}_3$ and the tensor factorization needs $O(k^3)$ time. Part c) requires calculation of $d\times k$ and $k^2 \times k$ linear systems in Eq.~\eqref{eq:linear_systems}, which takes at most $O(knd)$ running time. Hence, the total time complexity is $O(knd)$. 

\subsection{Main Result for Tensor Methods}\label{app:tensor_main_result}
The goal of this Section is to prove Theorem~\ref{thm:tensor_final}.
\restate{thm:tensor_final}
\begin{proof}
The success of Algorithm~\ref{alg:tensor} depends on two approximations. The first is the estimation of the normalized $ w_i^*$ up to some unknown sign flip, i.e., the error $\|\ov w_i^*- s_iV \wh{u}_i  \|$ for some $s_i\in\{-1,1\}$. The second is the estimation of the magnitude of $ w_i^*$ and the signs $v_i^*,s_i$ which is conducted in Algorithm~\ref{alg:recover}. 

For the first one, 
\begin{align}\label{eq:wi_minus_siVui}
\|\ov w_i^*- s_iV \wh{u}_i\| \leq & ~\| VV^\top \ov w_i^* -\ov w_i^* \| + \| V V^\top \ov w_i^* - V s_i \wh{u}_i\| \notag \\
= &~ \| VV^\top \ov w_i^* -\ov w_i^* \| + \| V^\top \ov w_i^* - s_i \wh{u}_i\|,
\end{align}
where the first step follows by triangle inequality, the second step follows by $V^\top V = I$.

We can upper bound $\| VV^\top \ov w_i^* -\ov w_i^* \|$,
\begin{align}\label{eq:wi_minus_siVui_part1}
\| VV^\top \ov w_i^* -\ov w_i^* \| \leq & ~ ( \| \wh{P}_2 - P_2 \| / \sigma_k(P_2) + \epsilon)  \notag \\
\leq & ~ ( \poly(k,\kappa) \| \wh{P}_2 - P_2 \| + \epsilon) \notag \\
\leq & ~  \poly(k,\kappa)  \epsilon , 
\end{align}

where the first step follows by Lemma~\ref{lem:subspace_estimation}, the second step follows by $\sigma_{k}(P_2) \geq 1/\poly(k,\kappa)$, and the last step follows by $ \| \wh{P}_2 - P_2 \| \leq \epsilon \poly(k,\kappa)$ if the number of samples is proportional to $\wt{O}(d/\epsilon^2)$ as shown in Lemma~\ref{lem:tensor_2_m}.

We can upper bound $\|V^\top \ov w_i^* - s_i \wh{u}_i\|$,
\begin{align}\label{eq:wi_minus_siVui_part2}
\|V^\top \ov w_i^* - s_i \wh{u}_i\| \leq  \poly(k,\kappa) \|\wh{R}_3 - R_3\|
\leq  \epsilon \poly(k, \kappa),
\end{align}
where the first step follows by Theorem 3 in \cite{kuleshov2015tensor}, and the last step follows by $\|\wh{R}_3 - R_3\| \leq \epsilon \poly(k,\kappa)$ if the number of samples is proportional to $\wt{O}(k^2/\epsilon^2)$ as shown in Lemma~\ref{lem:tensor_w_error}.

Combining Eq.~\eqref{eq:wi_minus_siVui}, \eqref{eq:wi_minus_siVui_part1} and \eqref{eq:wi_minus_siVui_part2} together,
\begin{align*}
\|\ov w_i^*- s_iV \wh{u}_i\| \leq \epsilon \poly(k,\kappa).
\end{align*}

For the second one, we can bound the error of the estimation of moments, $Q_1$ and $Q_2$, using number of samples proportional to $\wt{O}(d)$ by Lemma~\ref{lem:tensor_1_m} and Lemma~\ref{lem:tensor_2_m} respectively. The error of the solutions of the linear systems Eq.\eqref{eq:linear_systems} can be bounded by $\|Q_1-\wh{Q}_1\|,\|Q_2-\wh{Q}_2\|, \| \wh{u}_i - V^\top \ov w_i^*\|$ and $\|(I-VV^\top)\ov w_i^*\|$ according to Lemma~\ref{lem:solu_linsys1} and Lemma~\ref{lem:solu_linsys2}. Then we can bound the error of the output of Algorithm~\ref{alg:recover}. Furthermore, since $v_i^*$'s are discrete values, they can be exactly recovered. All the sample complexities mentioned in the above lemmata are linear in dimension and polynomial in other factors to achieve a constant error. So accumulating all these errors we complete our proof.

%The roadmap of the proof follows Algorithm~\ref{alg:tensor}. We bound the errors introduced by the steps in Algorithm~\ref{alg:tensor}. In particular, we bound $\|P_2-\wh{P}_2\|$ using Lemma~\ref{lem:tensor_2_m}, $\|(I-VV^\top)\ov w_i^*\|$ using Lemma~\ref{lem:subspace_estimation} and $\|R_3-\wh{R}_3\|$ using Lemma~\ref{lem:tensor_w_error}. In Appendix~\ref{app:linear_sys} we show the robustness of linear systems that are used in Algorithm~\ref{alg:recover}. Finally we accumulate the errors to the final output. Details and proofs can be found in Appendix~\ref{app:tensor}. %The following lemmata provide bounds for each step in Algorithm~\ref{alg:tensor}. 
\end{proof}

\begin{remark} The proofs of these lemmata for Theorem~\ref{alg:tensor} can be found in the following sections. Note that these lemmata also hold for any activations satisfying
Property~\ref{pro:gradient} and Assumption~\ref{assumption:non_zero_moment}.
However, since we are unclear how to implement  the last step of
Algorithm~\ref{alg:tensor} (Algorithm~\ref{alg:recover})
for general non-homogeneous activations, we restrict our theorem to homogeneous activations only.
\end{remark}

\subsection{Error Bound for the Subspace Spanned by the Weight Matrix}
\subsubsection{Error Bound for the Second-order Moment in Different Cases}
\begin{lemma}\label{lem:2nd_for_P2_approx}
Let $M_2$ be defined as in Definition~\ref{def:M_m}. Let $\wh{M}_2$ be the empirical version of $M_2$, i.e.,
\begin{align*}
\wh{M}_2 = \frac{1}{|S|}\sum_{ (x,y)\in S} y \cdot ( x \otimes  x -I_d),  
\end{align*}
where $S$ denote a set of samples from Distribution $\D$ defined in Eq.~\eqref{eq:model}. 
Assume $M_2 \neq 0$, i.e., $m_i^{(2)} \neq 0$ for any $i$. 
Then for any $0<\epsilon<1,t\geq 1$, if 
\begin{align*}
|S| \geq  \max_{i\in[k]} ( \| w_i^*\|^{p+1} / |m_{2,i}|  +1 ) \cdot \epsilon^{-2}   d \poly(\log d , t)
\end{align*} %\cdot t^{p+3}  \log^{p+4} (d)
with probability   at least $1-d^{-t}$,
\begin{align*}
 \|M_2 - \wh{M}_2\| \leq \epsilon \sum_{i=1}^k  |v_i^* m_{2,i} |. 
\end{align*}
\end{lemma}

\begin{proof} 

Recall that, for each sample $(x,y)$, $y = \sum_{i=1}^k v_i^* \phi( w_i^{*\top } x)$. We consider each component $i\in[k]$. 
Define function  $B_i(x) : \mathbb{R}^d \rightarrow \mathbb{R}^{d\times d}$ such that
\begin{align*}
 B_i(x) = \phi( w_i^{*\top } x) \cdot ( x \otimes  x -I_d). 
\end{align*}
 Define $g(z) = \phi(z) - \phi(0)$, then $|g(z)| = |\int_0^z \phi'(s) \mathrm{d} s |\leq L_1/(p+1) |z|^{p+1}$, which follows Property~\ref{pro:gradient}. %We will apply Lemma~\ref{lem:modified_bernstein_non_zero} and check each condition in it.

(\RN{1}) Bounding $\| B_i(x) \|$.

\begin{align*}
\|B_i(x)\|  \lesssim & ~ ( \frac{L_1}{p+1} | w_i^{*\top } x|^{p+1} + |\phi(0)|) (\| x_j\|^2 + 1) \\
\lesssim & ~ ( \frac{L_1}{p+1} \| w_i^*\|^{p+1} + |\phi(0)|) d \poly(\log d, t)
\end{align*}
where the last step follows by Fact~\ref{fac:inner_prod_bound} and Fact~\ref{fac:gaussian_norm_bound}.%, we have for any constant $t \geq1$, with probability  $1-n^{-1}d^{-t}$,
%\begin{align*}
%\|B_j\|   \lesssim  (L_1/(p+1) \| w_i^*\|^{p+1} + |\phi(0)|) d (t \log(n))^{(p+1)/2+1}
%\end{align*}
%Therefore, $R(t) = (L_1/(p+1) \| w_i^*\|^{p+1} + |\phi(0)|) d\cdot  t^{(p+1)/2+1}$ and $K(t) = (p+1)/2+1$. 

(\RN{2}) Bounding $\| \E_{x\sim \D_d} [ B_i(x)] \|$.

 Note that $ \E_{x\sim \D_d}[ B_i(x) ] = m_{2,i} \ov w_i^*\ov w_i^{*\top}$. Therefore, $\|\E_{x\sim \D_d}[ B_i(x) ] \| = | m_{2,i} | $. 

(\RN{3}) Bounding $\max ( \E_{x\sim \D_d} \| B_i(x) B_i(x)^\top \|,\E_{x\sim \D_d} \| B_i(x)^\top B_i(x) \| ) $.

Note that $B_i(x)$ is a symmetric matrix, thus it suffices to only bound one of them.
\begin{align*}
\left\| \E_{x\sim \D_d} [ B_i^2(x) ] \right\| \lesssim \left( \E_{x\sim \D_d}[ \phi( w_i^{*\top} x)^4 ]\right) ^{1/2}  \left( \E_{x\sim \D_d } [ \| x\|^4 ] \right)^{1/2} \lesssim ( \frac{L_1}{ p+1 } \| w_i^*\|^{p+1}+|\phi(0)|)^2 d.
\end{align*}

(\RN{4}) Bounding $\max_{\| a\|=\|b\|=1} (\E_{x\sim \D_d} [(a^\top B_i(x) b)^2] )$.

Note that $B_i(x)$ is a symmetric matrix, thus it suffices to consider the case where $a= b$.
\begin{align*}
\max_{\| a\|=1} \left( \E_{x\sim \D_d} \left[ ( a^\top B_i(x)  a)^2 \right] \right)^{1/2}  \lesssim \left( \E_{x\sim \D_d} [ \phi^4( w_i^{*\top} x) ] \right)^{1/4} \lesssim \frac{L_1}{p+1} \| w_i^*\|^{p+1}+|\phi(0)|.
\end{align*}

Define $L = \| w_i^*\|^{p+1}+|\phi(0)|$.
Then we have for any $0<\epsilon <1$, if
\begin{equation*}
n \gtrsim \frac{ L^2d + | m_{2,i} |^2 +  L |m_{2,i} |  d\cdot \poly(\log d, t) \epsilon}{ \epsilon^2 | m_{2,i} |^2 } t \log d 
\end{equation*} % t^{(p+1)/2+1} \log^{(p+1)/2+1}(n) 
with probability   at least $1-1/d^t$,
\begin{equation*}
\left\| \E_{x\sim \D_d}[B_i(x)]- \frac{1}{|S|}\sum_{x\in S} B_i(x) \right\| \leq \epsilon |m_{2,i} |.
\end{equation*}

\end{proof}

\begin{lemma}\label{lem:3rd_for_P2_approx}
Let $M_3$ be defined as in Definition~\ref{def:M_m}. Let $\wh{M}_3$ be the empirical version of $M_3$, i.e.,
\begin{align*}
\wh{M}_3 = \frac{1}{|S|} \sum_{(x,y)\in S} y \cdot ( x^{\otimes 3} -  x \wt{\otimes} I),  
\end{align*}
where $S$ denote a set of samples (each sample is i.i.d. sampled from Distribution $\D$ defined in Eq.~\eqref{eq:model}). 
Assume $M_3 \neq 0$, i.e., $m_{3,i} \neq 0$ for any $i$. Let $ \alpha$ be a fixed unit vector.
Then for any $0<\epsilon<1,t\geq 1$,  if 
\begin{equation*}
|S| \geq  \max_{i\in[k]} ( \| w_i^*\|^{p+1} /  |m_{3,i} (\ov w_i^{*\top}  \alpha)|^2 + 1  ) \cdot \epsilon^{-2} d \poly(\log d, t)
\end{equation*} %t^{p+4}\log^{p+5} (d) 
with probability   at least $1-1/d^{t}$,
\begin{align*}
 \|M_3(I,I, \alpha) - \wh{M}_3(I,I, \alpha)\| \leq   \epsilon \sum_{i=1}^k  |v_i^* m_{3,i} (\ov w_i^{*\top}  \alpha)|. 
 \end{align*}
\end{lemma}

\begin{proof} 

Since $y = \sum_{i=1}^k v_i^* \phi( w_i^{*\top } x)$. We consider each component $i\in[k]$. 

Define function $B_i(x) : \mathbb{R}^d \rightarrow \mathbb{R}^{d\times d}$ such that
\begin{align*}
B_i(x) = [\phi( w_i^{*\top } x) \cdot ( x^{\otimes 3} -  x \tilde\otimes I)](I,I, \alpha) = \phi( w_i^{*\top } x) \cdot (( x^\top  \alpha) x^{\otimes 2} -  \alpha^\top  x I -  \alpha  x^\top - x \alpha^\top).
\end{align*}
 Define $g(z) = \phi(z) - \phi(0)$, then $|g(z)| = |\int_0^z \phi'(s)ds |\leq \frac{L_1}{p+1} |z|^{p+1} \lesssim |z|^{p+1}$, which follows Property~\ref{pro:gradient}. In order to apply Lemma~\ref{lem:modified_bernstein_non_zero}, we need to bound the following four quantities, %and check each condition in it.

(\RN{1}) Bounding $\| B_i(x) \|$. 
\begin{align*}
\|B_i(x) \|  = & ~ \| \phi( w_i^{*\top } x) \cdot (( x^\top  \alpha) x^{\otimes 2} -  \alpha^\top  x I_d -  \alpha  x^\top - x \alpha^\top) \| \\
\leq & ~ |\phi( w_i^{*\top } x)| \cdot \| ( x^\top  \alpha) x^{\otimes 2} -  \alpha^\top  x I -  \alpha  x^\top - x \alpha^\top \| \\
\lesssim & ~ (  | w_i^{*\top } x|^{p+1} + |\phi(0)| ) \| ( x^\top  \alpha) x^{\otimes 2} -  \alpha^\top  x I -  \alpha  x^\top - x \alpha^\top \| \\
\lesssim & ~ (  | w_i^{*\top } x|^{p+1} + |\phi(0)| ) ( | x^\top  \alpha| \| x\|^2 + 3| \alpha^\top  x|) ,
\end{align*}
where the third step follows by definition of $g(z)$, and last step follows by definition of spectral norm and triangle inequality.

Using Fact~\ref{fac:inner_prod_bound} and Fact~\ref{fac:gaussian_norm_bound}, we have for any constant $t \geq1$, with probability  $1-1/(n d^{4t})$,
\begin{align*}
\|B_i(x)\|   \lesssim  ( \| w_i^*\|^{p+1} + |\phi(0)|) d \poly(\log d, t). % (t \log(n))^{(p+1)/2+2}
\end{align*}
%Therefore, $R(t) = (L_1/(p+1) \| w_i^*\|^{p+1} + |\phi(0)|) d t^{(p+1)/2+2}$ and $K(t) = (p+1)/2+2$. 

(\RN{2}) Bounding $\| \E_{x\sim \D_d}[ B_i(x) ] \| $.

Note that $ \E_{x\sim \D_d}[ B_i(x) ] = m_{3,i} (\ov w_i^{*\top}  \alpha) \ov w_i^*\ov w_i^{*\top}$. Therefore, $\| \E_{x\sim \D_d}[ B_i(x) ] \| = | m_{3,i} (\ov w_i^{*\top}  \alpha)| $. 

(\RN{3}) Bounding $\max( \| \E_{x\sim \D_d} [B_i(x) B_i(x)^\top ] \|, \| \E_{x\sim \D_d} [B_i(x)^\top B_i(x) ] \| )$.

Because matrix $B_i(x)$ is symmetric, thus it suffices to bound one of them,
\begin{align*}
\left\|\E_{x\sim \D_d}[ B_i^2(x) ] \right\| \lesssim & ~ \left( \E_{x\sim \D_d} \left[ \phi( w_i^{*\top} x)^4 \right] \right)^{1/2} \left( \E_{x\sim \D_d} \left[ ( x^\top  \alpha)^4 \right] \right)^{1/2} \left( \E_{x\sim \D_d} [ \| x\|^4 ]\right)^{1/2} \\
 \lesssim & ~ ( \| w_i^*\|^{p+1}+|\phi(0)|)^2 d.
\end{align*}

(\RN{4}) Bounding $\max_{\| a\|=\|b\|=1} (\E_{x\sim \D_d}[ ( a^\top B_i(x)  b)^2] )^{1/2} $.
 
\begin{align*}
 \max_{\| a\|=1} \left(\E_{x\sim \D_d} \left[ ( a^\top B_i(x)  a)^2 \right] \right)^{1/2}  \lesssim \left( \E_{x\sim \D_d} \left[ \phi^4( w_i^{*\top} x) \right]\right)^{1/4} \lesssim  \| w_i^*\|^{p+1}+|\phi(0)|.
\end{align*}

Define $L = \| w_i^*\|^{p+1}+|\phi(0)|$. %% ignore L_1/(p+1)
Then we have for any $0<\epsilon <1$, if
\begin{equation*}
|S| \gtrsim \frac{ L^2d + |m_{3,i} (\ov w_i^{*\top}  \alpha)|^2+  L|m_{3,i} (\ov w_i^{*\top}  \alpha)|  d \cdot \poly(\log d,t)  \epsilon}{ \epsilon^2 |m_{3,i} (\ov w_i^{*\top}  \alpha)|^2 } \cdot t \log d 
\end{equation*} %t^{(p+1)/2+2} \log^{(p+1)/2+2}(n)
with probability at least $1-d^{-t}$,
\begin{equation*}
\left \| \E_{x\sim \D_d} [ B_i(x)] - \frac{1}{|S|}\sum_{x\in S} B_i(x) \right\| \leq \epsilon |m_{3,i} (\ov w_i^{*\top}  \alpha)|.
\end{equation*}

\end{proof}

\begin{lemma}\label{lem:4th_for_P2_approx}
Let $M_4$ be defined as in Definition~\ref{def:M_m}. Let $\wh{M}_4$ be the empirical version of $M_4$, i.e.,
\begin{align*}
  \wh{M}_4 = \frac{1}{|S|}\sum_{(x,y)\in S} y \cdot ( x^{\otimes 4} -  ( x \otimes  x) \wt{\otimes} I +   I \wt{\otimes} I), 
\end{align*}
where $S$ denote a set of samples (where each sample is i.i.d. sampled are sampled from Distribution $\D$ defined in Eq.~\eqref{eq:model}). 
Assume $M_4 \neq 0$, i.e., $m_{4,i} \neq 0$ for any $i$. Let $ \alpha$ be a fixed unit vector.
Then for any $0<\epsilon<1,t\geq 1$, if 
\begin{equation*}
|S| \geq  \max_{i\in[k]} (  \| w_i^*\|^{p+1} / |m_{4,i} | (\ov w_i^{*\top}  \alpha)^2 +1   )^2  \cdot \epsilon^{-2} \cdot d \poly(\log d, t)
\end{equation*} %t^{p+5}  \log^{p+6} (d) 
with probability   at least $1-1/d^{t}$,
$$ \|M_4(I,I, \alpha, \alpha) - \wh{M}_4(I,I, \alpha, \alpha)\| \leq \epsilon \sum_{i=1}^k  |v_i^*m_{4,i}| (\ov w_i^{*\top}  \alpha)^2. $$  
\end{lemma}

\begin{proof} 

Since $y = \sum_{i=1}^k v_i^* \phi( w_i^{*\top } x)$. We consider each component $i\in[k]$. 

Define function $B_i(x) : \mathbb{R}^d \rightarrow \mathbb{R}^{d\times d}$ such that
\begin{align*}
B_i(x) & = ~ [\phi( w_i^{*\top } x) \cdot ( x^{\otimes 4} -  ( x \otimes x) \tilde\otimes I +   I \tilde \otimes I)](I,I, \alpha, \alpha) \\
& = ~ \phi( w_i^{*\top } x ) \cdot (( x^\top  \alpha)^2 x^{\otimes 2} - ( \alpha^\top  x)^2 I - 2( \alpha^\top x)( x \alpha^\top+ \alpha  x^\top) -  x x^\top + 2 \alpha  \alpha^\top + I).
\end{align*}
 Define $g(z) = \phi(z) - \phi(0)$, then $|g(z)| = |\int_0^z \phi'(s)ds |\leq L_1/(p+1) |z|^{p+1} \lesssim |z|^{p+1}$, which follows Property~\ref{pro:gradient}. %We will apply Lemma~\ref{lem:modified_bernstein_non_zero} and check each condition in it.

(\RN{1}) Bounding $\| B_i(x) \|$.

\begin{align*}
\|B_i(x) \|  
\lesssim & |\phi( w_i^{*\top } x )| \cdot (( x^\top  \alpha)^2 \| x\|^2 + 1 +\| x\|^2 + ( \alpha^\top  x )^2) \\
\lesssim & (  | w_i^{*\top } x |^{p+1} + |\phi(0)|) \cdot (( x^\top  \alpha)^2\| x \|^2 + 1 +\| x\|^2 + ( \alpha^\top  x )^2) 
\end{align*}
Using Fact~\ref{fac:inner_prod_bound} and Fact~\ref{fac:gaussian_norm_bound}, we have for any constant $t \geq1$, with probability  $1- 1/(nd^{4t})$,
\begin{align*}
\|B_i(x) \|   \lesssim  (  \| w_i^*\|^{p+1} + |\phi(0)|) d \poly(\log d,t). %(t \log(n))^{(p+1)/2+3}
\end{align*}
%Therefore, $R(t) = (  \| w_i^*\|^{p+1} + |\phi(0)|) d\cdot  t^{(p+1)/2+3}$ and $K(t) = (p+1)/2+3$. 

(\RN{2}) Bounding $\| \E_{x\sim \D_d}[B_i(x)] \|$.

 Note that $ \E_{x\sim \D_d}[ B_i(x) ] = m_{4,i} (\ov w_i^{*\top}  \alpha)^2 \ov w_i^*\ov w_i^{*\top}$. Therefore, $\| \E_{x\sim \D_d}[ B_i(x) ] \| = |m_{4,i} | (\ov w_i^{*\top}  \alpha)^2 $. 

(\RN{3}) Bounding $\max ( \E_{x\sim \D_d} \| B_i(x) B_i(x)^\top \|,\E_{x\sim \D_d} \| B_i(x)^\top B_i(x) \| ) $.

\begin{align*}
\left\|\E_{x\sim \D_d}[ B_i(x)^2] \right\| \lesssim & ~ \left( \E_{x\sim \D_d} [ \phi( w_i^{*\top} x)^4] \right)^{1/2} \left( \E_{x\sim \D_d} [ ( x^\top  \alpha)^8 ] \right)^{1/2} \left( \E_{x\sim \D_d}[\| x\|^4] \right)^{1/2} \\
\lesssim & ~ ( \| w_i^*\|^{p+1}+|\phi(0)|)^2 d.
\end{align*}

(\RN{4}) Bounding $\max_{\| a\| = \|b\| =1} (\E_{x\sim \D_d}[( a^\top B_i(x)  b)^2])^{1/2}$.

\begin{align*}
 \max_{\| a\|=1} \left(\E_{x\sim\D_d} \left[ ( a^\top B_i(x)  a)^2 \right] \right)^{1/2} \lesssim \left( \E_{x\sim \D_d} \left[ \phi^4( w_i^{*\top} x) \right] \right)^{1/4} \lesssim \| w_i^*\|^{p+1}+|\phi(0)|.
\end{align*}

Define $L =  \| w_i^*\|^{p+1}+|\phi(0)|$.
Then we have for any $0<\epsilon <1$, if
\begin{equation*}
n \gtrsim \frac{ L^2d + |m_{4,i} |^2 (\ov w_i^{*\top}  \alpha)^4+  L|m_{4,i} |(\ov w_i^{*\top}  \alpha)^2  d \poly(\log d,t) \epsilon}{ \epsilon^2 |m_{4,i}|^2 (\ov w_i^{*\top}  \alpha)^4 } \cdot t \log d 
\end{equation*} %t^{(p+1)/2+3} \log^{(p+1)/2+3}(n) 
with probability   at least $1-d^{-t}$,
\begin{equation*}
\left\| \E_{x\sim \D_d}[ B_i(x)]- \frac{1}{|S|}\sum_{x\in S} B_i(x) \right\| \leq \epsilon |m_{4,i}| (\ov w_i^{*\top}  \alpha)^2.
\end{equation*}

\end{proof}

\subsubsection{Error Bound for the Second-order Moment}%{Proof of Lemma~\ref{lem:tensor_2_m}}

The goal of this section is to prove Lemma~\ref{lem:tensor_2_m}, which shows we can approximate the second order moments up to some precision by using linear sample complexity in $d$.
\begin{lemma}[Estimation of the second order moment]\label{lem:tensor_2_m}
Let $P_2$ and $j_2$ be defined in Definition~\ref{def:P2_P3}.
Let $S$ denote a set of i.i.d. samples generated from distribution ${\cal D}$(defined in \eqref{eq:model}). Let $\wh{P}_2$ be the empirical version of $P_2$ using dataset $S$, i.e., $\wh{P}_2 = \E_{S}[P_2]$.  Assume the activation function satisfies Property~\ref{pro:gradient} and Assumption~\ref{assumption:non_zero_moment}. Then for any $0<\epsilon <1$ and $t\geq 1$, and $m_0=\min_{i\in[k]} \{|m_{j_2,i} |^2 (\ov w_i^{*\top}  \alpha)^{2(j_2-2)} \}$,  if 
\begin{equation*}
|S| \gtrsim \sigma_1^{2p+2}  \cdot d \cdot  \poly(t, \log d)  / (\epsilon^2 m_0),
\end{equation*} 
then with probability  at least $1-d^{-\Omega(t)}$,
\begin{align*}
 \|P_2 - \wh{P}_2\| \leq   \epsilon \sum_{i=1}^k  |v_i^*m_{j_2,i} (\ov w_i^{*\top}  \alpha)^{j_2-2}|. 
 \end{align*}

\end{lemma}

\begin{proof}
We have shown the bound for $j_2=2,3,4$ in Lemma~\ref{lem:2nd_for_P2_approx}, Lemma~\ref{lem:3rd_for_P2_approx} and Lemma~\ref{lem:4th_for_P2_approx} respectively.
To summarize, for any $0<\epsilon <1$ we have if
\begin{equation*}
|S| \geq  \max_{i\in[k]} \left\{\frac{(  \| w_i^*\|^{p+1} + |\phi(0)| + |m_{j_2,i} (\ov w_i^{*\top}  \alpha)^{(j_2-2)}|)^2   }{ |m_{j_2,i}|^2 (\ov w_i^{*\top}  \alpha)^{2(j_2-2)} }\right\} \cdot \epsilon^{-2}  d \poly(\log d, t)
\end{equation*} %\cdot t^{p+5} \cdot d \log^{p+6} (d)
with probability at least $1-d^{-t}$,
\begin{align*}
 \|P_2 - \wh{P}_2\| \leq   \epsilon \sum_{i=1}^k  |v_i^*m_{j_2,i} (\ov w_i^{*\top}  \alpha)^{j_2-2}|. 
 \end{align*}
%Note that  there exists a $ s_i\in\R^d$ such that $\ov{W}^\top  s_i = \be_i$, where $\ov{W} = [\ov w_1^{*}\;\ov w_2^{*}\;\cdots;\ov w_k^{*}]$ and $\be_i$ is one of the standard basis. Since $\sigma_k(P_2) \leq \frac{| s_i^\top P_2  s_i|}{\| s_i\|^2} \leq \frac{|m_{j_2,i} (\ov w_i^{*\top}  \alpha)^{(j_2-2)}|}{\| s_i\|^2}$ and $\| s_i\| \geq \|\ov{W}^\top  s_i\|/\|\ov{W}\| \geq 1/\sqrt{k}$, we have 
%$ \sigma_k(P_2)/k \leq \min_{i}\{|m_{j_2,i} (\ov w_i^{*\top}  \alpha)^{(j_2-2)}|\} $. 
%Also note that $\frac{1}{k} \sum_{i=1}^k  |v_i^*m_{j_2,i} (\ov w_i^{*\top}  \alpha)^{j_2-2}| \leq \|P_2\|  $. So we have for any $0<\delta <1$ we have if
%\begin{equation*}
%n \gtrsim \frac{\sigma_1^{2p+2}}{\sigma_k^2(P_2)} k^2\cdot \delta^{-2} \cdot t^{p+5} \cdot d \log^{p+6} (d) 
%\end{equation*} 
%with probability   at least $1-d^{-t}$,
%$$ \|P_2 - \wh{P}_2\| \leq   \delta k\|P_2\|. $$  
\end{proof}

\subsubsection{Subspace Estimation Using Power Method}%{Proof of Lemma~\ref{lem:subspace_estimation}}

Lemma~\ref{lem:subspace_estimation} shows a small number of power iterations can estimate the subspace of $\{ w_i^*\}_{i\in[k]}$ to some precision, which provides guarantees for Algorithm~\ref{alg:power_method}.
\begin{lemma}[Bound on subspace estimation]\label{lem:subspace_estimation}
%Let $$P_{2} := M_{j_2}(I,I,\underbrace{ \alpha,\cdots, \alpha}_{(j_2 - 2)~ \alpha\text{'s}}) = \sum_{i=1}^k v_i^* m_i^{(j_2)}( \alpha^\top \ov{ w}_i^*)^{{j_2-2}} \ov{ w}_i^{*\otimes 2},$$ where $j_2 = \min \{j\geq 2| M_{j}\neq 0\} $.  
Let $P_2$ be defined as in Definition.~\ref{def:P2_P3} and $\wh{P}_2$ be its empirical version. Let $U \in \mathbb{R}^{d\times k}$ be the orthogonal column span of $W^* \in \R^{d\times k}$. Assume $\|\wh{P}_2 - P_2\| \leq \sigma_k(P_2)/10$. Let $C$ be a large enough positive number such that $C>2\|P_2\|$.
Then after $T = O(\log(1/\epsilon))$ iterations, the output of Algorithm~\ref{alg:power_method}, $V \in \R^{d\times k}$, will satisfy
\begin{align*}
\|UU^\top-VV^\top\| \lesssim  \|\wh{P}_2 - P_2\|  / \sigma_k(P_2) +\epsilon, 
\end{align*}
which implies
\begin{align*}
\|(I-VV^\top)  w_i^*\| \lesssim (\|\wh{P}_2 - P_2\|  / \sigma_k(P_2)+\epsilon) \| w_i^*\| . 
\end{align*}
\end{lemma}

\begin{proof}
According to Weyl's inequality, we are able to pick up the correct numbers of positive eigenvalues and negative eigenvalues in Algorithm~\ref{alg:power_method} as long as $\wh{P}_2$ and $P_2$ are close enough. 

Let $U = [U_1\;U_2] \in \R^{d\times k}$ be the eigenspace of $\text{span}\{ w_1^*\; w_2^*\;\cdots\;  w_k^*\}$, where $U_1 \in \R^{d\times k_1}$ corresponds to positive eigenvalues of $P_2$ and $U_2\in \R^{d\times k_2}$ is for negatives. 

Let $\ov{V}_1$ be the top-$k_1$ eigenvectors of $CI+\wh{P}_2$. Let $\ov{V}_2$ be the top-$k_2$ eigenvectors of $CI-\wh{P}_2$.  Let $\ov{V} = [\ov{V}_1 \; \ov{V}_2] \in \R^{d\times k}$.

According to Lemma 9 in \cite{hk13}, we have
$\|(I - U_1U_1^\top)  \ov{V}_1\| \lesssim \|\wh{P}_2 - P_2\|/\sigma_{k}(P_2) $,$\|(I - U_2U_2^\top)  \ov{V}_2\| \lesssim \|\wh{P}_2 - P_2\| /\sigma_{k}(P_2)$ . Using Fact~\ref{fac:UU_VV}, we have $ \|(I - UU^\top)  \ov{V} \| = \|UU^\top - \bar V\bar V^\top\|$.

Let $\epsilon$ be the precision we want to achieve using power method. Let $V_1$ be the top-$k_1$ eigenvectors returned after $O(\log(1/\epsilon))$ iterations of power methods on $CI+\wh{P}_2$ and $V_2\in\R^{d\times k_2}$ for $CI-\wh{P}_2$ similarly.

According to Theorem 7.2 in \cite{arbenz2012lecture}, we have $\| \bar V_1\bar V_1^\top - V_1 V_1^\top \| \leq \epsilon$ and $\| \bar V_2\bar V_2^\top - V_2 V_2^\top \| \leq \epsilon$.

Let $U_{\perp} $ be the complementary matrix of $U$. Then we have,
\begin{align}\label{eq:IU1U1V1}
 \|(I-U_1U_1^\top) \ov{V}_1\| 
= & ~ \max_{\|a\|=1} \|(I-U_1U_1^\top) \ov{V}_1 a\| \notag \\
= & ~ \max_{\|a\|=1} \|(U_{\perp}U_{\perp}^\top + U_2U_2^\top) \ov{V}_1 a\| \notag \\
= & ~ \max_{\|a\|=1}\sqrt{\|U_{\perp} U_{\perp}^\top \ov{V}_1 a \|^2 + \|U_2 U_2^\top \ov{V}_1 a\|^2} \notag \\
\geq & ~ \max_{\|a\|=1}  \|U_2U_2^\top \ov{V}_1 a\| \notag \\
= & ~ \|U_2U_2^\top \ov{V}_1\|,
\end{align}
where the first step follows by definition of spectral norm, the second step follows by $I = U_1 U_1^\top + U_2 U_2^\top + U_{\perp}^\top U_{\perp}^\top$, the third step follows by $U_2^\top U_{\perp} = 0$, and last step follows by definition of spectral norm.

 We can upper bound $\|(I - UU^\top)  \ov{V} \| $,
\begin{align}\label{eq:I_UU_V_upper_bound}
& ~\|(I - UU^\top)  \ov{V} \| \notag \\
 \leq & ~(\|(I - U_1U_1^\top)  \ov{V}_1\| +\|(I - U_2U_2^\top)  \ov{V}_2\| +\|U_2U_2^\top  \ov{V}_1\|+\|U_1U_1^\top  \ov{V}_2\|) \notag  \\
 \leq & ~ 2(\|(I - U_1U_1^\top)  \ov{V}_1\| +\|(I - U_2U_2^\top)  \ov{V}_2\|) \notag \\
\lesssim & ~\|\wh{P}_2 - P_2\|/\sigma_k(P_2),
\end{align}
where the first step follows by triangle inequality, the second step follows by Eq.~\eqref{eq:IU1U1V1}, and the last step follows by Lemma 9 in \cite{hk13}.

We define matrix $R$ such that $\ov{V}_2 R = (I-V_1V_1^\top ) V_2$ is the QR decomposition of $(I-V_1V_1^\top ) V_2$, then we have
\begin{align*}
 & ~\|(I-\ov{V}_2 \ov{V}_2^\top)\ov{V}_2 \|\\
= & ~ \|(I-\ov{V}_2 \ov{V}_2^\top)(I-V_1V_1^\top) V_2 R^{-1}\| \\
 = & ~ \|(I-\ov{V}_2 \ov{V}_2^\top)( I - \ov{V}_1 \ov{V}_1^\top +  \ov{V}_1 \ov{V}_1^\top -V_1V_1^\top) V_2 R^{-1} \|\\
 \leq & ~ \underbrace{\|(I-\ov{V}_2 \ov{V}_2^\top)( I - \ov{V}_1 \ov{V}_1^\top) V_2 R^{-1} \| }_{\alpha}  + \underbrace{ \|(I-\ov{V}_2 \ov{V}_2^\top)\| \|R^{-1}\|\|\ov{V}_1 \ov{V}_1^\top -V_1V_1^\top\| }_{\beta},
 \end{align*}
 where the first step follows by $\ov{V}_2 = (I-V_1 V_1^\top) V_2 R^{-1}$, and the last step follows by triangle inequality.

Furthermore, we have,
 \begin{align*}
& ~ \alpha + \beta \\
= & ~ \|(I-\ov{V}_2 \ov{V}_2^\top - \ov{V}_1 \ov{V}_1^\top) V_2 R^{-1}\|   + \|(I-\ov{V}_2 \ov{V}_2^\top)\| \|R^{-1}\|\|\ov{V}_1 \ov{V}_1^\top -V_1V_1^\top\| \\ 
\leq & ~ \|(I-\ov{V}_2 \ov{V}_2^\top ) V_2 R^{-1}\| +  \| \ov{V}_1 \ov{V}_1^\top V_2 R^{-1}\|  + \|(I-\ov{V}_2 \ov{V}_2^\top)\| \|R^{-1}\|\|\ov{V}_1 \ov{V}_1^\top -V_1V_1^\top\| \\ 
\leq & ~ \|\ov{V}_2 \ov{V}_2^\top -V_2V_2^\top\| \|R^{-1}\| +  \| \ov{V}_1 \ov{V}_1^\top V_2\| \|R^{-1}\| + \|(I-\ov{V}_2 \ov{V}_2^\top)\| \|R^{-1}\|\|\ov{V}_1 \ov{V}_1^\top -V_1V_1^\top\| \\ 
= & ~ \|\ov{V}_2 \ov{V}_2^\top -V_2V_2^\top\| \|R^{-1}\| +  \| \ov{V}_1 \ov{V}_1^\top V_2\| \|R^{-1}\| +  \|R^{-1}\|\|\ov{V}_1 \ov{V}_1^\top -V_1V_1^\top\| \\ 
\leq & ~ \|\ov{V}_2 \ov{V}_2^\top -V_2V_2^\top\| \|R^{-1}\| +  \| (I-\ov{V}_2 \ov{V}_2^\top ) V_2 \| \|R^{-1}\| +  \|R^{-1}\|\|\ov{V}_1 \ov{V}_1^\top -V_1V_1^\top\| \\
\leq & ~ (2 \|\ov{V}_2 \ov{V}_2^\top -V_2V_2^\top\| +\|\ov{V}_1 \ov{V}_1^\top -V_1V_1^\top\| ) \|R^{-1}\| \\
\leq & ~ 3 \epsilon  \|R^{-1}\|  \\
\leq & ~ 6 \epsilon,
\end{align*}
where the first step follows by definition of $\alpha,\beta$, the second step follows by triangle inequality, the third step follows by $\| A B\| \leq \| A\| \| B\|$, the fourth step follows by $\|(I-\ov{V}_2 \ov{V}_2^\top)\| =1$, the fifth step follows by Eq.~\eqref{eq:IU1U1V1}, the sixth step follows by Fact~\ref{fac:UU_VV}, the seventh step follows by $\|\ov{V}_1\ov{V}_1^\top - V_1V_1^\top\|\leq \epsilon$ and $\|\ov{V}_2\ov{V}_2^\top - V_2 V_2^\top\| \leq \epsilon$, and the last step follows by $\| R^{-1}\| \leq 2$ (Claim~\ref{cla:sigma_k}).

Finally,
\begin{align*}
\|UU^\top -VV^\top\| \leq & ~ \|UU^\top - \ov{V}\ov{V}^\top\| +  \|\ov{V}\ov{V}^\top - VV^\top\| \\
 = & ~ \| (I - UU^\top )\ov{V} \| +  \|\ov{V}\ov{V}^\top - VV^\top\| \\
\leq & ~ \|\wh{P}_2 - P_2\|/\sigma_k(P_2) +  \|\ov{V}\ov{V}^\top - VV^\top\| \\
\leq & ~ \|\wh{P}_2 - P_2\|/\sigma_k(P_2) + \|\ov{V}_1\ov{V}_1^\top - V_1V_1^\top\| + \|\ov{V}_2\ov{V}_2^\top - V_2 V_2^\top\| \\
\leq & ~ \|\wh{P}_2 - P_2\|/\sigma_k(P_2) + 2\epsilon,
\end{align*}
where the first step follows by triangle inequality, the second step follows by Fact~\ref{fac:UU_VV}, the third step follows by Eq.~\eqref{eq:I_UU_V_upper_bound}, the fourth step follows by triangle inequality, and the last step follows by $\|\ov{V}_1\ov{V}_1^\top - V_1V_1^\top\|\leq \epsilon$ and $\|\ov{V}_2\ov{V}_2^\top - V_2 V_2^\top\| \leq \epsilon$.
 
Therefore we finish the proof. 
\end{proof}

It remains to prove Claim~\ref{cla:sigma_k}.
\begin{claim}\label{cla:sigma_k}
$ \sigma_k(R) \geq 1/2$.
\end{claim}
\begin{proof}
First, we can rewrite $R^\top R$ in the follow way,
\begin{align*}
R^\top R = V_2^\top(I-V_1V_1^\top ) V_2 = I - V_2^\top V_1V_1^\top  V_2
\end{align*}
Second, we can upper bound $\|V_2^\top V_1\|$ by $1/4$,
\begin{align*}
\|V_2^\top V_1\| = & ~ \|V_2 V_2^\top V_1\| \\
\leq & ~ \| ( V_2 V_2^\top - \ov{V}_2\ov{V}_2^\top)V_1 \| + \| \ov{V}_2\ov{V}_2^\top V_1 \| \\
\leq & ~ \| ( V_2 V_2^\top - \ov{V}_2\ov{V}_2^\top)V_1 \| + \|\ov{V}_2^\top (V_1 V_1^\top - \ov{V}_1\ov{V}_1^\top)\| +  \|\ov{V}_2^\top \ov{V}_1\ov{V}_1^\top\| \\
= & ~ \| ( V_2 V_2^\top - \ov{V}_2\ov{V}_2^\top)V_1 \| + \|\ov{V}_2^\top (V_1 V_1^\top - \ov{V}_1\ov{V}_1^\top)\|  \\
\leq & ~ \|  V_2 V_2^\top - \ov{V}_2\ov{V}_2^\top\| \cdot \| V_1 \| + \|\ov{V}_2^\top \| \cdot \| V_1 V_1^\top - \ov{V}_1\ov{V}_1^\top\|  \\
\leq & ~\epsilon +\epsilon\\
\leq & ~ 1/4,
\end{align*}
where the first step follows by $V_2^\top V_2 = I$, the second step follows by triangle inequality, the third step follows by triangle inequality, the fourth step follows by $\|\ov{V}_2^\top \ov{V}_1 \ov{V}_1^\top \|=0$, the fifth step follows by $\| A B \| \leq \|A \| \cdot \| B\|$, and the last step follows by $\|\ov{V}_1\ov{V}_1^\top - V_1V_1^\top\|\leq \epsilon$, $\| V_1 \|=1$, $\|\ov{V}_2\ov{V}_2^\top - V_2 V_2^\top\| \leq \epsilon$ and $\| \ov{V}_2^\top \|=1$, and the last step follows by $\epsilon < 1/8$.

Thus, we can lower bound $\sigma_k^2(R)$,
\begin{align*}
\sigma_k^2(R ) = & ~ \lambda_{\min} (R^\top R)\\ 
= & ~ \min_{\| a\|=1} a^\top R^\top R a\\
= & ~ \min_{ \|a\|=1} a^\top I a - \|V_2^\top V_1 a \|^2 \\
= & ~ 1 - \max_{\| a\|=1}  \|V_2^\top V_1 a\|^2 \\
= & ~ 1 - \| V_2^\top V_1 \|^2 \\
\geq &  ~3/4
\end{align*}
which implies $\sigma_k(R)\geq 1/2$.
\end{proof}

\subsection{Error Bound for the Reduced Third-order Moment}
\subsubsection{Error Bound for the Reduced Third-order Moment in Different Cases}
\begin{lemma}\label{lem:3rd_for_R3_approx}
Let $M_3$ be defined as in Definition~\ref{def:M_m}. Let $\wh{M}_3$ be the empirical version of $M_3$, i.e.,
\begin{align*}
  \wh{M}_3 = \frac{1}{|S|}\sum_{(x,y)\in S} y \cdot ( x^{\otimes 3} -  x \wt\otimes I), 
\end{align*}
where $S$ denote a set of samples (where each sample is i.i.d. sampled from Distribution $\D$ defined in Eq.~\eqref{eq:model}). 
Assume $M_3 \neq 0$, i.e., $m_{3,i} \neq 0$ for any $i$. Let $V\in\R^{d\times k}$ be an orthogonal matrix satisfying $\|UU^\top - VV^\top\| \leq 1/4$, where $U$ is the orthogonal basis of $\mathrm{span}\{ w_1^*, w_2^*,\cdots, w_k^*\}$. 
Then for any $0<\epsilon<1,t\geq 1$, if 
\begin{equation*}
|S| \geq \max_{i\in[k]} (  \| w_i^*\|^{p+1} / |m_{3,i} |^2 + 1 ) \cdot \epsilon^{-2} \cdot  k^2 \poly(\log d, t)
\end{equation*} % t^{(p+1)/2+3} \log^{p+7}(d)
 with probability at least $1-1/d^{t}$,
\begin{align*}
\|M_3(V,V,V) - \wh{M}_3(V,V,V)\| \leq   \epsilon \sum_{i=1}^k  |v_i^*m_{3,i} |. 
\end{align*}  
\end{lemma}

\begin{proof} 

Since $y = \sum_{i=1}^k v_i^* \phi( w_i^{*\top } x)$. We consider each component $i\in[k]$.  We define function $T_i(x) : \R^d \rightarrow \R^{k\times k \times k}$ such that, 
\begin{align*}
T_i(x) = \phi( w_i^{*\top } x ) \cdot ((V^\top  x)^{\otimes 3} - (V^\top  x) \wt\otimes I).
\end{align*}
We flatten tensor $T_i(x)$ along the first dimension into matrix $B_i(x)  \in \R^{k\times k^2}$. %Let $B= \E{B_j} $.
 Define $g(z) = \phi(z) - \phi(0)$, then $|g(z)| = |\int_0^z \phi'(s)ds |\leq L_1/(p+1) |z|^{p+1}$, which follows Property~\ref{pro:gradient}. 

 %We will apply Lemma~\ref{lem:modified_bernstein_non_zero} and check each condition in it.

(\RN{1}) Bounding $\| B_i(x) \|$.

\begin{align*}
\|B_i(x) \|  \leq & ~ |\phi( w_i^{*\top } x )| \cdot (\|V^\top  x \|^3 + 3k \|V^\top  x\| ) \\
\lesssim & ~ (  | w_i^{*\top } x|^{p+1} + |\phi(0)|) \cdot (\|V^\top  x \|^3 + 3k \|V^\top  x_j\| ) \\
\end{align*}
Note that $V^\top  x \sim \N(0,I_k)$. According to Fact~\ref{fac:inner_prod_bound} and Fact~\ref{fac:gaussian_norm_bound}, we have for any constant $t \geq1$, with probability  $1-1/ (n d^t )$,
\begin{align*}
\|B_i(x)\|   \lesssim  ( \| w_i^*\|^{p+1} + |\phi(0)|) k^{3/2} \poly(\log d,t) %(t \log(n))^{(p+1)/2+2}
\end{align*}
%Therefore, $R(t) = (L_1/(p+1) \| w_i^*\|^{p+1} + |\phi(0)|) k^{3/2} t^{(p+1)/2+2}$ and $K(t) = (p+1)/2+2$. 

(\RN{2}) Bounding $\| \E_{x\sim \D_d} [B_i(x) ] \|$.

 Note that $ \E_{x\sim \D_d} [B_i(x) ] = m_{3,i} (V^\top \ov w_i^{*}) \text{vec}((V^\top \ov w_i^{*})(V^\top \ov w_i^{*})^\top)^\top $. Therefore, $\|\E_{x\sim \D_d} [B_i(x) ]\| = |m_{3,i} | \|V^\top \ov w_i^{*}\|^{3} $. Since $\| VV^\top - UU^\top \| \leq 1/4$, $\|VV^\top \ov w_i^{*} - \ov w_i^{*}\| \leq 1/4 $ and $3/4\leq \|V^\top \ov w_i^{*}\|\leq 5/4$. So $\frac{1}{4}|m_{3,i}| \leq \|B\|\leq 2 |m_{3,i} | $.

(\RN{3}) Bounding $\max ( \E_{x\sim \D_d} \| B_i(x) B_i(x)^\top \|,\E_{x\sim \D_d} \| B_i(x)^\top B_i(x) \| ) $.

\begin{align*}
\left\|\E_{x\sim \D_d} [ B_i(x) B_i(x)^\top ] \right\| \lesssim & ~ \left( \E_{x\sim \D_d}[ \phi( w_i^{*\top} x)^4 ] \right)^{1/2}  \left( \E_{x\sim \D_d}[ \|V^\top  x\|^6] \right)^{1/2} \\
\lesssim & ~ ( \| w_i^*\|^{p+1}+|\phi(0)|)^2 k^{3/2}.
\end{align*}
\begin{align*}
 & ~\left\|\E_{x\sim\D_d}[ B_i(x)^\top B_i(x)] \right\| \\
 \lesssim & ~ \left( \E_{x\sim \D_d}[ \phi( w_i^{*\top} x)^4] \right)^{1/2} \left( \E_{x\sim \D_d}[ \|V^\top  x\|^4] \right)^{1/2} \max_{\|A\|_F=1} \left( \E_{x\sim \D_d}[\langle A,(V^\top  x)(V^\top  x)^\top \rangle^4]\right)^{1/2} \\
\lesssim & ~ ( \| w_i^*\|^{p+1}+|\phi(0)|)^2 k^2.
\end{align*}

(\RN{4}) Bounding $\max_{\| a\|=\|b\|=1} (\E_{x\sim \D_d} [(a^\top B_i(x) b)^2] )$.

\begin{align*}
&  ~ \max_{\|a\|=\| b\| = 1} \left(\E_{x\sim \D_d}[ (a^\top B_i(x) b)^2 ]\right)^{1/2} \\
   \lesssim & ~ \left( \E_{x\sim \D_d}[ (\phi( w_i^{*\top} x))^4] \right)^{1/4} \max_{\| a\|=1} \left( \E_{x\sim \D_d}[ ( a^\top V^\top  x)^4] \right)^{1/2}\max_{\|A\|_F=1}\left( \E_{x\sim \D_d}[\langle A,(V^\top  x)(V^\top  x)^\top \rangle^4] \right)^{1/2} \\
 \lesssim & ~ ( \| w_i^*\|^{p+1}+|\phi(0)|)k
\end{align*}

Define $L = \| w_i^*\|^{p+1}+|\phi(0)|$.
Then we have for any $0<\epsilon <1$, if
\begin{equation*}
|S| \gtrsim \frac{L^2k^2 + |m_{3,i} | ^2+ k^{3/2} \poly(\log d, t) | m_{3,i}| \epsilon}{ \epsilon^2 |m_{3,i} |^2 } t \log (k) 
\end{equation*} 
with probability at least $1-k^{-t}$,
\begin{equation*}
\left \|\E_{x\sim \D_d}[B_i(x)] - \frac{1}{|S|}\sum_{x\in S} B_i(x) \right\| \leq \epsilon |m_{3,i}|.
\end{equation*}
We can set $t =T \log(d)/\log(k)$, then  if
\begin{equation*}
|S| \geq \epsilon^{-2} (1+ 1/|m_{3,i} |^2)  \poly(T,\log d) % k^2 T^{(p+1)/2+2} \log^{(p+1)/2+2}(d)\log^{ (p+1)/2+2}(n)   T \log^2 d
\end{equation*}
 with probability  at least $1-d^{-T}$,
\begin{equation*}
\left \|\E_{x\sim \D_d}[B_i(x)] - \frac{1}{|S|}\sum_{x\in S} B_i(x) \right\| \leq \epsilon |m_{3,i}|.
\end{equation*}
Also note that for any symmetric 3rd-order tensor $R$, the operator norm of $R$, 
\begin{align*}
\|R\| = \max_{\| a\|=1} |R( a, a, a)| \leq \max_{\| a\|=1} \|R( a,I,I)\|_F = \|R^{(1)}\| .
\end{align*}
%Now using Lemma~\ref{lem:remove_n} to replace $\log(n)$ by $\log(d)$ and summing up all $k$ components complete the proof.
\end{proof}

\begin{lemma}\label{lem:4th_for_R3_approx}
Let $M_4$ be defined as in Definition~\ref{def:M_m}. Let $\wh{M}_4$ be the empirical version of $M_4$, i.e.,
\begin{align*}
  \wh{M}_4 = \frac{1}{|S|} \sum_{(x,y)\in S} y \cdot \left( x^{\otimes 4} -  ( x \otimes  x) \wt{\otimes} I +   I \wt{\otimes} I \right), 
\end{align*}
where $S$ is a set of samples (where each sample is i.i.d. sampled from Distribution $\D$ defined in Eq.~\eqref{eq:model}). 
Assume $M_4 \neq 0$, i.e., $m_{4,i} \neq 0$ for any $i$. Let $ \alpha$ be a fixed unit vector. Let $V\in\R^{d\times k}$ be an orthogonal matrix satisfying $\|UU^\top - VV^\top\| \leq 1/4$, where $U$ is the orthogonal basis of $\text{span}\{ w_1^*, w_2^*,\cdots, w_k^*\}$.  
Then for any $0<\epsilon<1,t\geq 1$, if 
\begin{align*}
|S| \geq \max_{i\in[k]}( 1+ \| w_i^*\|^{p+1} / |m_{4,i} ( \alpha^\top\ov  w_i^*) |^2  ) \cdot \epsilon^{-2} \cdot  k^2 \poly(\log d,t)
\end{align*} 
 with probability at least $1-d^{-t}$,
\begin{align*}
 \|M_4(V,V,V, \alpha) - \wh{M}_4(V,V,V, \alpha)\| \leq   \epsilon \sum_{i=1}^k  |v_i^*m_{4,i} ( \alpha^\top\ov  w_i^*)|. 
\end{align*}
\end{lemma}

\begin{proof} 

Recall that for each $(x,y) \in S$, we have $y = \sum_{i=1}^k v_i^* \phi( w_i^{*\top } x)$. We consider each component $i\in[k]$.  
We define function $r(x) :\R^d \rightarrow \R^k$ such that
\begin{align*}
r(x) = V^\top  x.
\end{align*}
Define function $T_i(x) : \R^d \rightarrow \R^{k\times k \times k}$ such that
 \begin{align*}
T_i(x)  =  \phi( w_i^{*\top} x )  \left( x^\top  \alpha \cdot  r(x) \otimes  r(x) \otimes  r(x) - (V^\top  \alpha) \tilde \otimes ( r(x) \otimes  r(x) )  -  \alpha^\top  x \cdot  r(x) \tilde \otimes I + (V^\top  \alpha) \tilde\otimes I \right).
\end{align*}
 
We flatten $T_i(x) :\R^d \rightarrow  \mathbb{R}^{k\times k \times k}$ along the first dimension to obtain function $B_i(x) : \R^d \rightarrow  \R^{k \times k^2}$.
 Define $g(z) = \phi(z) - \phi(0)$, then $|g(z)| = |\int_0^z \phi'(s)ds |\leq L_1/(p+1) |z|^{p+1}$, which follows Property~\ref{pro:gradient}. 

(\RN{1}) Bounding $\| B_i(x) \|$.
\begin{align*}
\|B_i(x) \|  \lesssim & ( | w_i^{*\top } x|^{p+1} + |\phi(0)|)  \cdot (|( x^\top  \alpha)| \|V^\top  x\|^3 + 3 \|V^\top  \alpha\| \|V^\top  x\|^2 \\
&+ 3 |( x^\top  \alpha)| \|V^\top  x_j\|\sqrt{k} +3 \|V^\top  \alpha\|\sqrt{k})
\end{align*}
Note that $V^\top x \sim \N(0,I_k)$. According to Fact~\ref{fac:inner_prod_bound} and Fact~\ref{fac:gaussian_norm_bound}, we have for any constant $t \geq1$, with probability  $1-1/( n d^t ) $,
\begin{align*}
\|B_i(x) \|   \lesssim  (\| w_i^*\|^{p+1} + |\phi(0)|) k^{3/2} \poly(\log d, t) 
\end{align*} 

(\RN{2}) Bounding $\| \E_{x\sim \D_d}[ B_i(x)] \|$.

 Note that $ \E_{x\sim \D_d}[ B_i(x)] = m_{4,i} ( \alpha^\top \ov w_i^*)(V^\top \ov w_i^{*}) \text{vec}((V^\top \ov w_i^{*})(V^\top \ov w_i^{*})^\top)^\top $. Therefore, 
\begin{align*}
\left\| \E_{x\sim \D_d} [ B_i(x) ] \right\| = |m_{4,i} ( \alpha^\top \ov w_i^*)| \|V^\top \ov w_i^{*}\|^{3} .
\end{align*}
Since $\| VV^\top - UU^\top \| \leq 1/4$, $\|VV^\top \ov w_i^{*} - \ov w_i^{*}\| \leq 1/4 $ and $3/4\leq \|V^\top \ov w_i^{*}\|\leq 5/4$. So $\frac{1}{4}|m_{4,i} ( \alpha^\top \ov w_i^*)| \leq \| \E_{x\sim \D_d}[ B_i(x)] \| \leq 2 |m_{4,i} ( \alpha^\top \ov w_i^*)| $. 

(\RN{3}) Bounding $\max ( \E_{x\sim \D_d}[B_i(x) B_i(x)^\top ] , \E_{x\sim \D_d}[B_i(x)^\top B_i(x) ] )$.
\begin{align*}
\left\|\E_{x\sim \D_d} [B_i(x) B_i(x)^\top ] \right\| \lesssim & ~ \left( \E_{x\sim \D_d} \left[ \phi( w_i^{*\top} x)^4 \right] \right)^{1/2}  \left( \E_{x\sim\D_d} \left[ ( \alpha^\top x)^4 \right] \right)^{1/2} \left( \E_{x\sim \D_d} \left[\|V^\top  x\|^6\right] \right)^{1/2} \\
\lesssim & ~(\| w_i^*\|^{p+1}+|\phi(0)|)^2 k^{3/2}.
\end{align*}
\begin{align*}
 & ~ \left\| \E_{x\sim \D_d}[ B_i(x)^\top B_i(x) ] \right\|  \\
\lesssim & ~ \left( \E_{x\sim \D_d} [ \phi( w_i^{*\top} x)^4]\right)^{1/2} \left( \E_{x\sim \D_d} [ ( \alpha^\top x)^4 ] \right)^{1/2} \left( \E_{x\sim \D_d}[\|V^\top  x\|^4] \right)^{1/2}  \\
& ~ \cdot \left( \max_{\|A\|_F=1}\E_{x\sim \D_d} \left[ \langle A,(V^\top  x)(V^\top   x)^\top \rangle^4 \right] \right)^{1/2} \\
\lesssim & ~ ( \| w_i^*\|^{p+1}+|\phi(0)|)^2 k^2.
\end{align*}

(\RN{4}) Bounding $ \max_{\| a\|=\| b\| = 1} ( \E_{x\sim \D_d} \left[ ( a^\top B_i(x)  b)^2 \right] )^{1/2} $.
\begin{align*}
&~ \max_{\| a\|=\| b\| = 1} \left(\E_{x\sim \D_d} \left[ ( a^\top B_i(x)  b)^2 \right] \right)^{1/2} \\
   \lesssim & ~\left( \E_{x\sim \D_d}[ \phi^4( w_i^{*\top} x) ] \right)^{1/4} \left( \E_{x\sim \D_d} \left[( \alpha^\top x)^4 \right]\right)^{1/4} \max_{\| a\|=1} \left( \E_{x\sim \D_d} \left[( a^\top V^\top  x)^4 \right]\right)^{1/2} \\
  & ~ \cdot \max_{\|A\|_F=1} \left( \E_{x\sim \D_d} \left[ \langle A,(V^\top  x)(V^\top  x)^\top \rangle^4 \right] \right)^{1/2} \\
 \lesssim & ~ ( \| w_i^*\|^{p+1}+|\phi(0)|)k.
\end{align*}

Define $L = \| w_i^*\|^{p+1}+|\phi(0)|$.
Then we have for any $0<\epsilon <1$, if
\begin{equation*}
|S| \geq \frac{L^2k^2 + |  m_{4,i} ( \alpha^\top \ov w_i^*)| ^2+ k^{3/2} \poly(t,\log d) | m_{4,i} ( \alpha^\top \ov w_i^*)| \epsilon}{ \epsilon^2 ( m_{4,i} ( \alpha^\top \ov w_i^*))^2 } \cdot t \log k
\end{equation*} 
with probability   at least $1-k^{-t}$,
\begin{align}\label{eq:bound_by_eps_m4i}
\left\| \E_{x\sim \D_d}[B_i(x)] - \frac{1}{|S|}\sum_{x\in S}^n B_i(x) \right\| \leq \epsilon | m_{4,i} ( \alpha^\top \ov w_i^*)|.
\end{align}
We can set $t =T \log(d)/\log(k)$, then  if
\begin{equation*}
|S| \geq \frac{(L+ | m_{4,i} ( \alpha^\top \ov w_i^*)|)^2 k^2 \poly(T,\log d) }{ \epsilon^2 | m_{4,i} ( \alpha^\top \ov w_i^*)|^2 } \cdot T \log^2 d
\end{equation*}
 with probability at least $1-d^{-T}$, Eq.~\eqref{eq:bound_by_eps_m4i} holds.
%\begin{equation*}
%\left\| \E_{x\sim \D_d}[B_i(x)] - \frac{1}{|S|}\sum_{x\in S}^n B_i(x) \right\| \leq \epsilon | m_{4,i} ( \alpha^\top \ov w_i^*)|.
%\end{equation*}
Also note that for any symmetric 3rd-order tensor $R$, the operator norm of $R$,
\begin{align*}
 \|R\| = \max_{\| a\|=1} |R( a, a, a)| \leq \max_{\| a\|=1} \|R( a,I,I)\|_F = \|R^{(1)}\|.
 \end{align*}
\end{proof}

\subsubsection{Final Error Bound for the Reduced Third-order Moment}%{Proof of Lemma~\ref{lem:tensor_w_error}}

Lemma~\ref{lem:tensor_w_error} shows $\wh{R}_3$ can approximate $R_3$ to some small precision with $\poly(k)$ samples. 
\begin{lemma}[Estimation of the reduced third order moment]\label{lem:tensor_w_error}
Let $U\in \R^{d\times k}$ denote the orthogonal column span of $W^*$. 
Let $ \alpha$ be a fixed unit vector and $V\in \R^{d\times k}$ denote an orthogonal matrix satisfying $\| VV^\top - UU^\top \| \leq 1/4$. Define 
$R_3 := P_3(V,V,V)$, where $P_3$ is defined as in Definition~\ref{def:P2_P3} using $ \alpha$. 
Let $\wh{R}_3$ be the empirical version of $R_3$ using dataset $S$, where each sample of $S$ is i.i.d. sampled from distribution ${\cal D}$(defined in \eqref{eq:model}). Assume the activation function satisfies Property~\ref{pro:gradient} and Assumption~\ref{assumption:non_zero_moment}. Then for any $0<\epsilon <1$ and $t\geq 1$, define $ j_3 = \min \{j\geq 3| M_{j}\neq 0\} $ and $m_0=\min_{i\in[k]} \{(m_i^{(j_3)}( \alpha^\top \ov{ w}_i^*)^{j_3-3})^2\}$, if
\begin{equation*}
|S| \geq \sigma_1^{2p+2} \cdot k^2 \cdot \poly(\log d,t)  /(\epsilon^2 m_0)
\end{equation*} 
then we have ,
\begin{align*}
 \|R_3 - \wh{R}_3\| \leq     \epsilon \sum_{i=1}^k  |v_i^* m_{j_3,i} (\ov w_i^{*\top}  \alpha)^{j_3-3}| , 
 \end{align*} 
holds with probability at least $1-d^{-\Omega(t)}$.
\end{lemma}

\begin{proof}
The main idea is to use matrix Bernstein bound after matricizing the third-order tensor. Similar to the proof of Lemma~\ref{lem:tensor_2_m}, we consider each node component individually and then sum up the errors and apply union bound. 

We have shown the bound for $j_3=3,4$ in Lemma~\ref{lem:3rd_for_R3_approx} and Lemma~\ref{lem:4th_for_R3_approx} respectively.
To summarize, for any $0<\epsilon <1$ we have if
\begin{equation*}
|S| \geq  \max_{i\in[k]} \left(1 +  \| w_i^*\|^{p+1} / | m_{j_3,i} (\ov w_i^{*\top}  \alpha)^{(j_3-3)}|^2   \right)  \cdot \epsilon^{-2} \cdot k^2  \poly(\log d , t)
\end{equation*} %t^{p+5}\log^{p+8} (d)  
with probability   at least $1-d^{-t}$,
\begin{align*}
\|R_3 - \wh{R}_3\| \leq  \epsilon \sum_{i=1}^k  |v_i^* m_{j_3,i} (\ov w_i^{*\top}  \alpha)^{j_3-3}| . 
\end{align*}
\end{proof}

\subsection{Error Bound for the Magnitude and Sign of the Weight Vectors}

The lemmata in this section together with Lemma~\ref{lem:tensor_2_m} provide guarantees for Algorithm~\ref{alg:recover}. In particular, Lemma~\ref{lem:tensor_1_m} shows with linear sample complexity in $d$, we can approximate the 1st-order moment to some precision.  And Lemma~\ref{lem:solu_linsys1} and Lemma~\ref{lem:solu_linsys2} provide the error bounds of linear systems Eq.~\eqref{eq:linear_systems} under some perturbations. 
\subsubsection{Robustness for Solving Linear Systems}\label{app:linear_sys}

\begin{lemma}[Robustness of linear system]\label{lem:robust_linear_system}
Given two matrices $A,\wt{A} \in\R^{d\times k}$, and two vectors $ b,\wt{b} \in\R^{d}$. Let $z^* = \argmin_{z\in \R^k}\| Az -  b\|$ and $\wh{z} = \argmin_{z\in \R^k} \| ( A + \wt{A} ) z - ( b + \wt{b} ) \|$. If $\|\wt{A}\| \leq \frac{1}{4\kappa} \sigma_k(A)$ and $\|\wt{b}\| \leq \frac{1}{4} \| b\|$,
then, we have 
\begin{align*}
\| z^* - \wh{z} \| \lesssim &(\sigma_k^{-4}(A) \sigma_1^2(A) +   \sigma_k^{-2}(A)  )\| b\| \|\wt{A} \|  +  \sigma_k^{-2}(A)  \sigma_1(A) \|\wt{b}\|.
\end{align*}
\end{lemma}

\begin{proof}
By definition of $z$ and $\wh{z}$, we can rewrite $z$ and $\wh{z}$,
\begin{align*}
z = & ~ A^\dagger b = (A^\top A)^{-1} A^\top b \\
\wh{z} = & ~ ( A + \wt{A} )^\dagger ( b + \wt{b} ) = ((A+\wt{A} )^\top (A+\wt{A} ))^{-1} (A+\wt{A} )^\top ( b+\wt{b} ).
\end{align*}
As $\|\wt{A} \| \leq \frac{1}{4\kappa} \sigma_k(A)$, we have $\| \wt{A}^\top A + A^\top \wt{A} \| \|(A^\top A)^{-1}\| \leq 1/4$. Together with $\|\wt{b}\|\leq  \frac{1}{4} \| b\|$, we can ignore the high-order errors. 
So we have 
\begin{align*}
&  ~ \|\wh{z} - z^*\|  \\
 \lesssim & ~ \| (A^\top A)^{-1} (\wt{A}^\top  b +A^\top \wt{b} ) + (A^\top A)^{-1}(A^\top \wt{A} + \wt{A}^\top A) (A^\top A)^{-1} A^\top  b \| \\
 \lesssim & ~ \|(A^\top A)^{-1} \|(\|\wt{A} \| \| b\| +\|A\| \|\wt{b} \|)  + \|(A^\top A)^{-2}\| \cdot \|A\| \|\wt{A} \| \|A\|\| b\|  \\
  \lesssim & ~ \sigma_k^{-2}(A) (\|\wt{A} \| \| b\| +\sigma_1(A) \|\wt{b}\|)  +  \sigma_k^{-4}(A) \cdot \sigma_1^2(A) \|\wt{ A} \| \| b\|.
 \end{align*}
\end{proof}

\subsubsection{Error Bound for the First-order Moment}%{Proof of Lemma~\ref{lem:tensor_1_m}}

\begin{lemma}[Error bound for the first-order moment]\label{lem:tensor_1_m}
Let $Q_1$ be defined as in Eq.~\eqref{eq:def_Q1} and $\wh{Q}_1$ be the empirical version of $Q_1$ using dataset $S$, where each sample of $S$ is i.i.d. sampled from distribution ${\cal D}$(defined in \eqref{eq:model}). Assume the activation function satisfies Property~\ref{pro:gradient} and Assumption~\ref{assumption:non_zero_moment}. Then for any $0<\epsilon <1$  and $t\geq 1$, define $ j_1 = \min \{j\geq 1| M_{j}\neq 0\} $ and $m_0 = \min_{i\in[k]}(m_{j_1,i} (\ov w_i^{*\top}  \alpha)^{j_1-1})^2$ if
 \begin{equation*}
|S| \geq \sigma_1^{2p+2} d  \poly(t ,\log d)  / (\epsilon^2 m_0)
\end{equation*} 
 we have with probability at least $1-d^{-\Omega(t)}$,
\begin{align*} \|Q_1 - \wh{Q}_1\| \leq   \epsilon \sum_{i=1}^k  |v_i^*m_{j_1,i} (\ov w_i^{*\top}  \alpha)^{j_1-1}|. 
\end{align*}  
\end{lemma}
%{\color{red}Don't we have $\|\wh{Q}_1\|$ on the right hand side of the above equation?}

 \begin{proof}
  We consider the case when $l_1 = 3$, i.e.,
 \begin{align*}
 Q_1 = M_{3}(I, \alpha, \alpha) = \sum_{i=1}^k v_i^*  m_{3,i} ( \alpha^\top \ov{ w}_i^*)^{3} \ov{ w}_i^{*}.
 \end{align*}
 And other cases are similar. 
 
Since $y = \sum_{i=1}^k v_i^* \phi( w_i^{*\top } x)$. We consider each component $i\in[k]$. 

Define function $B_i(x) : \R^d \rightarrow \R^d$ such that
\begin{align*}
 B_i(x) = [\phi( w_i^{*\top } x) \cdot ( x^{\otimes 3} -  x \wt{\otimes} I)](I, \alpha, \alpha) = \phi( w_i^{*\top } x) \cdot (( x^\top  \alpha)^2 x - 2(  x^\top \alpha) \alpha - x).
 \end{align*}

 Define $g(z) = \phi(z) - \phi(0)$, then $|g(z)| = |\int_0^z \phi'(s) \mathrm{d}s |\leq L_1/(p+1) |z|^{p+1}$, which follows Property~\ref{pro:gradient}.
 % We will apply Lemma~\ref{lem:modified_bernstein_non_zero} and check each condition in it.

(\RN{1}) Bounding $\| B_i(x) \|$.

\begin{align*}
\| B_i(x) \|  %= & \phi( w_i^{*\top } x_j) \cdot (( x_j^\top  \alpha) x_j^{\otimes 2} -  \alpha^\top  x_j I -  \alpha  x_j^\top - x_j \alpha^\top) \\
\leq & ~ |\phi( w_i^{*\top } x)| \cdot \|(( x^\top  \alpha)^2  x  - 2  \alpha^\top  x  \alpha  - x)\| \\
\leq & ~ (  | w_i^{*\top } x|^{p+1} + |\phi(0)|) ((( x^\top  \alpha)^2+1) \| x\| + 2| \alpha^\top  x|)
\end{align*}
According to Fact~\ref{fac:inner_prod_bound} and Fact~\ref{fac:gaussian_norm_bound}, we have for any constant $t \geq1$, with probability  $1- 1/ (n d^{t})$,
\begin{align*}
\| B_i(x) \|   \lesssim  ( \| w_i^*\|^{p+1} + |\phi(0)|) \sqrt{d} \poly(\log d, t) %(t \log(n))^{(p+1)/2+2}
\end{align*}
%Therefore, $R(t) = (L_1/(p+1) \| w_i^*\|^{p+1} + |\phi(0)|) \sqrt{d} t^{(p+1)/2+2}$ and $K(t) = (p+1)/2+2$. 

(\RN{2}) Bounding $\| \E_{x\sim \D_d}[B_i(x)] \|$.

 Note that $\E_{x\sim \D_d}[ B_i(x) ] =  m_{3,i} (\ov w_i^{*\top}  \alpha)^2 \ov w_i^*$. Therefore, $\| \E_{x\sim \D_d}[ B_i(x) ] \| = | m_{3,i} (\ov w_i^{*\top}  \alpha)^2| $. 

(\RN{3}) Bounding $\max ( \E_{x\sim \D_d} \| B_i(x) B_i(x)^\top \|,\E_{x\sim \D_d} \| B_i(x)^\top B_i(x) \| ) $.

\begin{align*}
\left\|\E_{x\sim \D_d} \left[ B_i(x)^\top B_i(x) \right] \right\| \lesssim & ~ \left( \E_{x\sim \D_d} \left[ \phi( w_i^{*\top} x)^4 \right] \right)^{1/2} \left( \E_{x\sim \D_d}\left[ ( x^\top  \alpha)^8 \right] \right)^{1/2} \left( \E_{x\sim \D_d} \left[ \| x\|^4 \right] \right)^{1/2} \\
\lesssim & ~ ( \| w_i^*\|^{p+1}+|\phi(0)|)^2 d.
\end{align*}
\begin{align*}
\left\| \E_{x\sim \D_d} \left[ B_i(x) B_i(x)^\top \right] \right\| \lesssim & ~ \left( \E_{x\sim \D_d} \left[ \phi( w_i^{*\top} x)^4 \right] \right)^{1/2} \left( \E_{x\sim \D_d} \left[ ( x^\top  \alpha)^8 \right] \right)^{1/2} \left(\max_{\| a\|=1}\E_{x\sim \D_d} \left[ ( x^\top a)^4 \right] \right)^{1/2}  \\
 \lesssim & ~ ( \| w_i^*\|^{p+1}+|\phi(0)|)^2.
\end{align*}

(\RN{4})  Bounding $\max_{\| a\|=\|b\|=1} (\E_{x\sim \D_d} [(a^\top B_i(x) b)^2] )$.

\begin{align*}
\max_{\| a\|=1} \left( \E_{x\sim \D_d} \left[ ( a^\top B_i(x)  a)^2 \right] \right)^{1/2} \lesssim  \left( \E_{x\sim \D_d} \left[ \phi^4( w_i^{*\top} x) \right] \right)^{1/4} \lesssim  \| w_i^*\|^{p+1}+|\phi(0)|.
\end{align*}

Define $L = \| w_i^*\|^{p+1}+|\phi(0)|$.
Then we have for any $0< \epsilon <1$, if
\begin{equation*}
|S| \gtrsim \frac{ L^2d + | m_{3,i} (\ov w_i^{*\top}  \alpha)^2|^2+  L| m_{3,i} (\ov w_i^{*\top}  \alpha)^2|  \sqrt{d} \poly(\log d,t)  \epsilon}{ \epsilon^2 | m_{3,i} (\ov w_i^{*\top}  \alpha)^2|^2 } \cdot t \log d 
\end{equation*} %t^{(p+1)/2+2} \log^{(p+1)/2+2}(n) 
with probability at least $1-1/d^{t}$,
\begin{equation*}
\left\| \E_{x\sim\D_d}[ B_i(x)] - \frac{1}{|S|}\sum_{x\sim S}^n  B_i(x) \right\| \leq \epsilon | m_{3,i} (\ov w_i^{*\top}  \alpha)^2|.
\end{equation*}
%Now using Lemma~\ref{lem:remove_n} to replace $\log(n)$ by $\log(d)$ and
Summing up all $k$ components, we obtain if 
\begin{equation*}
|S| \geq  \max_{i\in[k]} \left\{\frac{( \| w_i^*\|^{p+1} + |\phi(0)| + | m_{3,i} (\ov w_i^{*\top}  \alpha)^2|)^2   }{ | m_{3,i} (\ov w_i^{*\top}  \alpha)^2|^2 }\right\} \cdot \epsilon^{-2}  d \poly(\log d,t )%\cdot t^{p+4} \cdot d \log^{p+5} (d) 
\end{equation*} 
with probability at least $1-1/d^t$,
\begin{align*}
 \|M_3(I, \alpha, \alpha) - \wh{M}_3(I, \alpha, \alpha)\| \leq   \epsilon \sum_{i=1}^k  |v_i^* m_{3,i} (\ov w_i^{*\top}  \alpha)^2|. 
\end{align*}
Other cases ($j_1 = 1,2,4$) are similar, so we complete the proof.
 \end{proof}

\subsubsection{Linear System for the First-order Moment}%{Proof of Lemma~\ref{lem:solu_linsys1}}
The following lemma provides estimation error bound for the first linear system in Eq.~\eqref{eq:linear_systems}.
\begin{lemma}[Solution of linear system for the first order moment]\label{lem:solu_linsys1}
%Given $\{ w_1^*, w_2^*,\cdots, w_k^*\}$, 
Let $U\in \R^{d\times k}$ be the orthogonal column span of $W^*$. Let $V\in \R^{d\times k}$ denote an orthogonal matrix satisfying that $\| VV^\top - UU^\top \| \leq  \wh{\delta}_2 \lesssim 1/ ( \kappa^2 \sqrt{k})$. For each $i\in[k]$, let $\widehat{u}_i$ denote the vector satisfying $\| \wh{u}_i - V^\top \ov{ w}_i^*\| \leq \wh{\delta}_3\lesssim  1/ ( \kappa^2 \sqrt{k})$. %Let $$Q_1 := M_{l_1}(I,\underbrace{ \alpha,\cdots, \alpha}_{(l_1 - 1)~ \alpha\text{'s}}) = \sum_{i=1}^k v_i^* c^{(l_1)} \| w_i^*\|^{p+1} ( \alpha^\top \ov{ w}_i^*)^{{l_1-1}} \ov{ w}_i^{*}, $$ where $l_1\geq 1$ such that $M_{l_1}\neq 0$. 
Let $Q_1$ be defined as in Eq.~\eqref{eq:def_Q1} and $\wh{Q}_1$ be the empirical version of $Q_1$ such that $\|Q_1 - \wh{Q}_1\| \leq \wh{\delta}_4 \|Q_1\| \leq \frac{1}{4}\|Q_1\|$. 
%Define $z^*_i := v_i^*  c^{(l_1)} \| w_i^*\|^{p+1} ( \alpha^\top \ov{ w}_i^*)^{{l_1-1}}$. 
%Let $\wh{z} = \argmin_{z} \| \sum_{i=1}^k z_i V \wh{u}_i  - \wh{Q}_1\|$.  
Let $z^* \in \R^k$ and $\wh{z} \in \R^k$ be defined as in Eq.~\eqref{eq:zr_star} and Eq.~\eqref{eq:linear_systems}.
Then 
\begin{align*}
| \wh{z}_i - z_i^*| \leq (\kappa^4 k^{3/2} (\wh{\delta}_2+\wh{\delta}_3) + \kappa^2 k^{1/2} \wh{\delta}_4)\|z^*\|_1.
\end{align*}
%Let $ \alpha\in \R^{d}$ be a fixed unit vector sampled from uniform distribution. Then 
\end{lemma}
%{\color{red} why do we care about $\ell_1$ norm in the above Lemma 9?}

\begin{proof}
Let $A \in \mathbb{R}^{k\times k}$ denote the matrix where the $i$-th column is $s_i \ov{w}_i^*$. Let $\wt{A}\in \mathbb{R}^{k\times k}$ denote the matrix where the $i$-th column is $V\wh{u}_i$. Let $b\in \mathbb{R}^k$ denote the vector $Q_1$, let $\wt{b}$ denote the vector $ \wh{Q}_1 - Q_1$.
%We define $A  = \begin{bmatrix} s_1\ov{ w}_1^* & s_2\ov{ w}_2^* & \cdots & s_k\ov{ w}_k^* \end{bmatrix}$, $\wt{A}= \begin{bmatrix} V\wh{u}_1 & V\wh{u}_2 & \cdots & V\wh{u}_k \end{bmatrix} - A$, $ b = Q_1$ and $\wt{b} = \wh{Q}_1 - Q_1$. 
Then we have 
\begin{align*}
 \|A\|  \leq \sqrt{k}.
\end{align*}

Using Fact~\ref{fac:sigma_k_Wbar}, we can lower bound $\sigma_k(A)$,
\begin{align*}
\sigma_k(A)  \geq 1/\kappa.
\end{align*}

We can upper bound $\| \wt{A} \|$ in the following way,
\begin{align*}
\|\wt{A}\| \leq & ~ \sqrt{k}  \max_{i \in [k]} \{\| V\wh{u}_i - s_i\ov{ w}_i^*\| \} \\
\leq & ~ \sqrt{k}  \max_{i\in [k]} \{\| V\wh{u}_i - s_iVV^\top \ov{ w}_i^* + s_iVV^\top \ov{ w}_i^* - s_iUU^\top \ov{ w}_i^*\| \} \\
\leq & ~ \sqrt{k} (\wh{\delta}_3+\wh{\delta}_2).
\end{align*}

We can upper bound $\|b\|$ and $\| \wt{b}\|$,
\begin{align*}
\|  b\| = \|Q_1\| \leq k \sum_{i=1}^k |z_i^*|, \text{~and~} \|\wt{b} \| \leq \wh{\delta}_4 \|Q_1\|.
\end{align*}
To apply Lemma~\ref{lem:robust_linear_system}, we need $\wh{\delta}_4 \leq 1/4$ and $\wh{\delta}_2\lesssim 1/ ( \sqrt{k}\kappa^2 ) $, $\wh{\delta}_3 \lesssim 1/ ( \sqrt{k}\kappa^2 ) $. So we have,
\begin{align*}
\|\wh{z}_i - z_i^* \| \leq & ~ (\kappa^4 k^{3/2} (\wh{\delta}_2+\wh{\delta}_3) + \kappa^2 k^{1/2} \wh{\delta}_4)  \|Q_1\| \\
\leq & ~ (\kappa^4 k^{3/2} (\wh{\delta}_2+\wh{\delta}_3) + \kappa^2 k^{1/2} \wh{\delta}_4) \sum_{i=1}^k |z_i^*|.
\end{align*}
\end{proof}

\subsubsection{Linear System for the Second-order Moment}%
The following lemma provides estimation error bound for the second linear system in Eq.~\eqref{eq:linear_systems}.
\begin{lemma}[Solution of linear system for the second order moment]\label{lem:solu_linsys2}
%Given $\{ w_1^*, w_2^*,\cdots, w_k^*\}$, 
Let $U\in \R^{d\times k}$ be the orthogonal column span of $W^*$ denote an orthogonal matrix satisfying that $\| VV^\top - UU^\top \| \leq  \wh{\delta}_2 \lesssim 1/(\kappa \sqrt{k})$. For each $i\in[k]$, let $\wh{u}_i$ denote the vector satisfying $\| \wh{u}_i - V^\top \ov{ w}_i^*\| \leq \wh{\delta}_3 \lesssim 1/(\sqrt{k}\kappa^3)$.
%Let $$Q_2 := M_{l_2}(V,V,\underbrace{ \alpha,\cdots, \alpha}_{(l_2 - 2)~ \alpha\text{'s}}) = \sum_{i=1}^k v_i^* c^{(l_2)} \| w_i^*\|^{p+1} ( \alpha^\top \ov{ w}_i^*)^{{l_2-2}} (V\ov{ w}_i^{*})(V\ov{ w}_i^{*})^\top,$$ where $l_2 \geq 2$ such that $M_{l_2} \neq 0$.
Let $Q_2$ be defined as in Eq.~\eqref{eq:def_Q2} and $\wh{Q}_2$ be the estimation of $Q_2$ such that $\|Q_2 - \wh{Q}_2\|_F \leq \wh{\delta}_4 \|Q_2\|_F \leq  \frac{1}{4}\|Q_2\|_F$. %Define $r_i^* := v_i^* c^{(l_2)} \| w_i^*\|^{p+1} ( \alpha^\top \ov{ w}_i^*)^{{l_2-2}} $. 
Let $ r^* \in \R^k$ and $ \wh{r} \in \R^k$ be defined as in Eq.~\eqref{eq:zr_star} and Eq.~\eqref{eq:linear_systems}.
%Let $ \wh{r} = \argmin_{ r} \| \sum_{i=1}^k r_i  \wh{u}_i  \wh{u}_i^\top - \wh{Q}_2\|_F$.  
Then
\begin{align*}
|\wh{r}_i - r_i^*| \leq  ( k^3\kappa^{8}\wh{\delta}_3 +  \kappa^2 k^{2} \wh{\delta}_4) \| r^*\| .
\end{align*}
%Let $ \alpha\in \R^{d}$ be a fixed unit vector sampled from uniform distribution. Then 
\end{lemma}

\begin{proof}
For each $i\in [k]$, let $u_i  = V^\top \ov{w}_i^*$. Let  $A\in \R^{k^2 \times k}$ denote the matrix where the $i$-th column is $\text{vec}(u_i u_i^\top)$. Let $\wt{A}\in \mathbb{R}^{k^2 \times k}$ denote the matrix where the $i$-th column is $\text{vec}(u_i u_i^\top - \wh{u}_i \wh{u}_i^\top)$. Let $b\in \mathbb{R}^{k^2}$ denote the vector $\text{vex}(Q_2)$, let $\wt{b} \in \mathbb{R}^{k^2}$ denote the vector $\text{vec}(Q_2 - \wh{Q}_2)$.

% $ b \in \R^k, \wt{A} \in \R^{k^2\times k}, \wt{b} \in \R^{k}$, as follows,
%\begin{align*}
%A_i  =   \text{vec}(u_i u_i^\top), \text{~and~} \wt{A}_i = \text{vec}(u_i u_i^\top - \wh{u}_i \wh{u}_i^\top).
%\end{align*}

%\delta_3 k^3 v^2_{\max} \kappa^5 \frac{v_{\max}}{v_{\min}} + \delta_2 \kappa k v_{\max}\sigma_1^{p+1} \\

%\begin{align*}
% b =  \text{vec}(Q_2), \text{~and~} \wt{b} = \text{vec}(Q_2 - \wh{Q}_2).
%\end{align*}
%where $A_i$ is the $i$-th column of $A$. 

Let $\circ$ be the element-wise matrix product (a.k.a. Hadamard product), $\ov{W} = [\ov{ w}_1^{*}\;\ov{ w}_2^{*}\;\cdots\;\ov{ w}_k^{*}]$ and $U = [u_1\;u_2\;\cdots\;u_k] = V^\top \ov{W}$ .
We can upper bound $\|A\|$ and $\| \wt{A}\|$ as follows,
\begin{align*}
\|A\|  = & ~ \max_{\|x\|=1} \left\|\sum_{i=1}^k x_i \text{vec}(u_iu_i^\top) \right\| \\
= & ~ \max_{\| x\|=1 }\|U \diag(x)U^\top \|_F \\
\leq & ~ \|U\|^2 \\
\leq & ~ \sigma^2_1(V^\top \ov{W}),
\end{align*}
and
\begin{align*}
  \|\wt{A}\|  = & ~ \sqrt{k} \max_{i\in [k]} \|\wt{A}_i\| \\
  \leq & ~\sqrt{k} \max_{i \in [k] } \|u_iu_i^\top  - \wh{u}_i\wh{u}_i^\top\|_F \\
  \leq & ~ \sqrt{k} \max_{i \in [k] } 2  \|u_i  - \wh{u}_i\|_2 \\
  \leq & ~ 2\sqrt{k} \wh{\delta}_3 .
\end{align*}

We can lower bound $\sigma_k(A)$,
 \begin{align*}
\sigma_k(A) = & ~ \sqrt{\sigma_k(A^\top A)} \\
= & ~ \sqrt{\sigma_{k} ((U^\top U) \circ ( U^\top U))}\\
 = & ~ \min_{\|x\|=1}\sqrt{ x^\top ((U^\top U) \circ ( U^\top U))x } \\
=  & ~ \min_{\|x\|=1}\| (U^\top U)^{1/2} \diag(x) ( U^\top U)^{1/2}\|_F \\
\geq   & ~\sigma^2_k(V^\top \ov{W})
  \end{align*}
where fourth step follows Schur product theorem, the last step follows by the fact that $\|CB\|_F \geq \sigma_{\min}(C)\|B\|_F $ and $\circ$ is the element-wise multiplication of two matrices. 

We can upper bound $\| b\|$ and $\| \wt{b} \|$,
  \begin{align*}
 \| b\| \leq & \|Q_2\|_F  \leq \| r^*\| ,
 \end{align*}
  \begin{align*}
 \| \wt{b} \| = & \|Q_2 - \wh{Q}_2\|_F  \leq \wh{\delta}_4 \| r^*\|.
\end{align*}

%where $\xi_1$ follows the fact that $\|CB\|_F \leq \|C\|\|B\|_F $ for any matrices $C,B$,

Since $\|VV^\top \ov{W} - \ov{W}\| \leq \sqrt{k} \wh{\delta}_2 $, we have for any $x \in\R^k$,
\begin{align*}
\|VV^\top \ov{W} x \| \geq & ~ \|\ov{W} x \| - \|(VV^\top \ov{W} - \ov{W})x\| \\
\geq & ~ \sigma_k(\ov{W}) \|x\| -  \wh{\delta}_2 \sqrt{k}  \|x\|
\end{align*}

Note that according to Fact~\ref{fac:sigma_k_Wbar}, $\sigma_k(\ov{W}) \geq 1/\kappa$. Therefore, if $\wh{\delta}_2  \leq 1/(2\kappa \sqrt{k})$, we will have $\sigma_{k}(V^\top \ov{W}) \geq 1 / ( 2\kappa )$.
Similarly, we have $\sigma_{1}(V^\top \ov{W}) \leq \|V\|\|\ov{W}\| \leq \sqrt{k}$. Then applying Lemma~\ref{lem:robust_linear_system} and setting $\wh{\delta}_2 \lesssim \frac{1}{\sqrt{k}\kappa^3}$, we complete the proof. 
\end{proof}

\section{Acknowledgments}
%abcdefghijklmnopqrstuvwxyz
The authors would like to thank Surbhi Goel, Adam Klivans, Qi Lei, Eric Price, David P. Woodruff, Peilin Zhong, Hongyang Zhang and Jiong Zhang for useful discussions.
\end{document}